%% file: neurips_2026.tex
\documentclass{article}

\usepackage[main, final]{neurips_2026}

\usepackage[utf8]{inputenc} 
\usepackage[T1]{fontenc}    
\usepackage{hyperref}       
\usepackage{url}            
\usepackage{booktabs}       
\usepackage{amsfonts}       
\usepackage{nicefrac}       
\usepackage{microtype}      
\usepackage{xcolor}         
\usepackage{amsmath}
\usepackage{enumitem}
\usepackage{amsthm}
\usepackage{enumitem}
\usepackage{array}
\usepackage{caption}
\usepackage{mathtools}
\usepackage{amssymb}
\usepackage[capitalise,noabbrev]{cleveref}
\newtheorem{proposition}{Proposition}
\newtheorem{conjecture}{Conjecture}
\newtheorem{remark}{Remark}

\title{Why Learning Rediscovers the Closed-Form Diagonal Regularizer}

\author{%
  Jeahn Han \qquad Pyojin Kim\thanks{Corresponding author.} \\
  Gwangju Institute of Science and Technology \\
  \texttt{jeahnhaan@gm.gist.ac.kr}, \texttt{pjinkim@gist.ac.kr} \\
  \url{https://ugrm.github.io/diagonal-regularizer/}
}

\begin{document}

\maketitle

\begin{abstract}
Recovering modal amplitudes from a few sensors requires regularization, and the shape of the penalty decides how strongly each mode is shrunk.
The best shape depends on the truncation noise that discarded modes add to every measurement, which is unobserved and could be room-specific.
We identify a diagonal saturation principle: for a single snapshot, isotropic truncation noise makes the Bayes-optimal Tikhonov shape a closed-form power law $\Gamma_k \propto \lambda_k^{|s|}$, set by the prior alone and independent of the domain.
The off-diagonal anisotropy equals the Herfindahl index $H$ of the discarded-mode power exactly, Weyl's eigenvalue counting law fixes the number of discarded modes and hence $H$, and Berry's random-wave conjecture enters only through a diagonal term $D$, which we measure directly.
Together they predict an approximately flat loss landscape over the penalty exponent, which leaves a learned diagonal regularizer little to gain from tuning the exponent per room.
On simulated acoustic rooms the closed form stays within $5.83\%$ of a per-room power-law oracle across observation windows, and three learned diagonal architectures track its error, the closest within a fraction of a point and none more than $3$\,pp above it, while learning qualitatively different spectra.
In the truncation regime no tested method consistently improves on the closed form, including Learned Iterative Ridge, which couples modes across layers.
The framework extends to heat diffusion through a known Green's-function correction with no new free parameters.
\end{abstract}

\input{sections/01_introduction}
\input{sections/02_related_works}
\input{sections/03_problem_setup}
\input{sections/04_why_p_equals_s}
\input{sections/05_experiments}
\input{sections/06_can_learning_improve}
\input{sections/07_cross_physics_validation}
\input{sections/08_discussions_and_limitations}

\begin{ack}
This work was supported by the National Research Foundation of Korea~(NRF) grant funded by the Korea government~(MSIT) (No.RS-2024-00358374) and the Institute of Information \& Communications Technology Planning \& Evaluation (IITP) grant funded by the Korea government~(MSIT) (IITP-2026-RS-2026-25614465, AI Star Fellowship Support Program).
\end{ack}

\newpage
\bibliographystyle{abbrvnat}
\bibliography{refs}


\appendix
\newpage
\input{supplementary/A_extended_proof}
\input{supplementary/B_berry}
\input{supplementary/C_acoustic_experiments}
\input{supplementary/D_learning_experiments}

\input{supplementary/E_heat_equation}
\input{supplementary/F_prior_robustness}
\input{supplementary/G_rectangular_control_experiment}
\input{supplementary/H_sensor_noise}
\input{supplementary/I_real_data}

\clearpage  


\newpage
\input{checklist.tex}

\end{document}

%% file: sections/01_introduction.tex
\section{Introduction}
\label{sec:intro}

Modal inverse problems on bounded domains arise in acoustic, thermal, and electromagnetic PDEs.
A sensor array records finitely many numbers, so every estimator keeps finitely many modes, and the discarded modes enter each measurement as \emph{truncation noise}.
In room acoustics, our running example, a pressure field is a weighted sum of the room's eigenmodes whose weights are recovered from a few microphones.
With $K{=}50$ retained modes and $M{=}8$ microphones, one snapshot gives $8$ equations for $100$ unknowns, and a typical room has over $300$ modes in total, so truncation noise dominates the error budget.

\paragraph{The shape question.}
Regularization makes the problem solvable by penalizing amplitudes the data cannot constrain \citep{stuart2010inverse, kaipio2005statistical}.
The penalty has a scalar strength $\alpha$, whose selection is well studied \citep{hansen1992analysis, golub1979generalized, morozov1966regularization}, and a diagonal shape $\Gamma$ that decides how strongly each mode is shrunk.
In physical systems the energy of mode $k$ decays as $\lambda_k^{-s}$, where $\lambda_k$ is the eigenvalue (roughly frequency squared) and $s$ is set by the excitation; in our simulated data, excited by a broadband click, the population value measured from the modal time series is $|s| = 1.13$ (Appendix~\ref{app:estimating-s}).
The natural penalty is a power law $\Gamma_k = \lambda_k^p$: $p{=}0$ penalizes all modes equally, and $p{=}2$ suppresses high frequencies strongly.
Which $p$ is best, and does it depend on the room?

\paragraph{Why the noise decides.}
The optimal $p$ depends on how signal energy decays across modes, which is easy to measure, and on how truncation error spreads across the retained modes, which is not.
If the noise has room-specific structure, a learned regularizer, such as one trained by deep unrolling \citep{gregor2010learning, adler2018learned, aggarwal2018modl}, could exploit it; if the noise is \emph{isotropic}, then for a single snapshot $p^\star = |s|$, set by the signal alone.
Our key insight is that the discarded modes are too numerous to leave exploitable structure: Weyl's eigenvalue counting law~\citep{weyl1912asymptotische, ivrii2016100} supplies hundreds in the median room (Table~\ref{tab:isotropy-facts}), so the truncation noise concentrates around its isotropic average.
Berry's random-wave conjecture~\citep{berry1977regular}, which says that high-frequency eigenmodes of generic domains behave like random fields, controls only one term of the anisotropy, and we measure that term rather than assume it (Appendix~\ref{app:bernstein}).
Saturation is therefore a property of the truncation noise, not of the models: without discarded modes, freedom in the shape pays again (Appendix~\ref{app:no-trunc-control}).

\begin{figure}[t]
  \centering
  \includegraphics[width=\linewidth]{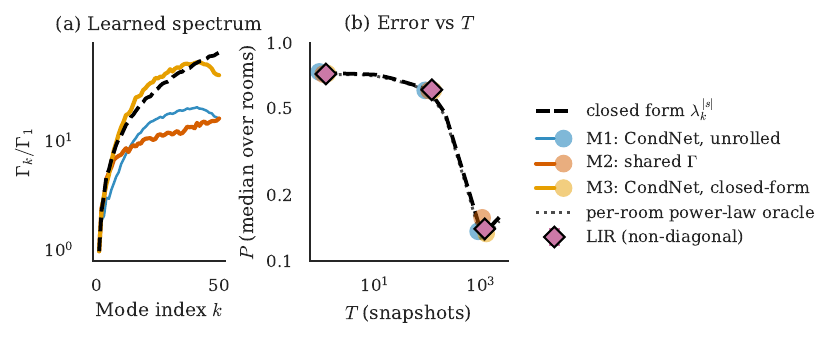}
  \caption{\textbf{Learned diagonal regularizers learn different spectra yet reach the closed form's error.}
  M3 follows the closed form, M1 and M2 learn flatter spectra that penalize high modes less (a, single runs).
  All three track the closed form's error, M2 slightly above at the longest window, and the mode-coupling LIR tracks it too (b).
  Lower $P$ is better, $1$ means no better than guessing zero.
  This is the saturation of \cref{sec:learned}, which \cref{sec:lir} finds outside the diagonal family as well.}
  \label{fig:learned-spectra}
\end{figure}

\paragraph{Findings.}
Across rooms ranging from triangles to decagons, none of the per-mode regularizers we test, including regularizers learned by deep unrolling, gains robustly over the closed form within the diagonal family in the truncation regime (Figure~\ref{fig:learned-spectra}).
Learned Iterative Ridge, which couples modes across layers and so lies outside the diagonal family, does not consistently improve on it either (\cref{sec:lir}); the claim concerns the methods tested, not learning in general.
In practice, $\Gamma_k = \lambda_k^{|s|}$ is set once per excitation regime and used for every room: what transfers is not the measured value of $|s|$ but the room-independence of the shape, which persists when $K$, $M$, and the excitation exponent change (Appendix~\ref{app:v2}).

\paragraph{Contributions.}
\begin{enumerate}
\item \textbf{Theory.}
Under isotropic truncation noise, $p^\star = |s|$ exactly at $T{=}1$ (Proposition~\ref{prop:isotropy}); Weyl's law and Berry's conjecture play separate, measurable roles in the isotropy argument (\cref{sec:theory}).
Multi-snapshot results are empirical.
\item \textbf{Verification.}
On $187$ in-scope rooms ($K_{\mathrm{total}} > K$), the median relative cost of $p{=}|s|$ against per-room power-law oracle tuning is at most $5.83\%$ at every observation window, and at $T{=}1000$ two rooms in three lose less than $1.1$\,pp absolute (\cref{sec:acoustic}); the $10$ rooms with $K_{\mathrm{total}} \leq K$ are reported separately.
\item \textbf{Diagonal saturation.}
Three learned diagonal architectures track the closed form's $P$-vs-$T$ curve, all within $3$\,pp and M3 closest (Table~\ref{tab:m3_performance}), while learning qualitatively different spectra; $p{=}|s|$ is a principled choice inside a wide near-optimal basin (\cref{sec:learned}).
\item \textbf{Cross-PDE consistency.}
For heat diffusion the theory predicts $\Gamma_k \propto \lambda_k^{s} e^{c\lambda_k}$ with $c = 2\kappa t$; per-room fits recover the rate to within a few percent (\cref{sec:heat}), a consistency check, not physical validation.
\end{enumerate}

%% file: sections/02_related_works.tex
\section{Related Work}
\label{sec:related}

\paragraph{Regularizer strength and shape.}
Tikhonov regularization has two parts: a scalar strength $\alpha$ and a matrix $\Gamma$ that decides which modes to penalize \citep{stuart2010inverse, kaipio2005statistical}.
Strength selection is mature, as the L-curve \citep{hansen1992analysis}, generalized cross-validation \citep{golub1979generalized}, the discrepancy principle \citep{morozov1966regularization}, and Bayesian posterior contraction \citep{cavalier2008nonparametric, knapik2011bayesian} choose $\alpha$ reliably.
For the shape, truncated SVD and Landweber iteration filter the operator's singular values \citep{engl1996regularization, hansen1998rank}; Pinsker's estimator gives minimax per-coordinate shrinkage \citep{pinsker1980optimal, johnstone2002function}, hierarchical Bayesian models derive per-component weights from data \citep{calvetti2025distributed}, and spectral Bayesian methods recover analogous structure on manifolds \citep{durastanti2026spectral}; these three require training data or an assumed smoothness class.
\citet{alberti2021learning} prove that the MSE-optimal shape depends only on the signal covariance $\Sigma_x$, with $O(1/\sqrt{m})$ generalization bounds, and \citet{leong2024star} confirm this geometrically.
In modal problems the truncation noise is approximately isotropic, which reduces the shape question to $\Sigma_x$, and a power-law excitation prior reduces $\Sigma_x$ to one scalar, $|s|$.

\paragraph{Learned regularizers.}
Algorithm unrolling \citep{gregor2010learning} turns iterative solvers into trainable networks and underlies many learned regularizers \citep{adler2018learned, sun2016deep, aggarwal2018modl}.
The learned answer is often the classical one: learned parameters converge to variational solutions \citep{kofler2023learning}, bilevel optimization reduces to hyperparameter tuning \citep{kunisch2013bilevel}, an untrained CNN matches a trained denoiser \citep{ulyanov2018deep}, and plug-and-play regularizers collapse to the denoiser's spectral penalty \citep{hurault2022proximal}.
Instabilities in learned methods trace to an accuracy--stability tradeoff \citep{antun2020instabilities, gottschling2025troublesome} that no algorithm can circumvent in general \citep{colbrook2022difficulty}.
We give one concrete reason for this pattern in modal problems: the loss landscape over the penalty exponent is flat.

\paragraph{Berry's conjecture and Weyl's law.}
Berry's conjecture states that high-frequency eigenmodes of generic domains look like random waves \citep{berry1977regular}; quantum ergodicity supports it \citep{shnirel1974ergodic, zelditch2005quantum}, and it fails for scarred states \citep{heller1984bound} and some symmetric geometries \citep{hassell2010ergodic}.
Weyl's law counts such modes: in 2D the eigenvalue count grows linearly with $\lambda$ \citep{weyl1912asymptotische, ivrii2016100}, a fact used in wave-chaotic compressive sensing \citep{del2020implementing} and cavity acoustics \citep{tanner2007wave}.
We are not aware of prior work that uses either result to choose a regularizer.

\paragraph{Sound-field reconstruction.}
Sparse-microphone reconstruction uses Bayesian methods \citep{schmid2021spatial}, compressive sensing \citep{antonello2017room}, spherical arrays \citep{fernandez2016sound}, physics-informed networks that penalize the PDE residual \citep{karakonstantis2024room}, and Mat\'ern-kernel Gaussian-process priors, whose regularity parameter plays the role of $|s|$ \citep{rasmussen2003gaussian}.
Sensor-placement methods ask where to put the microphones \citep{krause2008near, alexanderian2014optimal}; we ask which penalty shape to use and whether it depends on the room.

\paragraph{Backward heat problems.}
The backward heat equation is a standard ill-posed problem: mode amplitudes decay as $e^{-\kappa\lambda_k t}$, so recovering them amplifies noise exponentially \citep{beck1985inverse, kaipio2011bayesian}.
Minimax rates \citep{knapik2013bayesian} and variational source conditions \citep{hohage2017characterizations} give the scaling but not the regularizer, and they do not involve $\kappa$ or $t$.

%% file: sections/03_problem_setup.tex
\section{Problem Setup}
\label{sec:setup}

We consider modal inverse problems on a bounded 2D domain $\Omega \subset \mathbb{R}^2$ with eigenpairs $(\lambda_k, \varphi_k)$ of the Laplacian.
The framework is instantiated on the acoustic wave equation and validated on $187$ random convex polygons, with eigenpairs computed by FEM and pressure fields simulated by a finite-difference time-domain solver (\cref{sec:acoustic}, Appendix~\ref{app:room-geometry}); the heat equation serves as a cross-PDE check (\cref{sec:heat}).
The state at time $t$ admits the modal expansion
\begin{equation}
\label{eq:modal-expansion}
u(x, t) = \sum_{k=1}^{\infty} a_k(t)\,\varphi_k(x),
\end{equation}
with eigenfunctions $L^2$-normalized so that $\int_\Omega \varphi_k^2 = 1$ and $\mathbb{E}_x[\varphi_k^2(x)] = 1/|\Omega|$ for $x$ uniform in $\Omega$.
Each eigenfunction is a spatial pattern, and the amplitude $a_k(t)$ says how much of that pattern is present at time $t$.
Higher modes typically carry less energy: the variance of $a_k$ decays as a power law in $\lambda_k$ with an exponent $s$ that depends on the excitation statistics and not on the room geometry.
For the acoustic wave equation with uniform damping $\gamma$, each amplitude is a damped sinusoid,
\begin{equation}
\label{eq:acoustic-dynamics}
a_k(t) = e^{-\gamma t}\bigl[c_k \cos(\omega_k t) + \beta_k \sin(\omega_k t)\bigr], \qquad \omega_k = c_s\sqrt{\lambda_k}.
\end{equation}
A damping rate shared by every mode distinguishes acoustics from heat diffusion (\cref{sec:heat}).
We write $\sigma^2_{a,k}$ for the variance of the random initial conditions $(c_k, \beta_k)$.

\paragraph{Truncation and observation model.}
We retain $K = 50$ modes and observe through $M$ microphones at positions $\{x_m\}_{m=1}^M$:
\begin{equation}
\label{eq:obs-full}
y_m(t) = \underbrace{\sum_{k=1}^{K} a_k(t)\,\varphi_k(x_m)}_{\text{retained signal}} + \underbrace{\sum_{n=K+1}^{K_{\mathrm{total}}} a_n(t)\,\varphi_n(x_m)}_{\text{truncation noise } \eta_m^{\mathrm{trunc}}(t)}.
\end{equation}
Stacking $M$ microphones over $T$ snapshots gives
\begin{equation}
\label{eq:obs-stacked}
\tilde{\mathbf{y}} = \tilde{\Phi}\,\mathbf{a}_0 + \tilde{\boldsymbol{\eta}}, \qquad \tilde{\mathbf{y}} \in \mathbb{R}^{MT}, \quad \tilde{\Phi} \in \mathbb{R}^{MT \times 2K},
\end{equation}
where $\tilde{\Phi}$ combines the spatial measurement matrix $\Phi_{mk} = \varphi_k(x_m)$ with the temporal basis of~(\ref{eq:acoustic-dynamics}), and $\mathbf{a}_0$ collects the initial amplitudes $(c_1, \beta_1, \ldots, c_K, \beta_K)$.

\paragraph{Tikhonov estimator.}
The estimate of $\mathbf{a}_0$ is the penalized least-squares solution
\begin{equation}
\label{eq:tikhonov}
\hat{\mathbf{a}} = \arg\min_{\mathbf{a}} \bigl\{ \|\tilde{\mathbf{y}} - \tilde{\Phi}\,\mathbf{a}\|^2 + \alpha\,\mathbf{a}^\top \Gamma\,\mathbf{a} \bigr\} = \bigl(\tilde{\Phi}^\top \tilde{\Phi} + \alpha\,\Gamma\bigr)^{-1} \tilde{\Phi}^\top \tilde{\mathbf{y}},
\end{equation}
where $\alpha > 0$ is the strength and $\Gamma \succ 0$ is a diagonal matrix of per-mode penalties.
Quality is the normalized modal MSE
\begin{equation}
\label{eq:P-modal}
P_{\mathrm{modal}} = \frac{\sum_k \mathbb{E}[(\hat{a}_k - a_{0,k})^2]}{\sum_k \mathrm{Var}(a_{0,k})},
\end{equation}
with expectation and variance over the evaluated snapshots, so $P_{\mathrm{modal}} = 0$ is perfect reconstruction and $1$ matches predicting each mode's mean, which for the zero-mean acoustic amplitudes is guessing zero; we write $P$ hereafter.

\paragraph{Terminology.}
\emph{Ridge} means $\Gamma = I$ ($p{=}0$); the \emph{closed form} means $\Gamma_k = \lambda_k^{|s|}$; the \emph{per-room power-law oracle} means the best $(p^\star, \alpha^\star)$ of the scalar sweep $\Gamma_k = \lambda_k^{p}$ for a given room.
Every oracle in this paper is of this type; no experiment searches over unrestricted per-mode weights.

%% file: sections/04_why_p_equals_s.tex
\section{Why the Prior Sets the Penalty Shape}
\label{sec:theory}

For a single snapshot ($T{=}1$) with isotropic truncation noise, the optimal regularizer has the form $\Gamma_k^\star \propto \lambda_k^s$, where $s > 0$ is the prior spectral decay rate.\footnote{We write $s > 0$ in theoretical statements and $|s|$ in experiments, where the exponent is the absolute value of a fitted slope.}
The argument has three steps: a linear-estimation calculation shows that isotropic noise fixes the shape, an exact identity controls the off-diagonal anisotropy without Berry's conjecture, and Weyl's mode count sets the size of that anisotropy.
Appendix~\ref{app:proof} gives the details.

\subsection{Step 1: isotropic noise fixes the shape}

The truncation noise $\boldsymbol{\eta}_{\mathrm{trunc}}(t) = \sum_{n>K} a_n(t)\,\boldsymbol{\varphi}_n$ has covariance (at $t=0$; uniform damping cancels)
\begin{equation}
\label{eq:R-decompose}
R_{\mathrm{trunc}} = \sum_{n>K} \sigma^2_{a,n}\,\boldsymbol{\varphi}_n \boldsymbol{\varphi}_n^\top = \sigma^2_{\mathrm{trunc}}\bigl(I_M + E\bigr),
\end{equation}
where $\sigma^2_{\mathrm{trunc}} = M^{-1}\sum_{n>K} \sigma^2_{a,n}$ is the average noise power and $E$ is the deviation from isotropy.
When $E = 0$, the prior alone sets the shape:

\begin{proposition}[Isotropy $\Rightarrow$ power-law regularization, single snapshot]
\label{prop:isotropy}
Let $T=1$, so that $\mathbf{y} = \Phi\,\mathbf{a} + \boldsymbol{\eta}$ with $\Phi \in \mathbb{R}^{M \times K}$, $\Phi_{mk} = \varphi_k(x_m)$, and $\mathbf{a} \in \mathbb{R}^K$ the modal amplitudes at the observed snapshot.
Assume $\mathbb{E}[\mathbf{a}] = 0$, $\mathbb{E}[\boldsymbol{\eta}] = 0$, $\mathrm{Cov}(\boldsymbol{\eta}) = \sigma^2 I_M$, $\mathrm{Cov}(\mathbf{a}, \boldsymbol{\eta}) = 0$, and $\mathrm{Cov}(\mathbf{a}) = \Sigma_{\mathbf{a}}$ diagonal with $\Sigma_{kk} = c\,\lambda_k^{-s}$.
Then the linear minimum-MSE estimator of $\mathbf{a}$ is the Tikhonov estimator~\textup{(\ref{eq:tikhonov})} with $\alpha\,\Gamma = \sigma^2\,\Sigma_{\mathbf{a}}^{-1}$, so the optimal shape is $\Gamma^\star_{kk} \propto \lambda_k^s$, i.e., $p^\star = s$, independent of the domain.
If $(\mathbf{a}, \boldsymbol{\eta})$ are moreover jointly Gaussian, this estimator is MMSE-optimal among all estimators.\footnote{The multi-snapshot model ($T>1$) has temporally correlated truncation noise (Appendix~\ref{app:temporal}); we state no $T>1$ corollary, and every multi-snapshot result in Sections~\ref{sec:acoustic} and~\ref{sec:learned} is empirical. At $T{=}1$ the stacked model~(\ref{eq:obs-stacked}) has $2K$ unknowns with prior $\Sigma_{\mathbf{a}} \otimes I_2$, and the proposition applies verbatim with $\tilde{\Phi}$ for $\Phi$.}
\end{proposition}

\begin{proof}
Among linear estimators, the mean-squared error depends on $(\mathbf{a}, \boldsymbol{\eta})$ only through the second moments above, and its minimizer is $\hat{\mathbf{a}} = (\Phi^\top \Phi + \sigma^2 \Sigma_{\mathbf{a}}^{-1})^{-1} \Phi^\top \mathbf{y}$ (Appendix~\ref{app:prop1}).
Comparing with~(\ref{eq:tikhonov}) gives $\alpha\,\Gamma = \sigma^2\,\Sigma_{\mathbf{a}}^{-1}$.
Since $\Sigma_{\mathbf{a}}^{-1}$ is diagonal with entries $c^{-1}\lambda_k^s$, the shape $\Gamma_{kk} \propto \lambda_k^s$ is set by the prior, and the noise level sets only $\alpha$.
Under joint Gaussianity the same estimator is the posterior mean, hence MMSE among all estimators.
\end{proof}

Modes with less energy are penalized more.
\citet{alberti2021learning} show that the optimal shape requires $\Sigma_x$; in modal problems physics gives $\Sigma_x$ analytically, so the remaining question is whether the truncation noise is isotropic.
\begin{remark}[Departures from isotropy]
\label{rem:aniso}
Once isotropy is lost, the noise-whitening term is non-diagonal and the Bayes-linear estimator leaves the diagonal Tikhonov family, so Proposition~\ref{prop:isotropy} is exact only at $E = 0$.
Write the covariance as $\sigma^2(I_M + \Delta)$ with $\Delta$ trace-free, let $\widetilde{P}_\Delta(p) = \min_{\alpha > 0} P(p, \alpha; \Delta)$ be the profiled population risk, and let $p_0$, $p_\Delta$ minimize $\widetilde{P}_0$, $\widetilde{P}_\Delta$.
If $\partial_p \widetilde{P}_\Delta$ is continuously differentiable jointly in $(p, \Delta)$ and $p_0$ is an interior minimum with $\partial_p^2 \widetilde{P}_0(p_0) > 0$, the implicit function theorem gives $p_\Delta - p_0 = O(\|\Delta\|_{\mathrm{op}})$ and $\widetilde{P}_\Delta(p_0) - \widetilde{P}_\Delta(p_\Delta) = O(\|\Delta\|_{\mathrm{op}}^2)$; the constant depends on the room through $\partial_p^2 \widetilde{P}_0(p_0)$, so this is a local sensitivity statement, not a uniform guarantee.
Structured sensor arrays are outside this analysis (\cref{sec:discussion}).
\end{remark}

A controlled sweep through and beyond the most anisotropic rooms observed leaves the median cost of the fixed exponent essentially unchanged at $T \leq 100$; at $T{=}1000$ it stays small up to the mean anisotropy and rises by a few tenths of a point at the strongest levels (Appendix~\ref{app:aniso-sweep}, Table~\ref{tab:aniso-sweep}).

\subsection{Step 2: the off-diagonal anisotropy is exactly $H$}

Is the truncation noise isotropic?
The covariance~(\ref{eq:R-decompose}) is a weighted sum of rank-one matrices $\boldsymbol{\varphi}_n \boldsymbol{\varphi}_n^\top$, one per discarded mode, so the answer depends on the statistics of $\boldsymbol{\varphi}_n = [\varphi_n(x_1), \ldots, \varphi_n(x_M)]^\top$ across modes.
Berry's random-wave conjecture~\citep{berry1977regular} predicts that high-frequency eigenfunctions of generic bounded domains behave like random superpositions of plane waves.
We use only a sensor-level version:

\begin{conjecture}[Sensor-averaged eigenfunction independence]
\label{conj:berry}
For generic sensor placements $\{x_m\}_{m=1}^M$ drawn uniformly in $\Omega \subset \mathbb{R}^2$, the cross-correlations
\begin{equation}
\label{eq:C-kn}
C_{kn} := \frac{1}{M} \sum_{m=1}^M \varphi_k(x_m)\,\varphi_n(x_m)
\end{equation}
satisfy $\mathbb{E}[C_{kn}^2] \approx 1/M$ for $k \ne n$ under the discrete sensor normalization $\sum_m \varphi_n(x_m)^2 \approx 1$ (Appendix~\ref{app:anisotropy}).
\end{conjecture}

The off-diagonal part of the anisotropy does not need the conjecture.
Let $\psi_n = \sqrt{|\Omega|}\,\varphi_n$ and $w_n = \sigma^2_{a,n} / \sum_{m>K} \sigma^2_{a,m}$, so that $E_{ij} = \sum_{n>K} w_n\,\psi_n(x_i)\psi_n(x_j) - \delta_{ij}$ (population normalization, Appendix~\ref{app:anisotropy}).
For sensors drawn independently and uniformly in $\Omega$, the off-diagonal entries satisfy, exactly,
\begin{equation}\label{eq:anisotropy_bound}
\mathbb{E}[E_{ij}^2]
\;=\;
\frac{\sum_{n>K} (\sigma^2_{a,n})^2}{\bigl(\sum_{n>K} \sigma^2_{a,n}\bigr)^2}
\;=\;
H\,, \qquad i \neq j,
\end{equation}
where $H$ is the Herfindahl index of the noise power across discarded modes: the probability that two randomly drawn units of noise power come from the same mode.
The identity uses only $L^2$-orthonormality and the independence of the sensor draws.
The diagonal entries contribute $D = \mathrm{Var}_{x}\bigl[\sum_{n>K} w_n \psi_n(x)^2\bigr]$, the spatial variance of the discarded modes' summed intensity, so that $\mathbb{E}\|E\|_F^2 = M(M{-}1)H + MD$ exactly and $\mathbb{E}\|E\|_{\mathrm{op}}^2 \leq \mathbb{E}\|E\|_F^2$.
Berry's conjecture enters only through $D \approx 2H$, and at $M = 8$ the Berry-free off-diagonal term ($56H$) is the larger of the two.
Measured on the in-scope rooms, $D/(2H)$ exceeds the Berry value by a moderate factor (Table~\ref{tab:isotropy-facts}), which raises $\mathbb{E}\|E\|_F^2$ above the Berry prediction; the excess is cross-mode covariance of squared intensities, not heavier single-mode marginals (Appendix~\ref{app:anisotropy}).
By Jensen, $\mathbb{E}\|E\|_{\mathrm{op}} \lesssim \sqrt{M(M{-}1)H + MD}$, which with the measured $H$ and $D$ at $M = 8$ is an order-of-magnitude estimate consistent with the empirical median across the in-scope rooms (Table~\ref{tab:isotropy-facts}).

\subsection{Step 3: Weyl's law sets $H$}

How large is $H$?
Weyl's law~\citep{weyl1912asymptotische, ivrii2016100} states that in two dimensions
\begin{equation}
\label{eq:weyl}
N(\lambda) \sim \frac{|\Omega|}{4\pi}\,\lambda, \qquad \lambda \to \infty.
\end{equation}
For our rooms this gives hundreds of discarded modes in the median room, a few dozen in the smallest (Table~\ref{tab:isotropy-facts}).
With prior decay $\sigma^2_{a,n} \propto \lambda_n^{-s}$ and $|s| = 1.13$,
\begin{equation}
\label{eq:herfindahl}
H = \frac{\sum_{n>K}\lambda_n^{-2s}}{\bigl(\sum_{n>K}\lambda_n^{-s}\bigr)^2},
\end{equation}
a calculation from eigenvalues and prior weights alone, which for our rooms places $H$ far below one (Table~\ref{tab:isotropy-facts}).
The two classical results therefore play separate roles: Weyl's law fixes the number and density of the discarded modes, which with the prior decay determine $H$, and Berry's conjecture governs only $D$.
At matched $K_{\mathrm{total}}$, generic rooms and rectangles (where Berry's premise fails by construction) have the same $H$, their $D/(2H)$ separates without overlap, and their landscape-flatness medians are nearly identical (Table~\ref{tab:matched-ktotal}).
The residual anisotropy is moderate, and the regularizer shape is insensitive to it (Remark~\ref{rem:aniso}): across the in-scope rooms at $T{=}1000$, higher anisotropy even correlates with lower cost (Figure~\ref{fig:delta-vs-Eop}; Table~\ref{tab:isotropy-facts}), a room-size confound that vanishes once area is controlled for (Appendix~\ref{app:gap-diagnostics}).

\begin{table}[t]
\centering\footnotesize
\caption{\textbf{Only the first of the three layers is a theorem.}
The exact result holds for a single snapshot, whereas the saturation result is empirical and restricted to the tested diagonal methods in the truncation regime.
The no-truncation control of \cref{app:no-trunc-control}, where shape freedom pays again, is the explicit exception to the second layer.}
\label{tab:scope-layers}
\input{tables/T3_scope_layers}\end{table}

\paragraph{Scope of the claims.}
The contribution has three layers (Table~\ref{tab:scope-layers}, Appendix~\ref{app:prop1}): an exact single-snapshot theorem; an empirical saturation result for the tested diagonal methods in the truncation regime; and a learned method outside the diagonal family, which does not consistently improve either.
The no-truncation control of Appendix~\ref{app:no-trunc-control}, where shape freedom pays again, is the exception to the second layer.

%% file: tables/T3_scope_layers.tex
\begin{tabular}{@{}p{0.24\linewidth}p{0.72\linewidth}@{}}
\toprule
Layer & Claim \\
\midrule
Exact theory & Under exact isotropy at $T{=}1$, $\Gamma_k \propto \lambda_k^{|s|}$ is the optimal Tikhonov penalty shape (Proposition~\ref{prop:isotropy}). \\
Tested diagonal saturation & In the truncation regime, the closed form is near the per-room power-law oracle and is not consistently improved by M1--M3, the learned diagonal models tested (Sections~\ref{sec:acoustic} and~\ref{sec:learned}). \\
Beyond the diagonal family & LIR, which couples modes across layers, does not reach the oracle and does not consistently improve on the closed form (\cref{sec:lir}). \\
\bottomrule
\end{tabular}

%% file: sections/05_experiments.tex
\section{Empirical Verification on Acoustic Rooms}
\label{sec:acoustic}

The experiments go beyond Proposition~\ref{prop:isotropy}: the truncation noise is only approximately isotropic, and for $T > 1$ snapshots are correlated in time.
We evaluate $\Gamma_k = \lambda_k^{|s|}$ on the $187$ rooms of a $197$-room validation set that have truncation noise ($K_{\mathrm{total}} > K$), with $K = 50$ and $M = 8$; the $10$ rooms with $K_{\mathrm{total}} \leq K$ lie outside the regime of Proposition~\ref{prop:isotropy} and are reported separately (Appendix~\ref{app:per-room}).
The FDTD data are excited by a broadband click, and the analysis models the excitation as $\sigma^2_{a,k} \propto \lambda_k^{-|s|}$; OLS in log space on the modal variances gives $|\hat{s}| = 1.13$ with a narrow bootstrap interval (Table~\ref{tab:s-population}; procedure and sensitivity in Appendix~\ref{app:estimating-s}).
For each room the evaluation sweeps $p$ over a 61-point grid on $[0, 6]$ and computes $P(p, T)$ at $T \in \{1, 5, 10, 20, 50, 100, 200, 500, 1000, 2100\}$, grid-searching $\alpha$ at each point.
The pattern persists when the truncation rank, sensor count, and excitation exponent change ($K{=}100$, $M{\in}\{8,16\}$, $|s|{=}1.29$; Appendix~\ref{app:v2}), and under heavy-tailed ($t_3$) and correlated ($\rho{=}0.3$) priors, which leave $p^\star$ and the cost of the closed form unchanged in the truncation regime (Appendix~\ref{app:prior-robustness}).

\begin{figure}[t]
  \centering
  \includegraphics[width=\linewidth]{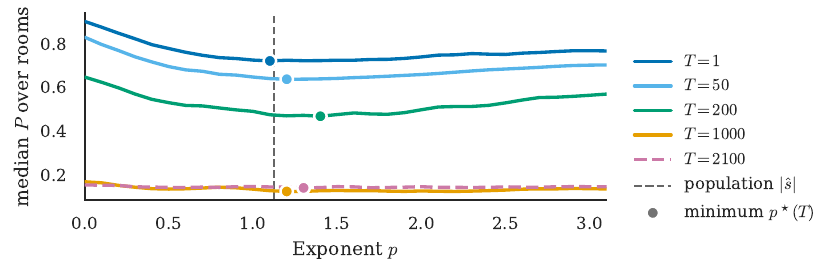}
  \caption{\textbf{The error landscape over the penalty exponent has a broad, shallow minimum that flattens as the window grows.}
  Each curve is the median error over the in-scope rooms for one window, where lower is better and $1$ means no better than guessing zero.
  The minimum moves with $T$, but the basin widens faster than it moves, and the closed-form exponent stays inside every basin.
  This flatness is why $p = |s|$, exact only at $T{=}1$ by \cref{prop:isotropy}, costs little at longer windows.}
  \label{fig:landscape}
\end{figure}

\paragraph{The landscape is flat.}
Figure~\ref{fig:landscape} shows a broad, shallow minimum of the median $P(p)$ at every $T$.
The per-room oracle $p^\star$ drifts upward as $T$ grows, but the basin widens faster than the optimum moves: per-room exponents are most concentrated at intermediate windows (Table~\ref{tab:cost-facts}), and the discarded-mode count supplied by Weyl's law flattens $P(p)$ (Appendix~\ref{app:matched-ktotal}).
Empirically, $p = |s|$ is therefore a principled choice inside a wide near-optimal basin rather than a sharply identified optimum: it is exact under isotropy at $T{=}1$, and the flatness explains why it costs little away from that idealization.
The location is not arbitrary.
Planted exponents are recovered where the landscape has curvature (Appendix~\ref{app:capacity}), a change of excitation moves $|s|$ and the formula follows (Appendix~\ref{app:v2}), and heat diffusion requires the predicted power-law--exponential form (\cref{sec:heat}).

\paragraph{Isotropy check.}
Under Berry's prediction, $|\Omega|\,\varphi_k(x_m)^2$ follows a $\chi^2(1)$ distribution at uniformly random sensor positions.
Pooling all rooms, the Kolmogorov--Smirnov statistic against $\chi^2(1)$ is small (Table~\ref{tab:isotropy-facts}); Appendix~\ref{app:berry} gives per-room, boundary, and worst-room analyses.

\begin{table}[t]
\centering
\footnotesize
\caption{%
  \textbf{Replacing per-room power-law oracle tuning with the closed form costs little at the population median, and most at long observation windows.}
  The relative cost $\delta$ is the extra modal error of the closed form as a fraction of the oracle error, reported as the largest median over the in-scope rooms within each regime.
  \Cref{tab:full-cost} gives the value at every window.
  This is the empirical cost of \cref{prop:isotropy} beyond a single snapshot.%
}
\label{tab:cost_tiers}
\small
\input{tables/T1_cost_tiers}\end{table}

\paragraph{The cost of using $|s|$.}
Table~\ref{tab:cost_tiers} summarizes the relative cost.
At $T{=}1000$, two in-scope rooms in three lose less than $1.1$\,pp absolute, and the worst in-scope absolute cost over all $T$ is several times that, in the prior-dominated regime (Table~\ref{tab:cost-facts}; per-room and per-$T$ breakdowns in Appendices~\ref{app:per-room} and~\ref{app:full-cost}).
As $T$ grows the landscape compresses until per-room ``optimal'' exponents are barely defined.
\Cref{sec:learned} asks whether a learned model finds any remaining structure.

%% file: tables/T1_cost_tiers.tex
\begin{tabular}{@{}lc@{}}
\toprule
Regime & Largest median $\delta$ \\
\midrule
$T \leq 50$ (prior-dominated)    & $0.61\%$ \\
$T \leq 100$ (crossover)         & $1.46\%$ \\
All $T$ (maximum at $T{=}1000$)  & $5.83\%$ \\
\bottomrule
\end{tabular}

%% file: sections/06_can_learning_improve.tex
\section{Learned Diagonal Regularizers Match the Closed Form}
\label{sec:learned}

Proposition~\ref{prop:isotropy} makes $\Gamma_k \propto \lambda_k^{|s|}$ optimal within the diagonal family under exact isotropy at $T{=}1$.
The gap of up to $5.83\%$ between the closed form and the per-room power-law oracle (Table~\ref{tab:cost_tiers}) leaves room for a learned diagonal regularizer to exploit per-room structure.
Three architectures look for it.

\paragraph{Setup.}
All models are trained on $n{=}800$ rooms ($K{=}50$, $M{=}8$, five seeds each; specifications in Appendix~\ref{app:architectures}).
\textsc{M1} (4{,}949 parameters) maps a 10-dimensional per-mode feature vector to a diagonal $\Gamma$ through a small network (CondNet) and solves the Tikhonov problem by $L{=}10$ unrolled gradient steps.
\textsc{M2} (70 parameters) learns one unconditional $\Gamma$ shared by all rooms, with the same unrolled solver.
\textsc{M3} (4{,}930 parameters) uses the CondNet of M1 with a closed-form differentiable linear solve.
The strength $\alpha$ is learned jointly with $\Gamma$ and used at evaluation.
If per-room adaptation helps, M1 and M3 should beat both M2 and the closed form.

\paragraph{Results.}
All three models converge smoothly over 500 epochs (Appendix~\ref{app:training-curves}).
M3 reaches $P$ within a fraction of a percentage point of the closed form at every $T$, slightly below it at $T{=}1$ and slightly above at longer windows (Table~\ref{tab:m3_performance}), although its five seeds learn qualitatively different spectra (Table~\ref{tab:m3-seed-exponents}).
The residual anisotropy therefore opens no gap that the tested models exploit, and the loss is flat across the shapes the seeds reach.
Figure~\ref{fig:learned-spectra} shows the same saturation across architectures.
At $T{=}1000$ M3's per-room exponents match theory in the median, while M2's shared exponent is about half of it; yet the median $P$ at the two exponents is nearly the same (Table~\ref{tab:learned-exponents}).
For M3 this is non-identifiability on a flat basin rather than a capacity limit, since M3 recovers planted exponents where curvature exists (Appendix~\ref{app:capacity}); M2's parameterization recovers planted exponents only up to a ceiling below one (Table~\ref{tab:learned-exponents}), so the basin reading is scoped to M3.
M2 still converges to a power law by SGD without any physics (Table~\ref{tab:learned-exponents}), and its shallower exponent costs nothing in $P$.

\begin{table}[t]
\footnotesize
\centering
    \caption{\textbf{M3 matches the error of the closed form at every tested window and recovers little of the gap to the oracle, although its seeds learn different spectra.}
    $P$ is the normalized modal error, lower is better, and M3 entries are mean and standard deviation over seeds.
    The gap is the closed form's error minus the oracle's, an upper bound on what per-room power-law tuning can recover.
    This is the diagonal saturation that \cref{sec:learned} reports for the truncation regime.}
\label{tab:m3_performance}
\input{tables/T2_m3_performance}\end{table}

Across all three architectures, five training sizes, and three snapshot counts, every valid per-seed evaluation in the truncation regime shows $\Delta P \geq 0$ relative to the oracle (Appendix~\ref{app:failures}).
In the no-truncation control of Appendix~\ref{app:prior-robustness}, outside the regime of Proposition~\ref{prop:isotropy}, M3 checkpoints trained with the same protocol beat the best power law by a clear median margin at $T{=}100$ (Appendix~\ref{app:no-trunc-control}): saturation is a property of the truncation noise, not of the models.
No room descriptor predicts $p^\star$ (Appendix~\ref{app:features}), and no geometric or noise diagnostic predicts the residual gap out of sample (Appendix~\ref{app:gap-diagnostics}).

\paragraph{Fitting $s$ per room does not help.}
The label-free alternative is a per-room slope $s_{\mathrm{room}}$ fitted to each room's spectrum.
Its estimation error exceeds the standard error of the population median by an order of magnitude (Table~\ref{tab:s-population}), and the deconvolved inter-room spread is effectively zero, so plugging in $s_{\mathrm{room}}$ adds more noise than signal: across $T$ it is never clearly better than $|s|$ and is worse at most windows (Table~\ref{tab:sroom-rho-vs-T}).

\subsection{Beyond the diagonal family}
\label{sec:lir}

Can coupling modes through the full $A^\top A$ structure go below the per-room power-law oracle?
Learned Iterative Ridge (LIR) runs $L$ steps of learned gradient descent on the Tikhonov objective with per-layer $(\eta_l, \alpha_l, D_l)$, $52L$ parameters in total (Appendix~\ref{app:lir}).
At no tested depth does LIR reach the oracle (Figure~\ref{fig:learned-spectra}b at $L{=}10$; Table~\ref{tab:lir-ablation}).
At $L{=}10$ it is slightly below the closed form at $T{=}1$ and above it at $T{=}100$ and $T{=}1000$, so no method tested, diagonal or not, consistently improves on the closed form.
Under the approximate isotropy of our data the Bayes-linear estimator lies outside the diagonal family, so a non-diagonal estimator could in principle go below the oracle; LIR as trained here does not.
Classical non-diagonal baselines (Wiener/LMMSE with a diagonal prior and isotropic noise, generalized Tikhonov, TSVD, early-stopped CGLS, Landweber) and per-coordinate shrinkage (Pinsker, empirical Bayes) apply fixed per-mode shrinkage profiles set by the prior and cannot improve on the closed form's shape (Appendix~\ref{app:lir}).

%% file: tables/T2_m3_performance.tex
\begin{tabular}{@{}l ccc c@{}}
\toprule
$T$ & $P(|s|)$ & $P_{\mathrm{M3}}$ & $P_{\mathrm{oracle}}$ & Gap \\
\midrule
1    & $0.7264$ & $0.7245 \pm 0.0021$ & $0.7151$ & $1.13$\,pp \\
100  & $0.6035$ & $0.6051 \pm 0.0027$ & $0.5936$ & $0.99$\,pp \\
1000 & $0.1308$ & $0.1339 \pm 0.0003$ & $0.1224$ & $0.84$\,pp \\
\bottomrule
\end{tabular}

%% file: sections/07_cross_physics_validation.tex
\section{Extension to Heat Diffusion}
\label{sec:heat}

Does the framework depend on the wave equation?
Heat diffusion is a different process, and it fits the same framework with one modification that the physics fixes.
For heat the estimation target is the current-state amplitude $a_k(t) = a_k(0)\,e^{-\kappa\lambda_k t}$ at each terminal snapshot; recovering $a_k(0)$ is exponentially ill-posed and is not attempted (Appendix~\ref{app:heat-why-different}).

\paragraph{The amplitude spectrum gains an exponential factor.}
In acoustics the amplitude variance is a power law, $\sigma^2_{a,k} \propto \lambda_k^{-|s|}$.
For heat the Green's function multiplies it by $e^{-2\kappa\lambda_k t}$:
\begin{equation}
\label{eq:heat-spectrum}
\sigma^2_{a,k}(t) \;\propto\; \lambda_k^{-s}\, e^{-c\,\lambda_k}, \qquad c = 2\kappa t.
\end{equation}
A power law $\Gamma_k = \lambda_k^p$ cannot represent this roll-off.
The optimal regularizer is $\Gamma_k \propto \lambda_k^s\, e^{c\,\lambda_k}$, the prior's exponential factor inverted, as $\Gamma \propto \Sigma_a^{-1}$ requires.

\paragraph{The second parameter is known.}
For each room and observation time, OLS on $\log \hat{\sigma}^2_{a,k}(t) = \beta_0 - |s|\log\lambda_k - c\,\lambda_k$ recovers $(\hat{s}, \hat{c})$ with a close fit at all but the earliest observation times (Appendix~\ref{app:heat-fitting}; Figure~\ref{fig:heat-2d-sweep} shows the landscapes).
Across five diagnostic rooms, per-room regression slopes of $\hat{c}$ against $c_{\mathrm{theory}}$ lie close to one with near-perfect fits and small intercepts (Table~\ref{tab:per-room-c-fit}).
The factor $c = 2\kappa t$ is not fitted: the experimenter chooses $t$, and $\kappa$ is a material constant ($\kappa = 1$ in our synthetic units).
Only $|s|$ is estimated from data; Appendix~\ref{app:heat-1vs2} reports the one-parameter fallback.

\begin{figure}[t]
  \centering
  \includegraphics[width=\linewidth]{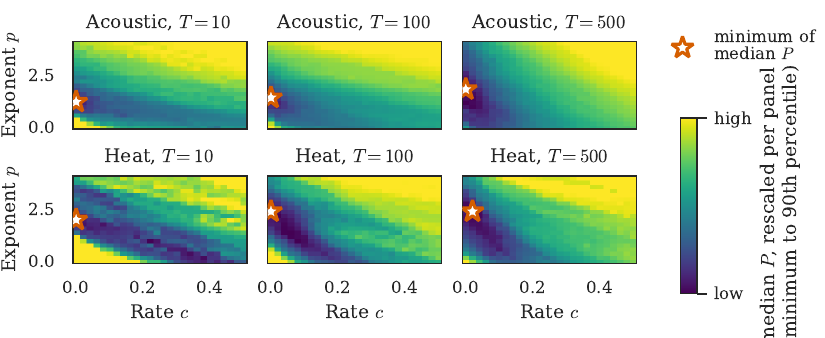}
  \caption{\textbf{Heat and acoustic landscapes differ in shape, not only in where the optimum lies.}
  Each map is the median error over $p$ and $c$ at one window, coloured from its own minimum upward so shapes, not levels, compare.
  Acoustic optima sit at $c = 0$, where a power law suffices (top).
  Heat optima sit near $c = 0$ at a larger $p$, a steeper power law standing in for the predicted exponential factor (bottom).
  This is the correction of \cref{eq:heat-spectrum}, whose rate \cref{fig:c-vs-t} checks.}
  \label{fig:heat-2d-sweep}
\end{figure}

\paragraph{Same framework, different physics.}
For the synthetic heat data $|s| \approx 1.0$ rather than $1.13$, because the excitation statistics differ; what carries over is the structure, a power law from the initial conditions composed with a PDE-specific correction known from theory.
The fitted spectra follow the predicted rate (Table~\ref{tab:per-room-c-fit}), but the two-parameter regularizer with $c = 2\kappa t$ at each snapshot is not evaluated: the 2-D sweep uses one $c$ for all snapshots of a window and finds its optimum near $c = 0$ with a steeper power law (Figure~\ref{fig:heat-2d-sweep}).
The simulated heat data contain no discarded modes, so the late-time breakdown of isotropy that discarded heat modes would cause (Appendix~\ref{app:heat_herfindahl}) is an analytic prediction these data do not test.
The heat data are generated by the same Green's function the theory predicts, so this is a consistency check on the fitting procedure, not physical validation.

%% file: sections/08_discussions_and_limitations.tex
\section{Conclusion}
\label{sec:discussion}

Under isotropic truncation noise, the optimal diagonal regularizer for a modal inverse problem is set by the excitation, not by the room: it is $\Gamma_k = \lambda_k^{|s|}$.
The statement is exact at $T{=}1$ (Proposition~\ref{prop:isotropy}) and empirical for $T > 1$, where the closed form stays within $5.83\%$ of per-room power-law oracle tuning at the population median.
The discarded modes explain why: Weyl's count makes the truncation noise nearly isotropic, Berry's conjecture affects only a measured diagonal term, and the resulting landscape is flat.
None of the learned architectures we tested consistently improves on the closed form in the truncation regime, including Learned Iterative Ridge, which couples modes across layers (\cref{sec:learned,sec:lir}); this bounds the methods tested, not learning in general.

\paragraph{Scope and limitations.}
Our rooms are random convex 2D polygons; the Gaussian prior is tested against heavy-tailed and correlated alternatives (Appendix~\ref{app:prior-robustness}).
A 3D stress test on analytic boxes (Appendix~\ref{app:box3d}) finds $H$ and landscape flatness in line with the 2D rooms at $T \leq 100$, but boxes are separable, so transfer to non-separable 3D rooms is an open hypothesis, and realistic sensing pipelines are untested.
Sensors are drawn i.i.d.\ uniformly; structured arrays (linear, circular, spherical) may introduce eigenfunction correlations the analysis does not cover.
Rectangular rooms, where Berry's premise fails analytically, are nearly as flat as generic rooms at matched discarded-mode counts (Appendix~\ref{app:weyl-dominance}).
Electronic sensor noise only shifts $\alpha$, and frequency-dependent damping composes with $|s|$; Appendix~\ref{app:sensor-noise} gives diagnostics for when the closed form should be trusted.
Real recordings are untested: a static compact array cannot resolve the retained modes at any room size the framework covers (Appendix~\ref{app:aperture}), so establishing the slope and the isotropy on real data needs a wider or moving array.

%% file: supplementary/A_extended_proof.tex

\section{Exact Theorem and Approximate Isotropy Argument}
\label{app:proof}

The argument of \cref{sec:theory} has an exact part, Proposition~\ref{prop:isotropy} at $T = 1$ (\cref{app:prop1}), and an approximate part, the isotropy of the truncation noise (\crefrange{app:anisotropy}{app:herfindahl}).
Whether the approximation is adequate is an empirical question (\cref{sec:acoustic}, Appendix~\ref{app:berry}, and \cref{app:aniso-sweep}).

\paragraph{Notation.}
$K$ retained modes, $M$ microphones, eigenvalues $\lambda_k$, eigenfunction values $\varphi_k(x_m)$ collected in $\boldsymbol{\varphi}_k = [\varphi_k(x_1), \ldots, \varphi_k(x_M)]^\top \in \mathbb{R}^M$, and prior variances $\sigma^2_{a,k} = c\,\lambda_k^{-s}$ with decay exponent $s > 0$.

\subsection[Derivation of Proposition 1 at T = 1]{Derivation of Proposition~\ref{prop:isotropy} at $T = 1$}
\label{app:prop1}

At $T = 1$ the dimensions of the statement and the proof coincide: $\mathbf{y} \in \mathbb{R}^M$, $\Phi \in \mathbb{R}^{M \times K}$ with $\Phi_{mk} = \varphi_k(x_m)$, $\mathbf{a} \in \mathbb{R}^K$ holds the modal amplitudes at the observed snapshot, and $\boldsymbol{\eta} \in \mathbb{R}^M$ is the truncation noise at that snapshot.
The proposition does not cover the multi-snapshot model of eq.~(\ref{eq:obs-stacked}), whose noise is temporally correlated (\cref{app:temporal}).

\paragraph{Step (i): model and assumptions.}
The single-snapshot observation is
\begin{equation}
\mathbf{y} = \Phi\,\mathbf{a} + \boldsymbol{\eta}.
\end{equation}
The prior has $\mathbb{E}[\mathbf{a}] = \mathbf{0}$ and diagonal $\mathrm{Cov}(\mathbf{a}) = \Sigma_{\mathbf{a}}$ with $(\Sigma_{\mathbf{a}})_{kk} = c\,\lambda_k^{-s}$: modes are uncorrelated, and low-frequency modes carry more energy than high-frequency ones.
The noise has $\mathbb{E}[\boldsymbol{\eta}] = \mathbf{0}$, $\mathrm{Cov}(\boldsymbol{\eta}) = \sigma^2 I_M$, and $\mathrm{Cov}(\mathbf{a}, \boldsymbol{\eta}) = 0$, which holds when the discarded-mode amplitudes are uncorrelated with the retained ones.
Isotropy is the key assumption (\cref{app:anisotropy}); no distributional assumption enters.

\paragraph{Step (ii): best linear estimator.}
For a linear estimator $\hat{\mathbf{a}} = W\mathbf{y}$ the mean-squared error is
\begin{equation}
\mathbb{E}\|W\mathbf{y} - \mathbf{a}\|^2 = \mathrm{tr}\bigl[(W\Phi - I)\,\Sigma_{\mathbf{a}}\,(W\Phi - I)^\top\bigr] + \sigma^2\,\mathrm{tr}\bigl[W W^\top\bigr],
\end{equation}
which depends on the joint law of $(\mathbf{a}, \boldsymbol{\eta})$ only through the second moments of Step~(i).
Setting its gradient with respect to $W$ to zero gives the linear minimum-MSE (LMMSE) estimator
\begin{equation}
\label{eq:lmmse}
\hat{\mathbf{a}}_{\mathrm{LMMSE}} = \Sigma_{\mathbf{a}}\Phi^\top\bigl(\Phi\Sigma_{\mathbf{a}}\Phi^\top + \sigma^2 I_M\bigr)^{-1}\mathbf{y}
= \bigl(\Phi^\top\Phi + \sigma^2\Sigma_{\mathbf{a}}^{-1}\bigr)^{-1}\Phi^\top\mathbf{y},
\end{equation}
where the second form follows from the push-through identity $\Sigma_{\mathbf{a}}\Phi^\top(\Phi\Sigma_{\mathbf{a}}\Phi^\top + \sigma^2 I_M)^{-1} = (\Phi^\top\Phi + \sigma^2\Sigma_{\mathbf{a}}^{-1})^{-1}\Phi^\top$.

\paragraph{Step (iii): connection to Tikhonov.}
At $T = 1$ the Tikhonov estimator of eq.~(\ref{eq:tikhonov}) is $\hat{\mathbf{a}}_{\mathrm{Tikh}} = (\Phi^\top \Phi + \alpha\,\Gamma)^{-1} \Phi^\top \mathbf{y}$, which coincides with eq.~(\ref{eq:lmmse}) if and only if
\begin{equation}
\label{eq:alpha-gamma}
\alpha\,\Gamma = \sigma^2\,\Sigma_{\mathbf{a}}^{-1}.
\end{equation}
Since $(\Sigma_{\mathbf{a}}^{-1})_{kk} = c^{-1}\,\lambda_k^{s}$, this gives $\Gamma_{kk} = \sigma^2 \lambda_k^{s} / (\alpha\,c)$.
The prefactor is absorbed into $\alpha$, so the shape is $\Gamma_{kk} \propto \lambda_k^s$, i.e., $p^\star = s$.

\paragraph{Step (iv): what Gaussianity adds.}
If $(\mathbf{a}, \boldsymbol{\eta})$ are jointly Gaussian, the posterior $p(\mathbf{a} \mid \mathbf{y})$ is Gaussian, and completing the square in $-\tfrac{1}{2\sigma^2}\|\mathbf{y} - \Phi\mathbf{a}\|^2 - \tfrac{1}{2}\mathbf{a}^\top\Sigma_{\mathbf{a}}^{-1}\mathbf{a}$ shows that its mean is eq.~(\ref{eq:lmmse}).
The same estimator is then the posterior mean, hence minimum-MSE among all estimators, linear or not.
Gaussianity is needed only to upgrade LMMSE-optimality to unrestricted MMSE-optimality: the shape conclusion is the same in both cases, and the second-moment version is what makes the population-risk comparisons of Appendix~\ref{app:prior-robustness} exact.

\paragraph{Why the prior alone sets the shape.}
The forward operator $\Phi^\top \Phi$ couples modes in the estimator, but it does not enter the penalty shape, which is $\Sigma_{\mathbf{a}}^{-1}$; the noise level $\sigma^2$ and $\Phi^\top\Phi$ affect only $\alpha$.
\qed

\paragraph{Scope of optimality.}
Under exact isotropy ($E=0$) at $T = 1$ and a Gaussian power-law prior, the closed form is MMSE-optimal among all estimators (Step~(iv)).
In our rooms isotropy is approximate (Table~\ref{tab:isotropy-facts}), so the closed form is not exactly MMSE: the Bayes-linear estimator under the true covariance is non-diagonal and couples modes through $E$.
Throughout the paper, the ``per-room power-law oracle'' is the best estimator within the scalar power-law family $(p^\star, \alpha^\star)$, not the unconstrained Bayes estimator or an unrestricted per-mode search.
The saturation result of \cref{sec:learned} is a statement about the diagonal family: the tested learned models do not improve on $\Gamma_k = \lambda_k^{|s|}$ in the truncation regime, and LIR (\cref{sec:lir}), a learned method outside that family, does not consistently improve on it either (Table~\ref{tab:scope-layers}).

\subsection{Anisotropy: identities and bound}
\label{app:anisotropy}
\label{app:bernstein}

\begin{table}[t]
\centering
\small
\caption{\textbf{The discarded modes are many, their Herfindahl index is small, and the measured anisotropy is moderate.}
These population values back the isotropy argument of \cref{sec:theory}.
Weyl's count sets the number of discarded modes and hence $H$, Berry's conjecture enters only through $D$, and the measured operator norm of $E$ agrees in order of magnitude with the Jensen bound.}
\label{tab:isotropy-facts}
\input{tables/T32_isotropy_facts}\end{table}

Proposition~\ref{prop:isotropy} requires isotropic noise.
Write the truncation-noise covariance of eq.~(\ref{eq:R-decompose}) as $R_{\mathrm{trunc}} = \sum_{n>K} \sigma^2_{a,n}\,\boldsymbol{\varphi}_n\boldsymbol{\varphi}_n^\top = \sigma^2_{\mathrm{trunc}}(I_M + E)$, with average noise power $\sigma^2_{\mathrm{trunc}} = M^{-1}\sum_{n>K} \sigma^2_{a,n}$, so that $E$ is the deviation from isotropy; the goal is a bound on $\|E\|_{\mathrm{op}}$.
Most of the bound needs no input from Berry's conjecture: the off-diagonal second moments of $E$ obey an exact identity, and the conjecture enters only through the diagonal term $D$.

\paragraph{The entries of $E$.}
The identities below use the \emph{population} normalization, which divides $R_{\mathrm{trunc}}$ by $\bar{\sigma}^2 = M^{-1}\,\mathbb{E}[\mathrm{tr}\,R_{\mathrm{trunc}}] = |\Omega|^{-1}\sum_{n>K}\sigma^2_{a,n}$.
With $\psi_n = \sqrt{|\Omega|}\,\varphi_n$ (unit second moment under uniform sensors) and weights $w_n = \sigma^2_{a,n} / \sum_{m>K} \sigma^2_{a,m}$, the population-normalized anisotropy matrix $E = E^{\mathrm{pop}}$ has entries
\begin{equation}
\label{eq:Epop}
E_{ij} = \sum_{n>K} w_n\,\psi_n(x_i)\,\psi_n(x_j) - \delta_{ij}.
\end{equation}
Uniform sensors and $L^2$-orthonormality give the population identities $\mathbb{E}_x[\psi_n(x)^2] = 1$ and $\mathbb{E}_x[\psi_n(x)\,\psi_m(x)] = \delta_{nm}$.

\paragraph{Off-diagonal entries: an exact identity.}
For $i \neq j$,
\begin{equation}
\mathbb{E}[E_{ij}^2]
= \sum_{n>K}\sum_{n'>K} w_n w_{n'}\,\mathbb{E}\bigl[\psi_n(x_i)\,\psi_{n'}(x_i)\bigr]\,\mathbb{E}\bigl[\psi_n(x_j)\,\psi_{n'}(x_j)\bigr]
= \sum_{n>K} w_n^2
= H,
\end{equation}
because $x_i$ and $x_j$ are independent draws (the expectation factorizes) and each factor equals $\delta_{nn'}$ by orthonormality.
The identity uses no property of Berry's conjecture and holds for any $L^2$-orthonormal eigenbasis, including the separable sine products of Appendix~\ref{app:weyl-dominance}.
Here $H$ is the Herfindahl index
\begin{equation}
H := \frac{\sum_{n>K}(\sigma^2_{a,n})^2}{\bigl(\sum_{n>K}\sigma^2_{a,n}\bigr)^2}.
\end{equation}
A Monte Carlo with i.i.d.\ Gaussian random-wave values reproduces the identity to within its sampling error, and a sensor-draw check on the room eigenfunctions agrees (Table~\ref{tab:isotropy-facts}).

\paragraph{Diagonal entries: where Berry enters.}
Write $V(x) = \sum_{n>K} w_n\,\psi_n(x)^2$, so that $E_{ii} = V(x_i) - 1$ with $\mathbb{E}[V] = 1$ and
\begin{equation}
\mathbb{E}[E_{ii}^2] = \mathrm{Var}_x[V] =: D
= \underbrace{\sum_{n>K} w_n^2\,\mathrm{Var}_x[\psi_n^2]}_{S\ \text{(single-mode)}}
+ \underbrace{\sum_{\substack{n \neq m \\ n, m > K}} w_n w_m\,\mathrm{Cov}_x\bigl(\psi_n^2, \psi_m^2\bigr)}_{C\ \text{(cross-mode)}}.
\end{equation}
Under Berry's random-wave statistics, $\psi_n(x)$ is a unit-variance Gaussian, so $\mathrm{Var}[\psi_n^2] = 2$, squared intensities of distinct modes are uncorrelated, and $D \approx 2H$; this is the only point at which the conjecture enters.
Measured with exact degree-4 FEM quadrature across the in-scope rooms, the diagonal term is about twice Berry's value (Table~\ref{tab:isotropy-facts}).
The excess comes from covariance between different modes, not from heavier single-mode marginals: the single-mode part sits slightly below Berry's value (Table~\ref{tab:isotropy-facts}).
In separable domains, where Berry's premise fails by construction, $D/(2H)$ is several times larger for the rectangles of Appendix~\ref{app:weyl-dominance} and larger still for the 3D boxes of Appendix~\ref{app:box3d}, growing with mode count (Table~\ref{tab:isotropy-facts}).

\paragraph{Assembling the bound.}
Summing the $M(M{-}1)$ off-diagonal and the $M$ diagonal entries gives, exactly,
\begin{equation}
\label{eq:frob-exact}
\mathbb{E}\|E\|_F^2 = M(M{-}1)\,H + M\,D, \qquad \mathbb{E}\|E\|_{\mathrm{op}}^2 \leq \mathbb{E}\|E\|_F^2;
\end{equation}
Berry's $D \approx 2H$ turns this into $M(M{+}1)H$.
At $M = 8$ the Berry-free off-diagonal term ($56H$) is the larger one, so the measured $D$, about twice Berry's, raises the full Frobenius second moment by about a quarter (Table~\ref{tab:isotropy-facts}).
A sensor-draw Monte Carlo reproduces both terms of eq.~(\ref{eq:frob-exact}) closely and puts $\|E\|_{\mathrm{op}}^2$ at about half of $\mathbb{E}\|E\|_F^2$ at the median, which measures the looseness of the Frobenius relaxation (Table~\ref{tab:isotropy-facts}).

\paragraph{Normalization conventions.}
The identities above use the population normalization $\bar{\sigma}^2$.
The reported operator norms (Table~\ref{tab:isotropy-facts}) use the \emph{realized-trace} normalization instead: it divides by $\mathrm{tr}(R_{\mathrm{trunc}})/M$ at the actual sensor draw and yields the trace-free matrix $E_{\mathrm{tf}} = \bar{E}/\bigl(1 + \mathrm{tr}(E^{\mathrm{pop}})/M\bigr)$, where $\bar{E} = E^{\mathrm{pop}} - \mathrm{tr}(E^{\mathrm{pop}})\,I_M/M$ is the trace-free projection.
For $\bar{E}$ the identity $\mathbb{E}\|\bar{E}\|_F^2 = (M{-}1)(MH + D)$ is exact, and $E_{\mathrm{tf}}$ differs from $\bar{E}$ only by the scalar trace factor, whose spread over sensor draws is small (Table~\ref{tab:isotropy-facts}).
$H$, $D$, and every ratio built from them are invariant to the $\varphi_n$ normalization.

\paragraph{From entries to operator norm.}
With $M = 8$ and the median $H$, Jensen's inequality gives the order-of-magnitude estimate $\mathbb{E}\|E\|_{\mathrm{op}} \lesssim \sqrt{M(M{-}1)H + MD}$ with the measured $D$, or $\sqrt{M(M{+}1)H}$ under $D = 2H$ (Table~\ref{tab:isotropy-facts}); because $H$ varies across rooms, the estimate is tight only at the median.
Across $187$ rooms\footnote{Ten rooms with $K_{\mathrm{total}} \leq 50$ have no discarded modes from which to compute $\|E\|_{\mathrm{op}}$ and are excluded here; all $197$ rooms are retained for the per-room exponent comparison in Appendix~\ref{app:sroom-vs-pop}.} (eigenfunctions evaluated at the sensor positions, realized-trace normalization $E_{\mathrm{tf}}$), the empirical $\|E\|_{\mathrm{op}}$ has a median below one and a long right tail (Table~\ref{tab:isotropy-facts}).
For the median room the Frobenius estimate is close (Table~\ref{tab:isotropy-facts}); for the worst rooms the Frobenius-to-operator-norm relaxation is loose.
The right-tail rooms are small rooms with few discarded modes and large $H$ (the extreme, a room with a single discarded mode, reaches the maximum $M{-}1$), not the rooms with the largest Berry violations: across rooms, $D/(2H)$ is nearly uncorrelated with $\|E\|_{\mathrm{op}}$, and the worst-Berry rooms of \cref{app:worst-rooms} have ordinary $D/(2H)$ (Table~\ref{tab:isotropy-facts}).

\paragraph{Signal and noise in modal coordinates.}
The noise is moderately anisotropic, not negligibly so.
In a decoupled modal model with prior variance $\lambda_k^{-|s|}$ and modal-noise variance $\propto \lambda_k^{q}$, the Wiener shrinkage is $1/(1 + c\,\lambda_k^{|s| + q})$, so the matching Tikhonov exponent is $p = |s| + q$; this is exact in the decoupled model and an approximation for the coupled forward operator.
Figure~\ref{fig:noise-profile} gives small slopes of either sign by two methods (Table~\ref{tab:isotropy-facts}), both an order of magnitude inside the basin width of Appendix~\ref{app:brent}, so we apply no correction.
Anisotropy across sensor directions does not enter this exponent at first order; its second-order cost is the subject of Remark~\ref{rem:aniso} and \cref{app:aniso-sweep}.

\paragraph{Cost, not drift, is the criterion.}
The oracle exponent $p^\star$ drifts upward as $T$ grows (Table~\ref{tab:cost-facts}), but the criterion is whether staying at $|s|$ is costly.
Because the landscape is flat (\cref{sec:acoustic}), the population-median relative cost peaks at $5.83\%$ at $T = 1000$, and at $T = 1000$ two in-scope rooms in three lose less than $1.1$\,pp absolute (Table~\ref{tab:cost-facts}).
Anisotropy does not explain the rooms with elevated~$\delta$: the worst-cost rooms are large rooms with low anisotropy, and the raw correlation of Figure~\ref{fig:delta-vs-Eop} (Table~\ref{tab:isotropy-facts}) disappears once room area is controlled for (Appendix~\ref{app:gap-diagnostics}).
The controlled sweep of \cref{app:aniso-sweep} confirms directly that anisotropy up to the mean observed level leaves the median cost of the fixed exponent essentially unchanged, and that stronger anisotropy raises it noticeably only at the longest window (Table~\ref{tab:aniso-sweep}).

\begin{figure}[t]
  \centering
  \includegraphics[width=\linewidth]{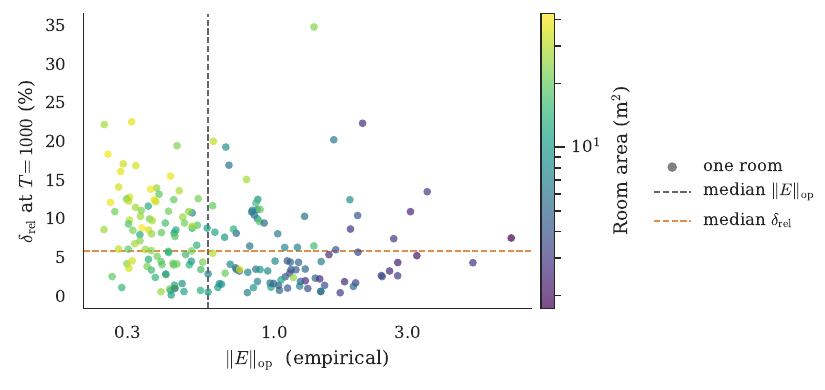}
  \caption{\textbf{Rooms with higher anisotropy pay \emph{lower} cost, a confound of room size.}
    Each point is an in-scope room, $\|E\|_{\mathrm{op}}$ is zero for isotropic truncation noise, and $\delta_{\mathrm{rel}}$ is the relative cost of the closed form against the per-room power-law oracle.
    Large rooms have many discarded modes and hence low anisotropy, yet they tend to pay more than small rooms.
    Once room area is controlled for, no correlation survives (\cref{app:gap-diagnostics}), so anisotropy does not explain which rooms pay the cost.}
  \label{fig:delta-vs-Eop}
\end{figure}

\begin{figure}[t]
  \centering
  \includegraphics[width=\linewidth]{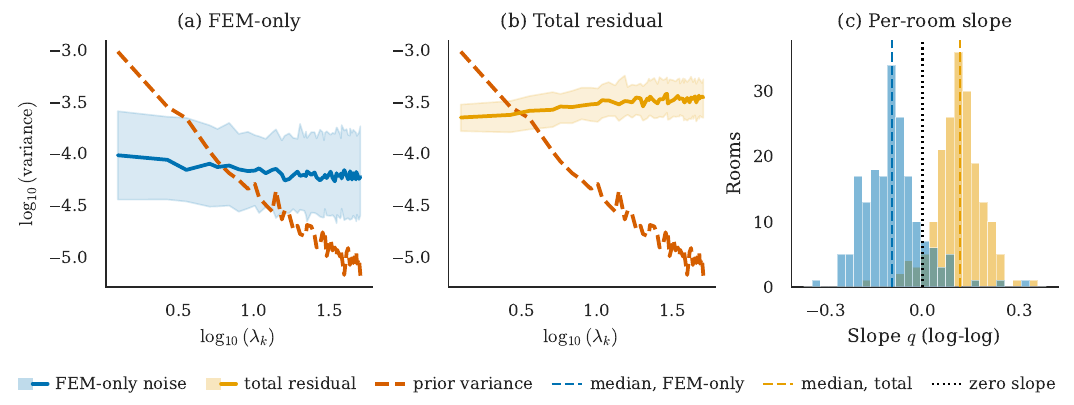}
  \caption{\textbf{Truncation noise is approximately isotropic across retained modes.}
  Estimated from the FEM residual alone, its median variance barely changes with eigenvalue while the prior variance decays (a).
  Estimated from the total residual against the spatial prior, it is also nearly flat (b).
  Per-room slopes of both estimates cluster around zero, the value for noise that is flat across modes (c).
  Both slopes lie well inside the basin width of \cref{app:brent}, so the closed form needs no correction.}
  \label{fig:noise-profile}
\end{figure}

\subsection{Herfindahl index: numerical values}
\label{app:herfindahl}

For the discarded-mode weights $\sigma^2_{a,n} = c\,\lambda_n^{-s}$ with $|s|=1.13$, the Herfindahl index $H = \sum_{n>K}\lambda_n^{-2s}/(\sum_{n>K}\lambda_n^{-s})^2$ is an order of magnitude smaller in the large decagons, which discard about a thousand modes, than in the small triangles, which discard a few dozen (Table~\ref{tab:herfindahl-values}).
For a representative octagon, the effective contributor count $1/H$ is a large fraction of its discarded modes.

$H$ does not vanish as rooms grow: because $|s|>1$, both $\sum n^{-2|s|}$ and $\sum n^{-|s|}$ converge as $K_{\mathrm{total}}\to\infty$, so $H$ tends to a small positive constant.
Its smallness at our operating point is a finite-sample property of the eigenvalue distribution, not an asymptotic guarantee.

\begin{table}[t]
\centering
\small
\caption{\textbf{The Herfindahl index is small in every room and smallest where the most modes are discarded.}
Its spread across rooms follows the discarded-mode count, a representative room keeps a large effective contributor count, and the limit for infinitely many modes is positive but small.
This is the quantity that \cref{eq:anisotropy_bound} ties to the off-diagonal anisotropy.}
\label{tab:herfindahl-values}
\input{tables/T35_herfindahl_values}\end{table}

\subsection{Temporal stacking}
\label{app:temporal}

The analysis of \crefrange{app:prop1}{app:anisotropy} applies at $T=1$.
For $T>1$ the stacked noise covariance is $\mathrm{Cov}(\tilde\eta) = \sum_{n>K} \sigma^2_{a,n}\,G_n \otimes (\boldsymbol{\varphi}_n\boldsymbol{\varphi}_n^\top)$, where $G_n$ is a temporal rank-two matrix encoding the damped-sinusoidal evolution of mode $n$ from eq.~(\ref{eq:acoustic-dynamics}) (one outer product each for the cosine and sine components).
The spatial factors still concentrate to $I_M$ in aggregate by the argument of \cref{app:anisotropy}, but a temporal structure remains: off-diagonal entries $\sum_n \sigma^2_{a,n}\,g_n(t)g_n(t')$ couple snapshots through shared initial conditions.
This temporal correlation shifts the optimal exponent $p^\star(T)$ away from $|s|$ as $T$ grows; \cref{sec:acoustic} measures its cost (peak median $5.83\%$ at $T = 1000$, Table~\ref{tab:cost_tiers}; over all rooms the median cost falls again at the longest window as the damped signal becomes uninformative regardless of regularizer, Table~\ref{tab:full-cost}).
We state no isotropic corollary for $T > 1$: Proposition~\ref{prop:isotropy} is a single-snapshot statement, and the multi-snapshot results of Sections~\ref{sec:acoustic} and~\ref{sec:learned} are empirical.

\subsection{Controlled anisotropy sweep}
\label{app:aniso-sweep}

Because Remark~\ref{rem:aniso} is local and room-dependent, a controlled sweep measures the cost of anisotropy directly.
The sweep replaces the empirical truncation-noise covariance of $24$ in-scope rooms with positive-definite, fixed-power covariances $R = \sigma^2(I_M + a\,\hat{E})$, $\mathrm{tr}\,\hat{E} = 0$, and evaluates the exact population risk of the diagonal Tikhonov family on the paper's $61 \times 12$ $(p, \alpha)$ grid.
The direction $\hat{E}$ comes from three families (each room's own empirical direction, a balanced spectrum, and a single spike that reaches a $9{:}1$ eigenvalue ratio at $a = 3.5$) in random and forward-aligned orientations, and $a$ rises from $0$ through the empirical median, mean, and 95th-percentile markers of $\|E_{\mathrm{tf}}\|_{\mathrm{op}}$ (Table~\ref{tab:isotropy-facts}), up to $3.5$ in the spike family.

Table~\ref{tab:aniso-sweep} reports the cost of the fixed exponent $p = |s|$ relative to the isotropic baseline of the same rooms.
Up to the mean marker the median cost barely moves at any $T$ and the worst sampled orientation adds at most half a point; beyond it the same holds at $T \leq 100$, while at $T = 1000$ the 95th-percentile and spike levels raise the median cost by a few tenths of a point and the worst sampled orientation by more than a point (Table~\ref{tab:aniso-sweep}).
The optimal exponent stays within a tenth of $|s|$ in every cell that retains the full room set.
Medians use those cells; in the empirical-direction family the positive-definiteness limit leaves only a few rooms at the largest $a$, and those reduced cells rise further.
The sweep isolates sensor-space anisotropy with temporally white noise, so it does not replace the temporally correlated multi-snapshot experiments of \cref{sec:acoustic}, and structured sensor arrays remain outside it.

\begin{table}[t]
\centering\small
\caption{\textbf{Anisotropy up to the observed mean barely moves the cost of the fixed exponent, and stronger anisotropy matters only at the longest window.}
Entries are the largest increases over the sweep in the relative cost of $p = |s|$ against the per-room power-law oracle, measured from the isotropic baseline of the same in-scope rooms; medians use the cells with the full room set.
The sweep tests the local sensitivity of \cref{rem:aniso} directly, for sensor-space anisotropy with temporally white noise.}
\label{tab:aniso-sweep}
\input{tables/T4_aniso_sweep}\end{table}

%% file: tables/T32_isotropy_facts.tex
\begin{tabular}{@{}>{\raggedright\arraybackslash}p{0.61\linewidth}>{\raggedright\arraybackslash}p{0.35\linewidth}@{}}
\toprule
Quantity & Value \\
\midrule
Discarded modes $K_{\mathrm{total}} - K$, median over the 197 validation rooms & 263 \\
Discarded modes, median over the 187 in-scope rooms & 287 \\
Herfindahl index $H$ of \cref{eq:herfindahl}, median & 0.005 \\
$D/(2H)$, median [IQR] & 2.16 [1.70, 2.72] \\
Excess of $\mathbb{E}\|E\|_F^2$ over the Berry value $D = 2H$ & 26\% \\
Jensen bound on $\mathbb{E}\|E\|_{\mathrm{op}}$ at $M = 8$ (under $D = 2H$) & 0.67 (0.60) \\
Empirical $\|E\|_{\mathrm{op}}$, median / mean / 95th percentile & 0.58 / 0.88 / 2.42 \\
Spearman $\rho(\|E\|_{\mathrm{op}}, \delta)$ at $T = 1000$ & $-0.30$ ($p < 10^{-4}$) \\
Pooled KS statistic against $\chi^2(1)$ ($N = 77{,}968$ samples) & 0.037 \\
Monte Carlo check of the off-diagonal identity, $\mathbb{E}[E_{ij}^2]/H$ (50{,}000 i.i.d.\ Gaussian random-wave draws, 6 rooms) & $1.00 \pm 0.01$ \\
Single-mode part $S/(2H)$ & 0.877 \\
$D/(2H)$ in separable domains: rectangles (\cref{app:weyl-dominance}) / 3D boxes (\cref{app:box3d}) & 6.6--8.3 / 15--76 \\
Sensor-draw Monte Carlo (32 rooms, 500 draws each): off-diagonal / diagonal term ratio & 0.97 / 0.99 \\
$\|E\|_{\mathrm{op}}^2 / \mathbb{E}\|E\|_F^2$ at the median & 0.51 \\
Std of the trace factor between $E_{\mathrm{tf}}$ and $\bar{E}$ over sensor draws (median over 30 rooms) & 0.08 \\
Frobenius estimate against the empirical median $\|E\|_{\mathrm{op}}$ & within 16\% \\
Spearman $\rho(D/(2H), \|E\|_{\mathrm{op}})$ & $-0.16$ \\
$D/(2H)$ of the worst-Berry rooms (\cref{app:worst-rooms}) & 1.4--2.5 \\
Median per-room noise slope $q$, FEM-only / total residual (\cref{fig:noise-profile}) & $-0.09$ / $+0.12$ \\
\bottomrule
\end{tabular}

%% file: tables/T35_herfindahl_values.tex
\begin{tabular}{lc}
\toprule
Quantity & Value \\
\midrule
$H$ across rooms, smallest (large decagons, about a thousand discarded modes) & 0.002 \\
$H$ across rooms, largest among triangles (about thirty discarded modes) & 0.03 \\
$H$, median over in-scope rooms & 0.005 \\
Representative octagon (room 00850): $K_{\mathrm{total}}$ / discarded modes & 483 / 433 \\
Representative octagon: $H$ / effective contributor count $1/H$ & 0.0038 / 260 \\
Limit of $H$ as $K_{\mathrm{total}} \to \infty$ with $|s| = 1.13$ and $K = 50$ & $\approx 0.0003$ \\
\bottomrule
\end{tabular}

%% file: tables/T4_aniso_sweep.tex
\setlength{\tabcolsep}{4pt}%
\begin{tabular}{@{}r cc cc@{}}
\toprule
 & \multicolumn{2}{c}{Up to the mean marker} & \multicolumn{2}{c}{Up to the spike level} \\
\cmidrule(lr){2-3}\cmidrule(lr){4-5}
$T$ & Median & Sampled orientation & Median & Sampled orientation \\
\midrule
1    & $+0.01$\,pp & $0.05$\,pp & $+0.04$\,pp & $0.26$\,pp \\
100  & $+0.01$\,pp & $0.04$\,pp & $+0.02$\,pp & $0.08$\,pp \\
1000 & $+0.08$\,pp & $0.50$\,pp & $+0.21$\,pp & $1.42$\,pp \\
\midrule
\multicolumn{5}{l}{Optimal exponent over all cells with the full room set: $p^\star \in [1.10, 1.20]$} \\
\bottomrule
\end{tabular}

%% file: supplementary/B_berry.tex

\section{Berry's Conjecture: Extended Validation}
\label{app:berry}

The isotropy argument uses Berry's conjecture only for the diagonal anisotropy term $D$ (\cref{sec:theory}).
The pooled test of \cref{sec:acoustic} gives a small Kolmogorov--Smirnov statistic (Table~\ref{tab:berry-ks}), and the per-room, near-wall, and worst-room analyses below show that the deviations are real but do not translate into meaningful reconstruction cost.
Two KS conventions appear, with different levels for the same room: squared values $|\Omega|\,\varphi_k(x_m)^2$ against $\chi^2(1)$ (\cref{app:berry-histograms}, \cref{app:berry-per-room}, \cref{app:worst-rooms}, and Appendix~\ref{app:gap-diagnostics}) and signed values $\sqrt{|\Omega|}\,\varphi_k(x_m)$ against $\mathcal{N}(0,1)$ (\cref{app:boundary-ks}, and \cref{app:berry-per-room} for comparison).

\begin{table}[t]
\centering
\small
\caption{\textbf{Berry's prediction fits the pooled eigenfunction values closely, while per-room deviations are real, grow with mode count, and carry almost no reconstruction cost.}
The pooled and per-room Kolmogorov--Smirnov statistics of \cref{app:berry} are collected here with their sample sizes and the cost of the worst-fitting rooms.
Only the diagonal term $D$ of \cref{eq:frob-exact} depends on the conjecture, which is why these deviations leave the closed form unaffected.}
\label{tab:berry-ks}
\input{tables/T36_berry_ks}\end{table}

\subsection{Eigenfunction value distribution}
\label{app:berry-histograms}

Under Berry's conjecture, $\varphi_k(x_m)$ at a uniformly random sensor position is a zero-mean Gaussian with variance $1/|\Omega|$, so $|\Omega|\,\varphi_k(x_m)^2$ follows $\chi^2(1)$.
Pooling every retained mode ($k \leq \min(K_{\mathrm{total}}, 50)$), all $M = 8$ sensors, and all $197$ validation rooms gives $N = 77{,}968$ samples.

\paragraph{Pooled distribution.}
Against $\chi^2(1)$ the pooled two-sided statistic is a few percent (Figure~\ref{fig:berry-qq}a; Table~\ref{tab:berry-ks}), which is how far the empirical CDF deviates from the Berry prediction at most.
At this sample size even so small a deviation is significant (Table~\ref{tab:berry-ks}); what matters is its effect on the regularizer.

\paragraph{Tail behavior.}
The empirical distribution has a lighter upper tail than $\chi^2(1)$ (Figure~\ref{fig:berry-qq}a), consistent with finite-mode effects: the conjecture is asymptotic in frequency, and the lowest retained modes are not yet in that regime.
Restricting the pool to $k \geq 10$ lowers the statistic further (Table~\ref{tab:berry-ks}).

\begin{figure}[t]
  \centering
  \includegraphics[width=\linewidth]{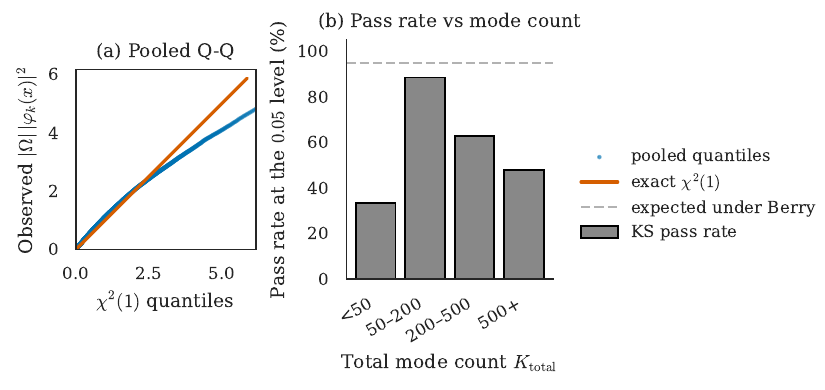}
  \caption{\textbf{Pooled eigenfunction values stay close to Berry's prediction, yet per-room tests detect real deviations.}
  Squared values follow $\chi^2(1)$ in the bulk but fall short in the upper tail (a).
  Per-room pass rates peak at moderate mode counts and stay below what exact agreement gives (b).
  Sample sizes are equal above $K$ modes, so the decline at high counts is real.
  Berry's conjecture enters the bound only through $D$ (\cref{app:anisotropy}), and the deviations cost nothing in reconstruction (\cref{app:worst-rooms}).}
  \label{fig:berry-qq}
\end{figure}

\paragraph{Why the pass rate is non-monotonic.}
The pass rate (Figure~\ref{fig:berry-qq}b) peaks at moderate mode counts and declines on both sides (Table~\ref{tab:berry-ks}), and sample size explains neither decline.
Every room with $K_{\mathrm{total}} \geq K$ contributes the same number of samples, and the median $D_{\mathrm{KS}}$ rises with mode count across the three bins, crossing the critical value only in the highest (Table~\ref{tab:berry-ks}).
The rooms with $K_{\mathrm{total}} < K$ contribute fewer samples yet deviate more (Table~\ref{tab:berry-ks}): all their modes are low-order, where random-wave behavior is not expected.
The regularizer is unaffected: a large $K_{\mathrm{total}}$ supplies enough discarded modes to flatten $P(p)$ even when individual modes conform imperfectly (Appendix~\ref{app:matched-ktotal}).

\subsection{Per-room KS distribution}
\label{app:berry-per-room}

A pooled statistic can hide rooms that violate the conjecture badly, so $D_{\mathrm{KS}}$ is also computed for each of the $197$ validation rooms separately, in both conventions.

\paragraph{Results.}
In the $\chi^2(1)$ convention the per-room values are concentrated just below the critical value, with a single outlying room (Table~\ref{tab:berry-ks}), and the distribution is unimodal with a slight right tail.
In the signed convention the same rooms give somewhat lower levels, and the room rankings agree only moderately between the conventions (Table~\ref{tab:berry-ks}), so the levels are not interchangeable.

\paragraph{Geometry and mode count.}
The rooms with the largest $D_{\mathrm{KS}}$ span every polygon type (Table~\ref{tab:berry-ks}), so geometry alone does not predict $D_{\mathrm{KS}}$.
Mode count correlates with it, weakly but significantly across the in-scope rooms and more weakly across all rooms (Table~\ref{tab:berry-ks}), consistent with the boundary rooms ($K_{\mathrm{total}} \leq K$) forming a distinct regime.
Every in-scope room contributes the same number of samples, so the correlation reflects geometry rather than statistical power; we do not read it causally.
What matters is whether $D_{\mathrm{KS}}$ predicts reconstruction cost: the five rooms with the largest $D_{\mathrm{KS}}$ keep a small relative cost at $T \leq 100$, and at $T = 1000$ four of them stay close to the in-scope median while room 00835 lies several times above it (Table~\ref{tab:berry-ks}, \cref{app:worst-rooms}).

\subsection{Boundary-stratified KS statistics}
\label{app:boundary-ks}

Berry's conjecture is known to degrade near boundaries.
Ranking each room's sensors by normalized wall distance $d_m = \mathrm{dist}(x_m, \partial\Omega)/\mathrm{diam}(\Omega)$ and pooling the signed values of the bottom (near-boundary) and top (interior) quartiles, two sensors each, over the $187$ in-scope rooms shows a slight degradation near walls, where the pooled $D_{\mathrm{KS}}$ is higher than in the interior (Table~\ref{tab:boundary-ks}).
Per room, with the same samples per group, both quartiles sit at the sampling-noise level and at nearly the same median (Table~\ref{tab:berry-ks}).
The regularizer acts in modal space, not sensor space, and the aggregate $M \times M$ covariance mixes near-boundary and interior sensors (\cref{app:anisotropy}).

\begin{table}[t]
\centering
\small
\caption{\textbf{Berry agreement degrades slightly near walls.}
Lower $D_{\mathrm{KS}}$ means closer agreement with Berry's Gaussian prediction, here for signed eigenfunction values pooled over the in-scope rooms by quartile of sensor wall distance.
The regularizer acts in modal space through the aggregate sensor covariance, which mixes near-wall and interior sensors (\cref{app:anisotropy}).}
\label{tab:boundary-ks}
\input{tables/T5_boundary_ks}\end{table}

\subsection{Rooms with the worst Berry agreement}
\label{app:worst-rooms}

The five rooms with the highest per-room $D_{\mathrm{KS}}$ ($\chi^2(1)$ convention) test whether the largest Berry violations translate into reconstruction cost.

\paragraph{Selection.}
Validation rooms with fewer than $50$ discarded modes ($K_{\mathrm{total}} - K < 50$) have too few for a meaningful per-room test and are excluded; Table~\ref{tab:worst-rooms} lists the top five of the rest (counts in Table~\ref{tab:berry-ks}).

\begin{table}[t]
\centering
\small
\caption{\textbf{The rooms with the worst Berry agreement are mid-to-large rooms with near-median $H$.}
Rooms are ranked by $D_{\mathrm{KS}}$ in the $\chi^2(1)$ convention among those with enough discarded modes for a per-room test, and $\delta$ is the relative cost of the closed form against the per-room power-law oracle.
Their costs stay modest except for one room at the longest window, where oracle dispersion rather than Berry failure drives the cost (\cref{app:non-monotonic}).}
\label{tab:worst-rooms}
\input{tables/T6_worst_rooms}\end{table}

\paragraph{Observations.}
The worst-Berry rooms are mid-to-large, many-sided rooms with well above the median mode count (Table~\ref{tab:berry-ks}), consistent with the positive correlation between $D_{\mathrm{KS}}$ and mode count (\cref{app:berry-per-room}).
Their Herfindahl indices are near the population median, and their diagonal-anisotropy ratios are ordinary (Tables~\ref{tab:berry-ks} and~\ref{tab:isotropy-facts}).
Even room 00835, the worst of them, keeps a small cost at $T \leq 100$ (Table~\ref{tab:berry-ks}); its high cost at $T = 1000$ reflects oracle dispersion at long observation windows, which affects all rooms (\cref{app:non-monotonic}), not Berry failure.

\paragraph{Berry agreement and reconstruction cost.}
Across the eligible rooms, the correlation between $D_{\mathrm{KS}}$ and the cost at $T{=}1000$ is significant but weak (Table~\ref{tab:berry-ks}).
Rooms with worse Berry agreement have slightly higher cost, yet even the worst-Berry rooms stay within the range reported in the main text.
Berry violations are real but do not translate into meaningful reconstruction cost: the discarded-mode count, not Berry conformity, sets the flattening of the landscape (Appendix~\ref{app:matched-ktotal}), and anisotropy at the measured levels barely moves the cost of the fixed exponent (Appendix~\ref{app:aniso-sweep}).

%% file: tables/T36_berry_ks.tex
\begin{tabular}{@{}>{\raggedright\arraybackslash}p{0.61\linewidth}>{\raggedright\arraybackslash}p{0.35\linewidth}@{}}
\toprule
Quantity & Value \\
\midrule
Pooled $D_{\mathrm{KS}}$ against $\chi^2(1)$, all retained modes ($N = 77{,}968$) & 0.037 ($p < 10^{-10}$) \\
Pooled $D_{\mathrm{KS}}$, modes $k \geq 10$ only & 0.016 \\
Samples per room for $K_{\mathrm{total}} \geq K$ / critical value at the 0.05 level & 400 / 0.067 \\
Pass rate at the 0.05 level, $K_{\mathrm{total}} \in [50, 200)$ (the peak) & 88.5\% \\
Median per-room $D_{\mathrm{KS}}$ by bin, $[50, 200)$ / $[200, 500)$ / $\geq 500$ (rooms) & 0.054 (61) / 0.061 (54) / 0.068 (73) \\
Rooms with $K_{\mathrm{total}} < K$ (9 rooms): samples / median $D_{\mathrm{KS}}$ & 136--376 / 0.089 \\
Per-room $D_{\mathrm{KS}}$, $\chi^2(1)$ convention: median [IQR] / maximum / rooms above 0.12 & 0.061 [0.046, 0.073] / 0.197 / 1 \\
Per-room $D_{\mathrm{KS}}$, signed convention: median [IQR] / maximum & 0.054 [0.042, 0.066] / 0.205 \\
Spearman correlation of room rankings between the two conventions & 0.57 \\
Rooms with $D_{\mathrm{KS}} > 0.08$ (tri / quad / pent / hex / hept / oct / non / dec) & 38 (4 / 6 / 2 / 5 / 3 / 9 / 5 / 4) \\
Spearman $\rho(D_{\mathrm{KS}}, K_{\mathrm{total}} - K)$, in-scope (Pearson $r$) / all rooms & $+0.29$ ($+0.30$, $p < 10^{-4}$) / $+0.16$ \\
Worst-five rooms: largest $\delta$ at $T \leq 100$ / $\delta(T{=}1000)$ relative to the in-scope median & 3.4\% / four within $1.5\times$, room 00835 at $3.9\times$ \\
Near-wall vs interior sensors: pooled $D_{\mathrm{KS}}$ excess / per-room medians (100 samples per group) & 22\% / 0.084 vs 0.085 \\
Rooms excluded from the worst-room test (fewer than 50 discarded modes) / eligible rooms & 29 / 168 \\
Worst-Berry rooms: vertices / $K_{\mathrm{total}}$ / $H$ & 6--10 / $> 400$ / 0.0025--0.0042 \\
Room 00835: $D_{\mathrm{KS}}$ / $\delta(T{=}1000)$ / $\delta$ at $T \leq 100$ & 0.098 / 22.5\% / $\leq 3\%$ \\
Spearman $\rho(D_{\mathrm{KS}}, \delta(T{=}1000))$ over the eligible rooms & 0.19 ($p = 0.012$) \\
\bottomrule
\end{tabular}

%% file: tables/T5_boundary_ks.tex
\begin{tabular}{lcc}
\toprule
Sensor group & Samples & Pooled $D_{\mathrm{KS}}$ \\
\midrule
Interior (top quartile) & 18{,}700 & 0.0198 \\
Near-boundary (bottom quartile) & 18{,}700 & 0.0242 \\
All sensors & 74{,}800 & 0.0203 \\
\bottomrule
\end{tabular}

%% file: tables/T6_worst_rooms.tex
\begin{tabular}{@{}l c c c c c c c c@{}}
\toprule
Room & Verts & Area & $K_{\mathrm{total}}$ & $H$ & $D_{\mathrm{KS}}$ & $\delta(T{=}1)$ & $\delta(T{=}100)$ & $\delta(T{=}1000)$ \\
\midrule
00963 & 6  & 15.6 & 422  & 0.0042 & 0.103 & 0.0\% & 1.8\% & 5.4\%  \\
00924 & 9  & 18.8 & 508  & 0.0037 & 0.101 & 0.5\% & 0.1\% & 5.8\%  \\
00860 & 8  & 22.4 & 604  & 0.0033 & 0.100 & 0.0\% & 1.5\% & 3.8\%  \\
00835 & 10 & 36.8 & 987  & 0.0025 & 0.098 & 0.1\% & 2.8\% & 22.5\% \\
00900 & 10 & 34.5 & 921  & 0.0026 & 0.097 & 3.2\% & 2.2\% & 8.5\%  \\
\bottomrule
\end{tabular}

%% file: supplementary/C_acoustic_experiments.tex

\section{Acoustic Experiments: Extended Results}
\label{app:acoustic}

\subsection{Estimating $|s|$ in practice}
\label{app:estimating-s}

The population exponent $|s|$ is the one measured quantity the paper depends on.

\paragraph{Procedure.}
For each room, the empirical variance $\hat{\sigma}^2_{a,k}$ of the $k$-th modal amplitude is computed from the amplitude time series across observation windows, and OLS fits
\begin{equation}
\log \hat{\sigma}^2_{a,k} = \beta_0 - |s| \log \lambda_k + \epsilon_k, \qquad k = 1, \ldots, K,
\end{equation}
where $\beta_0$ absorbs the constant $c$; $|s|$ is the absolute value of the slope.
The discarded modes are never observed, so the regression uses only the $K = 50$ retained modes and assumes that the power law $\sigma^2_{a,k} \propto \lambda_k^{-|s|}$ continues into the truncation band.

\paragraph{When the extrapolation holds.}
Two conditions suffice.
\begin{enumerate}
\item[(a)] \textbf{Dense eigenvalue spectrum.}
Weyl's law gives approximately uniform eigenvalue spacing in 2D, $\lambda_k \sim k/|\Omega|$, so no gap separates the retained and discarded bands, and the fit over modes $1$--$50$ extends to modes $51$--$313$.
\item[(b)] \textbf{Stationary excitation statistics.}
The excitation treats all modes alike, so the same power law holds above and below the truncation index; the broadband click of the simulations below has this property.
\end{enumerate}

\paragraph{Excitation and prior.}
Each room is excited by a band-limited click from a point source, and the FDTD pressure field is projected onto the FEM eigenmodes to give the modal amplitudes; the decay $\lambda_k^{-|s|}$ is measured from these amplitudes, not imposed.
The analysis models them with independent Gaussian amplitudes of variance $\lambda_k^{-|s|}$, the standard diffuse-field model and a reasonable approximation for reverberant rooms excited by broadband sources, where many incoherent sources or a broadband impulse leave modal amplitudes approximately independent.
It is also the diagonal prior of Proposition~\ref{prop:isotropy}; off-diagonal prior structure (e.g., adjacent-mode correlations from a localized source) would make the optimal $\Gamma$ non-diagonal, a case Appendix~\ref{app:prior-robustness} tests.

\paragraph{Population estimate.}
The ten rooms with $K_{\mathrm{total}} \leq 50$ have no discarded modes and hence no estimable noise profile; they are excluded from the fit but kept in all downstream evaluation.
Table~\ref{tab:s-population} summarizes the per-room fits of the remaining $187$ rooms: the per-room standard deviation is the dispersion of fitted slopes, and the bootstrap quantities resample rooms to describe the population median, the value carried into all downstream experiments.

\begin{table}[t]
\centering
\small
\caption{\textbf{The population median of $|\hat{s}|$ is precise, although single-room slopes are not.}
Statistics summarize the per-room fits of $|s|$ over the in-scope rooms, and bootstrap quantities resample rooms.
The median is the exponent carried into all downstream experiments, and per-room slopes do not close the oracle gap (\cref{app:sroom-vs-pop}).}
\label{tab:s-population}
\input{tables/T7_s_population}\end{table}

\paragraph{Sensitivity to estimation error.}
Re-estimating $|s|$ from $50$ rooms shifts the exponent by up to about a tenth, and using such an estimate raises the median reconstruction error by less than half a point at every $T$ (Table~\ref{tab:acoustic-facts}), because $P(p)$ is flat near its minimum (Figure~\ref{fig:landscape}).
With more rooms the estimate of $|s|$ rises toward its full-sample value, which it approaches for the larger subsets, and the median cost $\Delta P$ falls over the same range (Figure~\ref{fig:size-ablation-app}).

\begin{figure}[t]
  \centering
  \includegraphics[width=\linewidth]{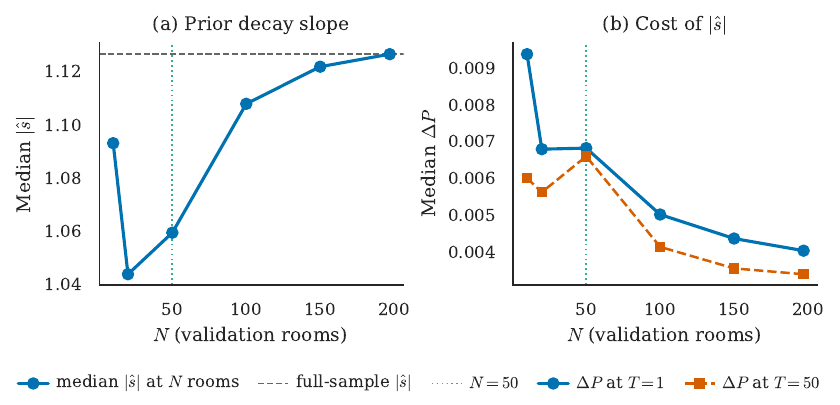}
  \caption{\textbf{The exponent of the closed form can be estimated from a few dozen rooms at little cost.}
  On subsets of the validation rooms the median fitted exponent sits below its full-sample value and climbs toward it as rooms are added (a).
  The median absolute cost of the resulting closed form stays small at every sample size, where $\Delta P$ is its extra error over the per-room power-law oracle (b).
  The flat landscape of \cref{fig:landscape} keeps this estimation error cheap.}
  \label{fig:size-ablation-app}
\end{figure}

\paragraph{Use in a new domain or a single room.}
Given a stable modal decomposition, a new domain or excitation regime uses the same procedure (estimate the retained-mode energies, regress log energy on log eigenvalue, take the absolute slope); a new excitation changes the value, not the procedure.
The precondition is the real constraint: known or computed eigenmodes and enough spatial diversity (Appendix~\ref{app:aperture}).
With a single room we would report the bootstrap interval (Table~\ref{tab:s-population}) and use the width of the $P(p)$ basin as a sensitivity check.

\subsection{Full cost table}
\label{app:full-cost}

At large $T$ the worst-room relative cost in Table~\ref{tab:full-cost} inflates because the oracle floor $P(p^\star, T)$ shrinks toward zero, while absolute gaps stay bounded at all $T$ (Figure~\ref{fig:cost-distribution-app}; Table~\ref{tab:acoustic-facts}); \cref{app:non-monotonic} explains the non-monotone median.

\begin{table}[t]
\centering
\small
\caption{\textbf{The median cost of the closed form stays small at every observation window, while worst-room relative costs inflate as the oracle floor shrinks.}
The relative cost $\delta(T)$ is the extra modal error of the closed form as a fraction of the per-room power-law oracle error, over the in-scope and boundary rooms together.
\Cref{tab:cost_tiers} summarizes the in-scope rooms alone, and \cref{app:non-monotonic} explains why the median rises and then falls.}
\label{tab:full-cost}
\input{tables/T8_full_cost}\\[2pt]
\footnotesize
\textsuperscript{$\dagger$}Over the $187$ in-scope rooms alone, without the $10$ boundary rooms included here, this row reads $5.83\%$ (Table~\ref{tab:cost_tiers}).
\end{table}

\begin{table}[t]
\centering
\small
\caption{\textbf{The per-room oracle exponent drifts and spreads with the observation window, while the closed form's cost stays small for most rooms.}
The exponent is concentrated at short and intermediate windows and barely defined at the longest, which is why the closed form costs little there despite the drift (\cref{sec:acoustic}).
The last two rows summarize the absolute costs of \cref{tab:full-cost} against the threshold used in \cref{sec:acoustic}.}
\label{tab:cost-facts}
\input{tables/T33_cost_facts}\end{table}

\begin{figure}[t]
  \centering
  \includegraphics[width=\linewidth]{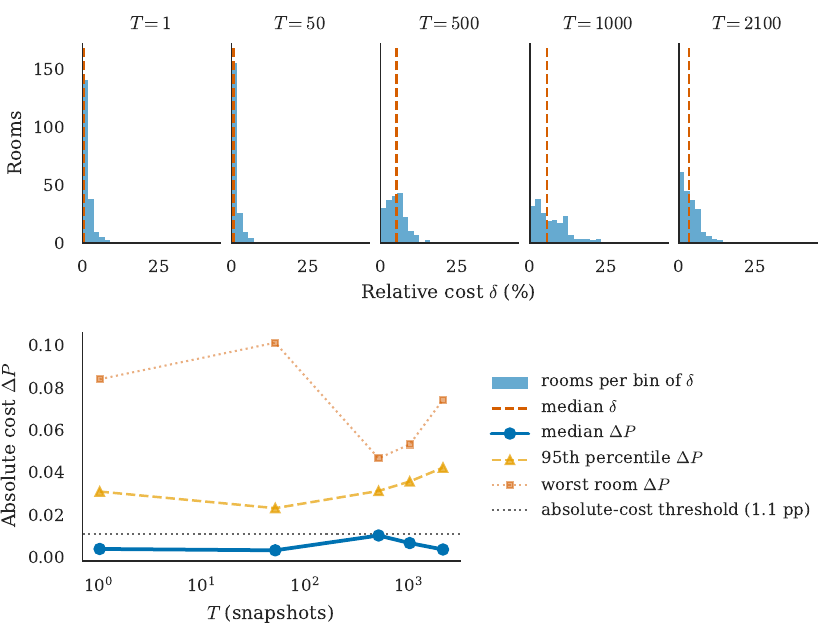}
  \caption{\textbf{The closed form costs little in most rooms at every window.}
  $\Delta P$ is the closed form's extra error over the per-room power-law oracle, and $\delta$ divides it by the oracle error.
  Relative costs pile up near zero at short windows and spread out at long ones (top).
  The median absolute cost stays at or below the threshold of \cref{sec:acoustic}, while the worst room loses several times that (bottom).
  It extends \cref{tab:cost_tiers} to single rooms, boundary rooms included.}
  \label{fig:cost-distribution-app}
\end{figure}


\begin{table}[t]
\centering
\small
\caption{\textbf{The closed form's worst cases are boundary rooms and short windows, and estimating its exponent from a few dozen rooms is cheap.}
The table collects the sensitivity of the exponent estimate, the worst rooms, the two parts of the relative cost, and the grid resolution.
Every worst case leaves most of the error in the oracle floor, the error that remains with the best per-room power law (\cref{tab:adaptation-gap}).}
\label{tab:acoustic-facts}
\input{tables/T37_acoustic_facts}\end{table}

\subsection{Worst-case rooms}
\label{app:per-room}

\paragraph{Worst absolute-cost room.}
Room 00806, a small triangle and a boundary room, has the largest absolute cost across all $T$, reached at $T = 100$, and even there most of its error is the oracle floor (Table~\ref{tab:acoustic-facts}).
This is the expected failure mode at $K_{\mathrm{total}} \leq K$: without discarded modes there is no truncation noise, and the per-room oracle picks an exponent far from the population $|s|$.

\paragraph{Worst in-scope rooms.}
Among in-scope rooms ($K_{\mathrm{total}} > K$) the largest absolute cost occurs at a short window in a room with few discarded modes, and the worst in-scope gap at $T = 1000$ is smaller still, with almost all of that room's error in the oracle floor (Table~\ref{tab:acoustic-facts}).

\paragraph{Population view.}
In both worst cases most of the error is the oracle floor, the error that remains with the best per-room power-law regularizer and the limit of any static regularizer with $M = 8$ sensors and $K = 50$ modes.
At $T = 1000$, two rooms in three have absolute cost below $1.1$\,pp, with or without the boundary rooms (Table~\ref{tab:acoustic-facts}; \cref{sec:acoustic}), and the flat basin of the population median (Figure~\ref{fig:landscape}) holds room by room (Figure~\ref{fig:per-room-app}).

\subsection{Non-monotonic relative cost}
\label{app:non-monotonic}

Table~\ref{tab:full-cost} shows $\delta(T)$ rising by an order of magnitude from $T = 1$ to $T = 1000$ and falling again at $T = 2100$.
The median adaptation gap $\mathrm{gap}(T) = P(|s|, T) - P(p^\star, T)$ and the oracle floor $P(p^\star, T)$ reach their extrema at different $T$: the gap peaks earlier than the floor bottoms out (\cref{app:gap-floor}; Table~\ref{tab:acoustic-facts}).
Where the gap is largest the floor is still high (Table~\ref{tab:adaptation-gap}), which dilutes the ratio; by the time the floor bottoms out the gap has shrunk, which raises the ratio; beyond that window both push $\delta$ down.
Signal decay does not drive the peak: more than half of the median signal energy remains at $T = 2100$ (Table~\ref{tab:acoustic-facts}), and decay is monotone in $T$, so it cannot produce a non-monotone $\delta(T)$.

\subsection{Adaptation gap and oracle floor}
\label{app:gap-floor}

As $T$ grows, posterior concentration compresses the landscape: the posterior variance shrinks, $\Gamma$ loses influence, and the median $P$ range across $p \in [0, 3]$ falls by more than an order of magnitude from $T = 1$ to $T = 2100$ (Figure~\ref{fig:landscape}; Table~\ref{tab:acoustic-facts}).
Table~\ref{tab:adaptation-gap} tracks the two parts of the relative cost through this compression.

\paragraph{The adaptation gap is unimodal.}
The gap is flat for $T \leq 50$ (prior-dominated regime), roughly triples by $T = 500$ as the per-room oracle specializes, and falls back to its short-window level at $T = 2100$ as the landscape flattens (Table~\ref{tab:adaptation-gap}).

\paragraph{The oracle floor is U-shaped.}
The floor falls by a factor of six from $T = 1$ to $T = 1000$ and rebounds at $T = 2100$ (Table~\ref{tab:adaptation-gap}).
At large $T$ all exponents give nearly the same $P$, so a per-room ``optimal'' exponent is barely defined.

\begin{table}[t]
\centering
\small
\caption{\textbf{The adaptation gap and the oracle floor peak at different windows.}
The gap is the closed form's extra error over the per-room power-law oracle, and the floor is the oracle's own error, boundary rooms included.
Wide IQRs at short windows come from noise in the per-room oracle exponent, not from true dispersion of the gap.
Relative cost is the gap divided by the floor, so it peaks where the floor is lowest, not where the gap is largest (\cref{app:non-monotonic}).}
\label{tab:adaptation-gap}
\input{tables/T9_adaptation_gap}\end{table}

\subsection{Grid resolution}
\label{app:brent}

The $61$-point grid finds the global minimum reliably: the basin of $P(p)$ is an order of magnitude wider than the grid spacing (Table~\ref{tab:acoustic-facts}), so no continuous optimum falls between grid points.

\subsection{Per-room landscape gallery}
\label{app:per-room-landscapes}

Figure~\ref{fig:per-room-app} shows that the basin holds room by room: across six rooms with $3$ to $10$ vertices, per-room $p^\star$ varies by more than a factor of three (Table~\ref{tab:acoustic-facts}), yet $|\hat{s}| = 1.13$ lies inside the basin $P(p) \leq 1.05\,P(p^\star)$ in every room.

\begin{figure}[t]
  \centering
  \includegraphics[width=\linewidth]{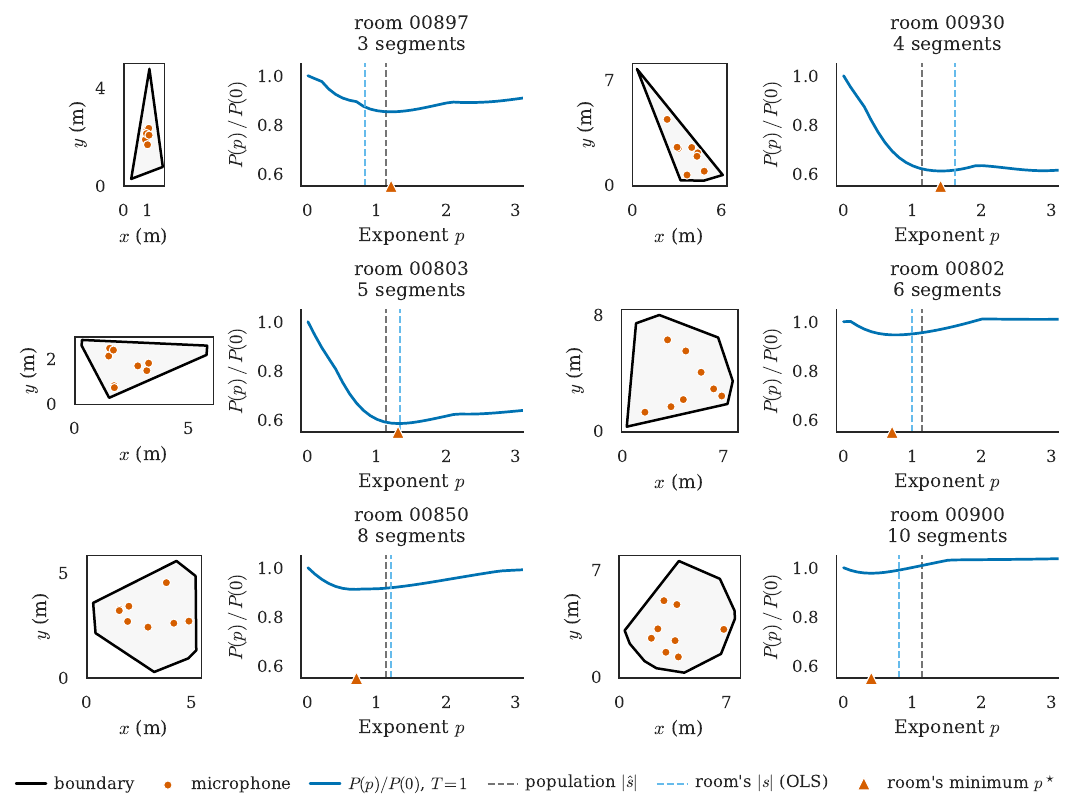}
  \caption{\textbf{A wide, flat basin appears in each of six room geometries.}
  Each row is one room, drawn with the eight microphones its curve uses (left).
  Its landscape at $T{=}1$ is divided by the ridge error, so values below $1$ mean a gain over ridge (right).
  Per-room slopes and minima vary from room to room, yet the closed-form exponent stays inside every basin.
  The flat median landscape of \cref{fig:landscape} holds room by room, not only on average.}
  \label{fig:per-room-app}
\end{figure}

\subsection{Dataset generation}
\label{app:room-geometry}

Rooms are random convex 2D polygons with $3$--$10$ vertices, with eigenpairs computed by FEM on a triangulated mesh; pressure fields are simulated on the same room by a finite-difference time-domain (FDTD) solver and projected onto the FEM eigenfunctions to give the modal amplitude time series.
Of the $1{,}000$ generated rooms, $800$ are used for training (\cref{sec:learned}) and $197$ for validation; three rooms (00905, 00913, 00921) have degenerate geometry and are excluded.

%% file: tables/T7_s_population.tex
\begin{tabular}{lc}
\toprule
Statistic & Value \\
\midrule
Median $|\hat{s}|$ & 1.13 \\
Mean $|\hat{s}|$ & 1.12 \\
Std across rooms of the point estimates $|\hat{s}|$ & 0.25 \\
Typical per-room bootstrap SE & 0.27 \\
Inter-room std after deconvolving the per-room SE & $\approx 0$ (Appendix~\ref{app:sroom-vs-pop}) \\
Bootstrap SE of the median & 0.03 \\
Bootstrap 95\% CI on the median & $[1.08, 1.18]$ \\
Median of the per-room fits, all 197 rooms (187 in-scope) & 1.1243 (1.1266) \\
\bottomrule
\end{tabular}

%% file: tables/T8_full_cost.tex
\begin{tabular}{@{}r cccc@{}}
\toprule
$T$ & Median $\delta$ & IQR & 95th pct & Worst room \\
\midrule
1    & 0.5\% & $[0.2, 1.9]$ & 4.8\% & 15.0\% \\
5    & 0.5\% & $[0.1, 1.6]$ & 5.6\% & 15.8\% \\
10   & 0.5\% & $[0.1, 1.5]$ & 4.9\% & 16.2\% \\
20   & 0.5\% & $[0.1, 1.4]$ & 3.9\% & 20.2\% \\
50   & 0.6\% & $[0.2, 1.5]$ & 4.5\% & 20.5\% \\
100  & 1.5\% & $[0.6, 2.8]$ & 6.2\% & 27.5\% \\
200  & 2.2\% & $[0.7, 4.1]$ & 7.9\% & 14.8\% \\
500  & 5.1\% & $[2.7, 7.0]$ & 12.3\% & 26.8\% \\
1000 & 5.6\%\textsuperscript{$\dagger$} & $[2.8, 10.3]$ & 16.9\% & 34.8\% \\
2100 & 3.3\% & $[1.4, 5.7]$ & 10.5\% & 44.8\% \\
\bottomrule
\end{tabular}

%% file: tables/T33_cost_facts.tex
\begin{tabular}{lcccccc}
\toprule
$T$ & 1 & 50 & 100 & 500 & 1000 & 2100 \\
\midrule
Per-room $p^\star$, median over in-scope rooms & 1.2 & 1.2 & 1.3 & 1.6 & 2.0 & 0.6 \\
Per-room $p^\star$, std across in-scope rooms & 0.85 & 0.26 & 0.30 & 0.42 & 0.99 & 1.67 \\
\midrule
\multicolumn{7}{l}{Rooms below the 1.1\,pp absolute-cost threshold at $T = 1000$: 69.0\%} \\
\multicolumn{7}{l}{Worst in-scope absolute cost over all $T$ (prior-dominated regime): 7.61\,pp} \\
\bottomrule
\end{tabular}

%% file: tables/T37_acoustic_facts.tex
\begin{tabular}{@{}>{\raggedright\arraybackslash}p{0.61\linewidth}>{\raggedright\arraybackslash}p{0.35\linewidth}@{}}
\toprule
Quantity & Value \\
\midrule
Re-estimating $|s|$ from 50 rooms: deviation for the first 50 validation rooms / worst over 1{,}000 random subsets (95th percentile) & 6\% / 9.8\% (6.6\%) \\
Estimate from the first 50 validation rooms ($|s| = 1.06$): largest median per-room rise in $P$ over $T$ & 0.38\,pp \\
Worst absolute cost of all rooms (room 00806, boundary, $K_{\mathrm{total}} = 19$, $T = 100$): $P(|s|)$ / $P(p^\star)$ / $\Delta P$ / floor fraction & 0.491 / 0.385 / 10.6\,pp / 78.4\% \\
Worst in-scope absolute cost (room 00931, $K_{\mathrm{total}} = 89$, $T = 5$) & 7.61\,pp \\
Worst in-scope gap at $T = 1000$ (room 00890, $K_{\mathrm{total}} = 62$): $\Delta P$ / floor fraction & 4.84\,pp / 95.0\% \\
Rooms below $1.1$\,pp absolute cost at $T = 1000$: all rooms / in-scope & 67.0\% / 69.0\% \\
Window where the adaptation gap peaks / where the oracle floor is lowest & $T = 500$ / $T = 1000$ \\
Median signal energy remaining at $T = 2100$ & 60\% \\
Range of the median $P$ over $p \in [0, 3]$: $T = 1$ / $T = 2100$ (ratio) & 0.18 / 0.013 ($14\times$) \\
Basin width of $P(p)$ (region within 5\% of the minimum) vs grid spacing & $\Delta p \approx 1$--$2$ vs 0.1 \\
Per-room $p^\star$ in the gallery, decagon to quadrilateral (room 00930) & about 0.4 to 1.4 \\
\bottomrule
\end{tabular}

%% file: tables/T9_adaptation_gap.tex
\begin{tabular}{rccc}
\toprule
$T$ & Median gap (pp) & IQR (pp) & Median $P(p^\star)$ \\
\midrule
1    & 0.40 & $[0.12, 1.42]$ & 0.715 \\
5    & 0.35 & $[0.06, 1.07]$ & 0.715 \\
10   & 0.34 & $[0.10, 1.03]$ & 0.703 \\
20   & 0.36 & $[0.08, 0.98]$ & 0.679 \\
50   & 0.34 & $[0.13, 0.88]$ & 0.638 \\
100  & 0.89 & $[0.34, 1.53]$ & 0.594 \\
200  & 0.91 & $[0.33, 1.87]$ & 0.461 \\
500  & \textbf{1.04} & $[0.56, 1.65]$ & 0.220 \\
1000 & 0.68 & $[0.35, 1.34]$ & \textbf{0.122} \\
2100 & 0.38 & $[0.16, 1.21]$ & 0.152 \\
\bottomrule
\end{tabular}

%% file: supplementary/D_learning_experiments.tex

\section{Learning Experiments: Extended Results}
\label{app:learning}

These checks ask whether the saturation of \cref{sec:learned} reflects the models or the loss landscape.

\paragraph{Notation.}
The stacked state $\mathbf{a}_0 \in \mathbb{R}^{2K}$ pairs the cosine and sine amplitudes $(c_k, \beta_k)$ of each mode; both share the prior variance $c\,\lambda_k^{-s}$, so a vector $\Gamma \in \mathbb{R}^K$ defines the tied penalty $\Gamma^{(2K)} = \Gamma \otimes I_2 = \mathrm{diag}(\Gamma_1, \Gamma_1, \ldots, \Gamma_K, \Gamma_K)$.
All learned models parameterize $\Gamma$ in this $K$-dimensional space; LIR acts on $K$-dimensional mode-pair vectors with the cosine/sine expansion implicit.
$A = \tilde{\Phi}$ and $y = \tilde{\mathbf{y}}$ denote the stacked operator and measurements of eq.~(\ref{eq:obs-stacked}).

\subsection{Architecture specifications}
\label{app:architectures}

\paragraph{M1: CondNet with unrolled solver (4{,}949 parameters).}
The CondNet, a 3-layer MLP ($10 \to 64 \to 64 \to 1$, ReLU activations, softplus output), maps a 10-dimensional per-mode feature vector to a positive penalty per mode, giving the diagonal $\Gamma \in \mathbb{R}^K$.
For mode $k$, with $G = \tilde\Phi^\top\tilde\Phi/\lambda_{\max}(\tilde\Phi^\top\tilde\Phi)$, $d_k$ the mean of the two diagonal entries of $G$ for mode $k$, and $b_{kj}$ the Frobenius norm of the $2\times2$ block of $G$ coupling modes $k$ and $j \neq k$, the features are $\log\lambda_k$, $\log(c_s\sqrt{\lambda_k})$, $\log d_k$, $\log\max_j b_{kj}$, $\log\operatorname{mean}_j b_{kj}$, $2K d_k/\mathrm{tr}\,G$, and $\max_j b_{kj}/d_k$, plus three entries shared by all modes: $\log M$, $\log T$, and the room damping rate $\gamma$.
An $L{=}10$ unrolled gradient-descent solver minimizes $\frac{1}{2}\|Aa - y\|^2 + \frac{1}{2}\alpha\,a^\top \mathrm{diag}(\Gamma)\,a$ with learned step sizes $\eta_l$ and strengths $\alpha_l$ ($10$ of each); its conditioning input is fixed, so $\Gamma$ is constant across iterations.
The head acts on each mode independently, so M1 stays within the diagonal Tikhonov class.

\paragraph{M2: shared $\Gamma$ with unrolled solver (70 parameters).}
One unconditional diagonal $\Gamma_k = \mathrm{softplus}(\theta_k)$, $\theta \in \mathbb{R}^{50}$, is shared by all rooms and solved with the same $L{=}10$ unrolled scheme and learned $(\eta_l, \alpha_l)$ as M1.
M2 is the learned analogue of the closed form, fitted by SGD without physics; a non-power-law $\Gamma$ from M2 that beat $|s|$ would reveal room-independent structure beyond the theory.

\paragraph{M3: CondNet with closed-form solver (4{,}930 parameters).}
M3 passes the CondNet output to a differentiable closed-form solve,
\begin{equation}
\hat{\mathbf{a}}(\Gamma) = \bigl(\tilde{\Phi}^\top \tilde{\Phi} + \alpha\,\mathrm{diag}(\Gamma^{(2K)})\bigr)^{-1}\tilde{\Phi}^\top \tilde{\mathbf{y}},
\end{equation}
where $\alpha$ is a single learned scalar; M1 and M3 differ only in this solve ($L{=}10$ gradient steps in M1) and are both differentiable end to end.
The training loss is the per-sample $\|\hat{\mathbf{a}} - \mathbf{a}_0\|^2 / \|\mathbf{a}_0\|^2$, with gradients through the inverse by implicit differentiation.
M3 is the strongest diagonal baseline: it adapts $\Gamma$ to each room and optimizes the evaluation metric directly.

\subsection{Training curves and learned spectra}
\label{app:training-curves}

\begin{figure}[t]
  \centering
  \includegraphics[width=\linewidth]{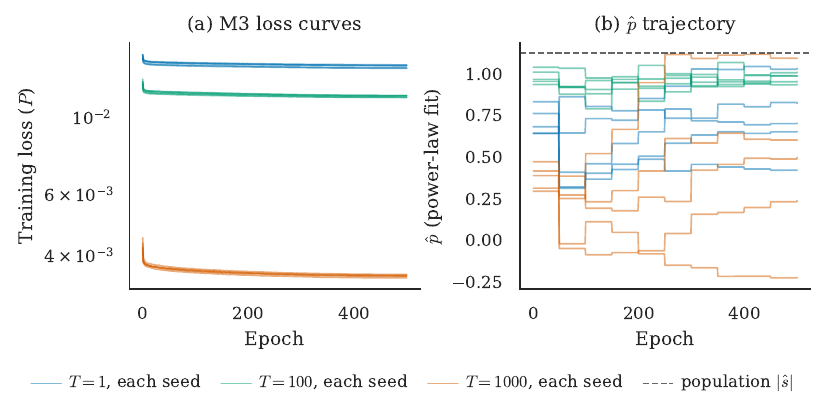}
  \caption{\textbf{M3 training converges while its seeds settle on different spectral exponents.}
  The training loss decreases smoothly and agrees across seeds at each observation window (a).
  The exponent of a power-law fit to the learned spectrum spreads across seeds, most widely at the longest window, where the spectra also fit a power law poorly (b).
  Equal loss at different exponents means the loss is nearly flat across the shapes the seeds reach, the basin that \cref{sec:learned} reads as non-identifiability for M3.}
  \label{fig:m3-training-app}
\end{figure}

\begin{table}[t]
\centering
\small
\caption{\textbf{M3's seeds end at different exponents at every window and agree least at the longest.}
Each entry is the exponent of a power-law fit to one seed's learned spectrum, with the fit's $R^2$ below it.
The spread widens and the fits worsen as the window grows, yet \cref{tab:m3_performance} shows equal error for every seed.
This is the flat basin that \cref{sec:learned} reads as non-identifiability for M3.}
\label{tab:m3-seed-exponents}
\input{tables/T31_m3_seed_exponents}\end{table}

\begin{table}[t]
\centering
\small
\caption{\textbf{M3 and M2 settle on very different exponents at the longest window and reach almost the same error.}
M3's per-room exponents match theory in the median, M2's shared exponent is about half of it, and the error at the two exponents is nearly equal.
The capacity row, the limit of M2's parameterization (\cref{app:capacity}), scopes the flat-basin reading of \cref{sec:learned} to M3.}
\label{tab:learned-exponents}
\input{tables/T34_learned_exponents}\end{table}

Figure~\ref{fig:m3-training-app}a shows that training converges: the loss decreases smoothly for all $T$ and all 5 seeds, with no instability, overfitting, or collapse, and final losses agree closely across seeds.
The learned spectra do not: the five seeds end at different exponents at every window, most widely at the longest (Table~\ref{tab:m3-seed-exponents}), yet all reach the same reconstruction error (Table~\ref{tab:m3_performance}).
This is the most direct evidence of landscape flatness: the loss is nearly constant across the shapes the seeds reach, so nothing pulls them toward one shape.
At the longest window the learned spectra also depart from a power law: the power-law fit worsens for every seed, sharply for four of the five (Table~\ref{tab:m3-seed-exponents}).

\subsection{Room descriptors do not predict the oracle exponent}
\label{app:features}

\begin{table}[t]
\centering
\small
\caption{\textbf{Tier~1 features, which are accessible at inference, do not predict $p^\star$.}
Correlations are with the per-room oracle exponent $p^\star$ at $T{=}1000$, and none survives Bonferroni correction across both feature tiers.
These descriptors therefore give the learned models of \cref{sec:learned} little signal about the per-room exponent.}
\label{tab:tier1-features}
\input{tables/T10_tier1_features}\end{table}

\begin{table}[t]
\centering
\small
\caption{\textbf{Tier~2 features, which require room geometry, collapse to one room-size factor.}
Correlations are with the per-room oracle exponent $p^\star$ at $T{=}1000$.
All features except the segment count are perfectly rank-correlated, mean spacing inversely, as Weyl's law in \cref{eq:weyl} implies.
This factor is the strongest predictor of $p^\star$ among all tested features, yet it explains little of the variance and is unavailable at inference.
Geometry would thus add only a weak room-size signal to the learned models of \cref{sec:learned}.}
\label{tab:tier2-features}
\input{tables/T11_tier2_features}\end{table}

\begin{figure}[t]
  \centering
  \includegraphics[width=\linewidth]{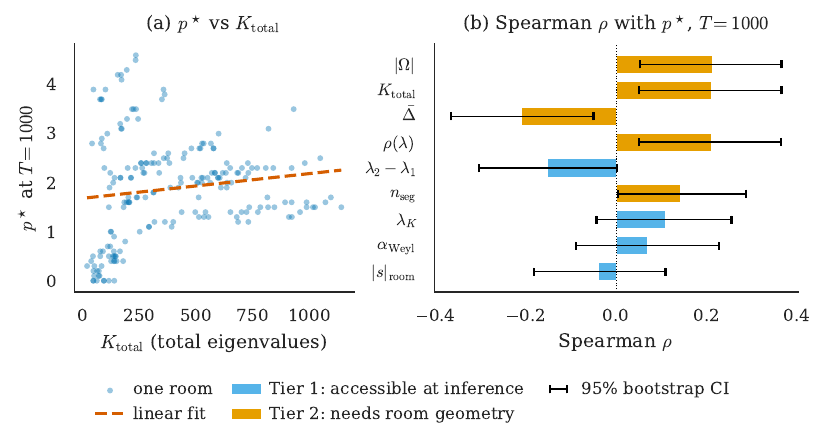}
  \caption{\textbf{No feature accessible at inference predicts $p^\star$.}
  The per-room oracle exponent rises only weakly with the total mode count (a).
  Correlations, with their confidence intervals, are weak for accessible features and largest for geometric proxies of room size (b).
  No tested input would therefore let the learned models of \cref{sec:learned} track a room's oracle exponent.}
  \label{fig:geometric-predictor-app}
\end{figure}

M3 might fail to beat $|s|$ because it lacks the right inputs; a model with room geometry (area, perimeter, vertex count) might predict $p^\star$.
A regression of per-room $p^\star(T = 1000)$ on two tiers of features across 196 rooms (one room excluded for degenerate geometry) tests this: Tier~1 follows from the eigenvalue spectrum and is accessible at inference (Table~\ref{tab:tier1-features}); Tier~2 requires the room geometry (Table~\ref{tab:tier2-features}).
No Tier~1 correlation survives Bonferroni correction across the nine features, and the strongest explains only a few percent of the variance (Table~\ref{tab:tier1-features}).

Room area correlates significantly with $p^\star$ (Table~\ref{tab:tier2-features}), but area, $K_{\mathrm{total}}$, mean spacing, and eigenvalue density are perfectly rank-correlated (Weyl's law: $K_{\mathrm{total}} \propto \mathrm{Area}$, spacing $\propto 1/\mathrm{Area}$) and form one factor, room size (Figure~\ref{fig:geometric-predictor-app}).
Even this strongest predictor explains only a few percent of the variance.
A random forest on all nine features explains only a small fraction of the variance of $p^\star(T = 1000)$ (Table~\ref{tab:tier2-features}) and does worse than the mean at $T = 50$; the rest is noise or information that none of the features captures.

\begin{table}[t]
\centering
\small
\caption{\textbf{An exponent adjusted for $K_{\mathrm{total}}$ costs more than $|s|$ at every tested window except the longest.}
The adjusted exponent is $\hat{p}_{\mathrm{adj}} = \beta_0 + \beta_1 K_{\mathrm{total}}$, and $\delta$ is the median relative cost against the per-room power-law oracle over the in-scope rooms, where lower is better.
Its small gain at the longest window is outweighed by the losses at the other windows, so the closed form, whose cost \cref{tab:cost_tiers} reports, remains the safer choice.}
\label{tab:ktotal-adjusted}
\input{tables/T12_ktotal_adjusted}\end{table}

\paragraph{A $K_{\mathrm{total}}$-adjusted exponent.}
The strongest simple rule, $\hat{p}_{\mathrm{adj}} = \beta_0 + \beta_1 K_{\mathrm{total}}$ fitted by OLS, trims the peak cost at $T = 1000$ by about a point (fit and evaluation both restricted to the in-scope rooms) but raises the cost at $T \leq 100$ by several points (Table~\ref{tab:ktotal-adjusted}); $|s|$ is the safer choice across operating conditions.

\subsection{Robustness to dataset parameters}
\label{app:v2}

\begin{table}[t]
  \centering
  \small
  \caption{\textbf{On a dataset with more modes and a steeper excitation exponent, M3 beats the closed form in only one cell, and only slightly.}
  Entries are the median $P$ over seeds and all validation rooms, boundary rooms included, and positive $\Delta P$ means M3 has higher error than the closed form.
  The one gain reflects the distance between $|\hat{s}|$ and the oracle exponent there, not a shape advantage.
  The saturation of \cref{sec:learned} thus persists when $K$, $M$ and the excitation exponent change.}
  \label{tab:v2-app}
  \input{tables/T13_v2_app}\end{table}

The main results use $K = 50$, $M = 8$, and $|s| = 1.13$.
To test whether the pattern depends on this configuration, M3 is retrained on a parameter-varied dataset with $K = 100$, $M \in \{8, 16\}$, and an excitation regime with $|s| = 1.29$ (Table~\ref{tab:v2-app}); at $K = 100$ a $10^{-8}$ diagonal jitter is added to M3's linear solve for conditioning.
M3 trains on $n = 800$ rooms, and the medians run over $5$ seeds and $197$ validation rooms ($167$ in-scope with $K_{\mathrm{total}} > K$ plus $30$ boundary rooms).
The one negative cell, at $M = 16$ and $T = 1$ (Table~\ref{tab:v2-app}), reflects the gap between $|\hat{s}|$ and the per-room $p^\star$ at that operating point, not a shape advantage of the learned model.
At $M = 16$ the overall $P$ is lower than at $M = 8$, the landscape remains flat, and M3 still does not improve on the closed form; $p^\star$ depends weakly on $M$ through the observation matrix, so exact agreement between $M = 8$ and $M = 16$ is not expected.
$|s|$ itself is a per-regime quantity measured from data ($1.13$ and $1.29$ here); what transfers is its role in $\Gamma_k = \lambda_k^{|s|}$.

\subsection{Per-room versus population exponent}
\label{app:sroom-vs-pop}

\begin{table}[t]
\centering
\small
\caption{\textbf{The per-room exponent correlates with $p^\star$ most strongly at intermediate $T$ but does not reliably cost less than the population $|s|$.}
The relative cost $\delta$ against the per-room power-law oracle is a median over all validation rooms, boundary rooms included, and lower is better.
Estimation noise in $s_{\mathrm{room}}$ outweighs its per-room signal, so per-room slopes do not close the oracle gap of \cref{sec:learned}.}
\label{tab:sroom-rho-vs-T}
\input{tables/T14_sroom_rho_vs_T}\end{table}

Does the oracle gap close if each room uses its own fitted exponent $s_{\mathrm{room}}$ instead of the population $|s|{=}1.13$?
For all $197$ rooms (boundary-inclusive, since the slope fit benefits from the largest sample), $s_{\mathrm{room}}$ is fitted from the retained modes ($K \leq 50$) by the population log-log procedure, and $P(s_{\mathrm{room}})$ is compared with $P(|s|)$ and with the per-room power-law oracle $P(p^\star)$.\footnote{Ten rooms have $K_{\mathrm{total}} \leq 50$ and no discarded modes; they are excluded from the population fit (Appendix~\ref{app:estimating-s}), and their $s_{\mathrm{room}}$ is fitted on all $K_{\mathrm{total}}$ of their modes. Excluding them changes the correlations and the median costs only slightly and leaves the sign of $\delta(s_{\mathrm{room}}) - \delta(|s|)$ unchanged at every $T$ (Table~\ref{tab:sroom-rho-vs-T}).}
The median per-room estimate is indistinguishable from the population value (Table~\ref{tab:s-population}).
The point estimates spread across rooms by less than the median per-room bootstrap SE (Table~\ref{tab:s-population}), so the deconvolved inter-room spread is effectively zero;\footnote{The method-of-moments estimate of the inter-room variance component is slightly negative under REML, a known boundary artifact when estimation noise exceeds the observed spread; it is reported as zero, consistent with per-room true values that are statistically indistinguishable from the population mean.} this per-room error also exceeds the standard error of the population median by an order of magnitude.

$s_{\mathrm{room}}$ carries some information about $p^\star$ at intermediate $T$ (Table~\ref{tab:sroom-rho-vs-T}): the correlation is moderate at $T \in \{50, 100\}$, negligible at $T{=}1$, and weak at $T{=}1000$, a drop caused by the drift of $p^\star$ far from $s_{\mathrm{room}}$ (Table~\ref{tab:cost-facts}).
Yet reconstruction with $s_{\mathrm{room}}$ improves on $|s|$ only marginally, at $T{=}100$ and $T{=}500$, and is worse at every other $T$, including $T{=}50$, where the correlation is strongest (Table~\ref{tab:sroom-rho-vs-T}); at $T{=}1000$ fewer than half of the rooms benefit from per-room tuning.

The mechanism is a noisy plug-in: $\hat{s}_{\mathrm{room}}$ adds far more estimation noise than per-room signal, the setting in which shrinkage to the grand mean dominates per-unit estimation \citep{james1961estimation, stein1956inadmissibility}; the analogy is qualitative because $P(p)$ is not quadratic in the exponent.
At $T{=}1000$ a second effect adds to the first: $p^\star$ has drifted away from $s_{\mathrm{room}}$, so even a noise-free $s_{\mathrm{room}}$ would target the wrong value.
Per-room slope variability does not explain the oracle gap.

\subsection{Per-seed sweep}
\label{app:failures}

\begin{figure}[t]
  \centering
  \includegraphics[width=\linewidth]{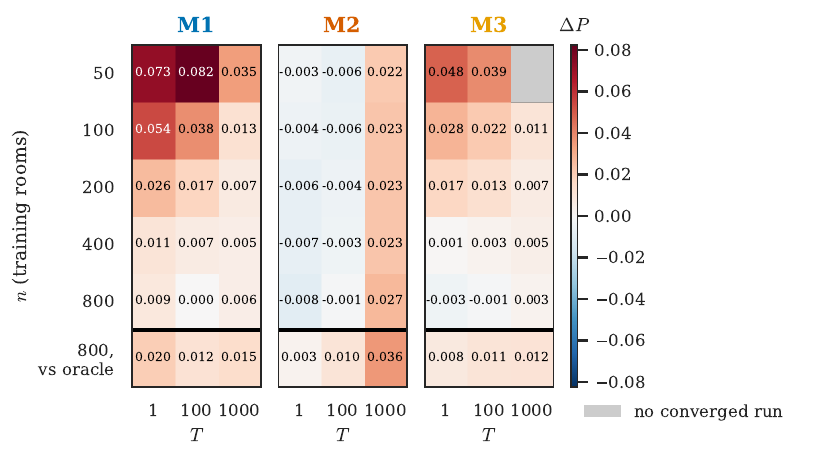}
  \caption{\textbf{No tested learned diagonal regularizer beats the power-law family.}
  Positive $\Delta P$, a median over seeds, means higher error than the baseline.
  The baseline is the closed form with oracle $\alpha$, except in the bottom row, which uses the per-room power-law oracle.
  Every valid trained seed is worse than the oracle, so small negative values against the closed form are no gain over the family.
  Grey cells had no converged run.
  This is the saturation of \cref{sec:learned} in the truncation regime.}
  \label{fig:sweep-heatmap}
\end{figure}

Figure~\ref{fig:sweep-heatmap} compares every trained seed with both baselines.
The main-grid baseline is the closed form at the grid value nearest $|\hat{s}|$ with oracle $\alpha$, and against the per-room power-law oracle every per-seed $\Delta P$ is positive (Table~\ref{tab:learned-exponents}).
A small number of training configurations fail numerically, all in M3 at small training sizes and mostly at $T = 1000$ (Table~\ref{tab:learned-exponents}), because $\tilde{\Phi}^\top\tilde{\Phi}$ is ill-conditioned; all main-text results use $n = 800$, where every seed succeeds.

\subsection{Model capacity verification}
\label{app:capacity}

\begin{figure}[t]
  \centering
  \includegraphics[width=\linewidth]{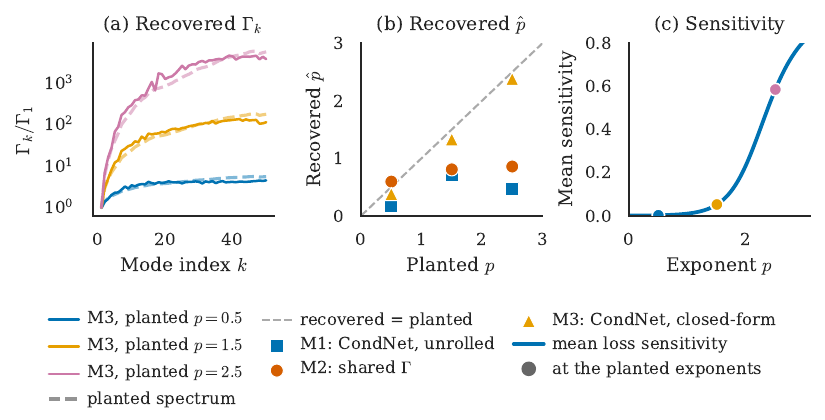}
  \caption{\textbf{M3 can represent the planted power laws, and M2 cannot.}
  With planted power-law priors, M3 learns spectra that follow the planted ones (a).
  M3's recovered exponent rises with the planted one, while M2 saturates and M1 is non-monotonic (b).
  M3 falls slightly short of every target, consistent with a loss that barely responds to the exponent (c).
  Its failure to improve on $|s|$ is a property of the landscape, not the architecture, so the basin reading of \cref{sec:learned} applies to M3 only.}
  \label{fig:capacity-check}
\end{figure}

\begin{table}[t]
\centering
\small
\caption{\textbf{M3 recovers a planted exponent that rises with the target, M2 saturates, and M1 is not even monotone.}
Each entry is the exponent of a power-law fit to the spectrum a model learns on synthetic data whose optimal shape is a known power law.
M3 falls slightly short of every target, consistent with a flat loss, so the basin reading of \cref{sec:learned} is scoped to M3.}
\label{tab:capacity-check}
\input{tables/T38_capacity_check}\end{table}

Does M3 fail to beat $|s|$ because it cannot represent the optimal $\Gamma$?
M3 is trained on synthetic data whose optimal $\Gamma$ is a power law with known exponent $p_{\mathrm{target}} \in \{0.5, 1.5, 2.5\}$ (prior $\sigma^2_{a,k} \propto \lambda_k^{-p_{\mathrm{target}}}$).
Figure~\ref{fig:capacity-check} and Table~\ref{tab:capacity-check} show that M3 recovers the targets monotonically, slightly below each, which is consistent with landscape flatness (a weak gradient near the minimum).
M2 saturates below one because its softplus parameterization limits the reachable spectra, and M1 is non-monotonic, consistent with its weaker performance in the main experiments (Table~\ref{tab:capacity-check}).

M3 tracks power laws across the tested range, slightly below each target, so its failure to improve on $|s|$ is a property of the loss landscape, not of the architecture.
Because M2 saturates and M1 is non-monotonic, the basin (non-identifiability) reading of \cref{sec:learned} is scoped to M3.

\subsection{Learned Iterative Ridge (LIR)}
\label{app:lir}

\paragraph{Architecture.}
LIR parameterizes a linear estimator as $L$ steps of learned gradient descent on a Tikhonov objective, starting from $a^{(0)} = 0$:
\begin{equation}
  a^{(l)} = a^{(l-1)} - \eta_l\bigl(A^\top A\, a^{(l-1)} + \alpha_l\, D_l \odot a^{(l-1)} - A^\top y\bigr), \qquad l = 1, \dots, L,
\end{equation}
where $\eta_l \in \mathbb{R}_+$ is a per-layer step size, $\alpha_l \in \mathbb{R}_+$ a per-layer strength, and $D_l \in \mathbb{R}_+^K$ a per-layer diagonal penalty shape on mode pairs.
$D_l$ is diagonal and positive, so that $\alpha_l D_l \odot a$ remains the gradient of a convex quadratic penalty ($52L$ parameters).
Training uses Adam (learning rate $10^{-3}$), 500 epochs (200 for $L = 1$), MSE loss on 800 training rooms, and 5 seeds $\in \{42, \dots, 46\}$.

\begin{table}[tb]
\centering
\small
\caption{\textbf{LIR does not reach the per-room power-law oracle at any tested depth, and improves on the closed form only at the shortest window.}
Entries are the normalized modal error $P$ (median over validation rooms) as mean and standard deviation over seeds, where lower is better.
A single step is a constant filter and has the highest error in every column.
This is the result outside the diagonal family that \cref{sec:lir} reports.}
\label{tab:lir-ablation}
\input{tables/T15_lir_ablation}\end{table}

\begin{figure}[tb]
  \centering
  \includegraphics[width=\linewidth]{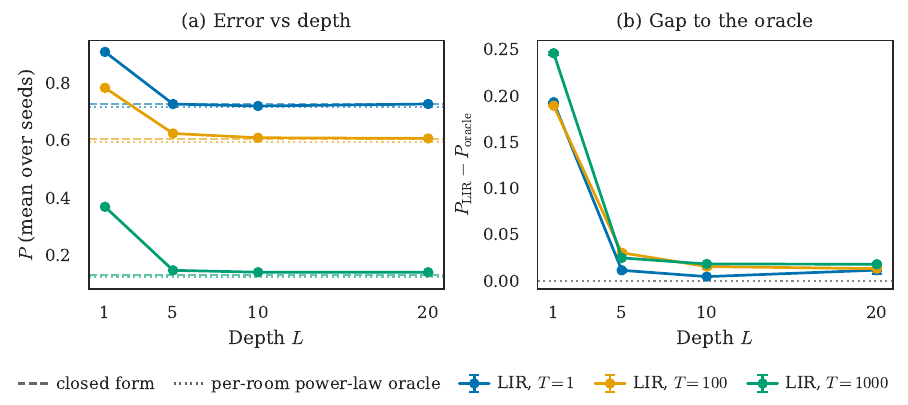}
  \caption{\textbf{LIR's error levels off by $L = 5$ near the closed form and stays above the oracle.}
  Depth beyond five layers changes the error little (a).
  Positive $\Delta P$ means LIR has higher error than the per-room power-law oracle (b).
  A single step is a constant filter, and composing several updates brings LIR to the closed form but not below the oracle.}
  \label{fig:lir-ablation}
\end{figure}

\paragraph{Depth ablation.}
Table~\ref{tab:lir-ablation} and Figure~\ref{fig:lir-ablation} report $P$ against depth $L$ and window $T$.
At $L = 1$ the map is a scaled adjoint $\hat{a} = \eta_1 A^\top y$, a constant filter that is worse than the oracle at all $T$; at $L \geq 5$ LIR is close to the closed form and above the oracle at all $T$, and only $L = 10$ at $T = 1$ is below the closed form (Table~\ref{tab:lir-ablation}).
At $T = 1$, $L = 20$ has lower training loss than $L = 10$ but higher validation $P$; at $T = 100$ and $T = 1000$ the two depths agree within the spread over seeds.

\paragraph{Classical non-diagonal alternatives.}
Wiener/LMMSE with a diagonal prior and isotropic noise, generalized Tikhonov, TSVD, early-stopped CGLS, and Landweber iteration apply fixed per-mode shrinkage profiles in the modal basis, set by the prior, so they cannot improve on the closed form's shape: a signal-matched Wiener filter with prior $\Sigma_a = \mathrm{diag}(\lambda_k^{-s})$ recovers exactly $\Gamma_k = \lambda_k^{|s|}$, and Tikhonov soft shrinkage dominates TSVD \citep{hansen1998rank}.
LIR instead composes $L$ maps with $52L$ learnable coefficients and could in principle distribute regularization across coupled modes; as trained here, it does not consistently improve on the closed form (Table~\ref{tab:lir-ablation}).

\subsection{No single diagnostic predicts the residual gap}
\label{app:gap-diagnostics}

\begin{table}[t]
\centering
\small
\caption{\textbf{Every geometry or noise descriptor that correlates with the gap at the longest window does so through room size, and none survives once area is controlled for.}
Raw and area-adjusted Spearman correlations with the per-room relative gap over the in-scope rooms, with the mechanical reconstruction of the gap from drift and curvature and the out-of-sample prediction at shorter windows.
This is why \cref{fig:delta-vs-Eop} shows a confound rather than a mechanism.}
\label{tab:gap-correlations}
\input{tables/T39_gap_correlations}\end{table}

\begin{table}[tb]
\centering
\footnotesize
\caption{\textbf{No accessible diagnostic predicts the residual gap.}
Correlations are Spearman $\rho$ with the per-room relative gap over in-scope rooms, and prediction is the cross-validated $R^2$ of a regression on all descriptors.
Out of sample they predict nothing until the longest window, where only room size acts, and only the basin width, which needs the minimum loss, correlates strongly.
The gap is thus a landscape property at long windows, not a diagnostic the learned models of \cref{sec:learned} could exploit.}
\label{tab:gap-diagnostics}
\input{tables/T18_gap_diagnostics}\end{table}

The regressions of \cref{app:features} target $p^\star$; what matters in practice is the per-room relative gap $\delta(T)$ to the per-room power-law oracle.
On the $187$ in-scope rooms the candidates are $H$, $\|E\|_{\mathrm{op}}$, the per-room Berry statistic $D_{\mathrm{KS}}$ ($\chi^2(1)$ convention, \cref{app:worst-rooms}), the eigenvalue dynamic range $(\lambda_K/\lambda_1)^{|s|}$, $K_{\mathrm{total}}$, and area, plus two quantities that need no oracle argmin: the curvature of the $\alpha$-minimized $P(p)$ at $p = |s|$ (second difference at $p \in \{1.0, 1.1, 1.2\}$; fully oracle-free) and the basin width $0.1 \times \#\{p \in [0,3] : P(p) \leq 1.05\,\min_p P\}$ (argmin-free, but it uses the minimum loss value).
Mechanically, the gap is argmin drift converted by curvature: $\delta \approx \tfrac{1}{2}(P''/P)(p^\star)\,(p^\star - |s|)^2$ reconstructs the per-room $\delta$ well at $T \leq 100$ (Table~\ref{tab:gap-correlations}), but it uses $p^\star$ and so explains without predicting.

At $T \leq 100$ no descriptor predicts the gap out of sample (Tables~\ref{tab:gap-diagnostics} and~\ref{tab:gap-correlations}).
At $T = 1000$ the raw correlations are moderate for $H$, area, and $K_{\mathrm{total}}$, weaker for $\|E\|_{\mathrm{op}}$ and $D_{\mathrm{KS}}$, and absent for the dynamic range (Table~\ref{tab:gap-correlations}).
These are one correlation in several units, since area, $K_{\mathrm{total}}$, and $H$ are perfectly rank-correlated: after controlling for area none survives (Table~\ref{tab:gap-correlations}), including the raw correlation of Figure~\ref{fig:delta-vs-Eop}.
The only strong scalar at $T = 1000$ is the basin width (Table~\ref{tab:gap-correlations}); the fully oracle-free curvature at $|s|$ is weak at every $T$.
The residual gap is therefore a landscape property at long $T$, not a geometry or anisotropy diagnostic, and no single scalar predicts it across observation times; the learned models of \cref{sec:learned} found none to exploit.

%% file: tables/T31_m3_seed_exponents.tex
\begin{tabular}{lccccc}
\toprule
 & seed 42 & seed 43 & seed 44 & seed 45 & seed 46 \\
\midrule
$\hat{p}$, $T{=}1$    & 0.65 & 1.03 & 0.42 & 0.70 & 0.82 \\
$\hat{p}$, $T{=}100$  & 0.94 & 1.01 & 0.96 & 0.99 & 0.99 \\
$\hat{p}$, $T{=}1000$ & $-0.22$ & 0.24 & 0.50 & 0.60 & 1.09 \\
\midrule
$R^2$, $T{=}1$    & 0.76 & 0.79 & 0.63 & 0.65 & 0.91 \\
$R^2$, $T{=}100$  & 0.80 & 0.82 & 0.83 & 0.83 & 0.78 \\
$R^2$, $T{=}1000$ & 0.20 & 0.08 & 0.21 & 0.23 & 0.73 \\
\bottomrule
\end{tabular}

%% file: tables/T34_learned_exponents.tex
\begin{tabular}{@{}>{\raggedright\arraybackslash}p{0.61\linewidth}>{\raggedright\arraybackslash}p{0.35\linewidth}@{}}
\toprule
Quantity at $T = 1000$ & Value \\
\midrule
M3 per-room exponent $\hat{p}$, median [IQR] over the 187 in-scope rooms & 1.13 [0.66, 1.39] \\
M2 shared exponent $\hat{p}$ & 0.62 \\
$R^2$ of a power law fitted to the M2 spectrum & 0.94 \\
Difference in median $P$ between the two exponents & 0.009 \\
Largest planted exponent that M2's parameterization recovers (\cref{app:capacity}) & 0.86 \\
Smallest per-seed $\Delta P$ against the per-room oracle over all valid configurations & $+0.002$ \\
Training configurations / numerical failures (all M3, $n \leq 400$) / failures at $T = 1000$ & 225 / 13 / 10 \\
\bottomrule
\end{tabular}

%% file: tables/T10_tier1_features.tex
\begin{tabular}{lccc}
\toprule
Feature & Spearman $\rho$ & 95\% CI & $p$-value \\
\midrule
Spectral gap ($\lambda_2 - \lambda_1$) & $-0.15$ & $[-0.30,\ 0.00]$ & 0.03 \\
Weyl exponent & $+0.07$ & $[-0.09,\ 0.23]$ & 0.33 \\
Per-room $|s|$ & $-0.04$ & $[-0.18,\ 0.11]$ & 0.58 \\
$\lambda_K$ & $+0.11$ & $[-0.04,\ 0.25]$ & 0.13 \\
\bottomrule
\end{tabular}

%% file: tables/T11_tier2_features.tex
\begin{tabular}{lccc}
\toprule
Feature & Spearman $\rho$ & 95\% CI & $p$-value \\
\midrule
Room area & $+0.21$ & $[0.05,\ 0.37]$ & 0.003 \\
$K_{\mathrm{total}}$ & $+0.21$ & $[0.05,\ 0.37]$ & 0.003 \\
Mean spacing & $-0.21$ & $[-0.37,\ -0.05]$ & 0.003 \\
Eigenvalue density & $+0.21$ & $[0.05,\ 0.36]$ & 0.003 \\
$n_{\mathrm{segments}}$ & $+0.14$ & $[0.00,\ 0.29]$ & 0.05 \\
\midrule
\multicolumn{4}{l}{Random forest on all nine features, cross-validated $R^2$ for $p^\star(T{=}1000)$: 0.14 (below the mean at $T{=}50$)} \\
\bottomrule
\end{tabular}

%% file: tables/T12_ktotal_adjusted.tex
\begin{tabular}{lcc}
\toprule
$T$ & $\delta$ using $|s|$ & $\delta$ using $\hat{p}_{\mathrm{adj}}$ \\
\midrule
1 & 0.57\% & 2.9\% \\
50 & 0.61\% & 4.1\% \\
100 & 1.46\% & 5.0\% \\
1000 & 5.83\% & 4.8\% \\
\bottomrule
\end{tabular}

%% file: tables/T13_v2_app.tex
\begin{tabular}{@{}lccc@{}}
    \toprule
    & $T = 1$ & $T = 100$ & $T = 1000$ \\
    \midrule
    $M = 8$, M3            & 0.615 & 0.444 & 0.097 \\
    $M = 8$, Closed form ($|\hat{s}|$)  & 0.607 & 0.438 & 0.095 \\
    $M = 8$, $\Delta P$    & $+0.008$ & $+0.006$ & $+0.002$ \\
    \midrule
    $M = 16$, M3           & 0.473 & 0.254 & 0.078 \\
    $M = 16$, Closed form ($|\hat{s}|$) & 0.477 & 0.251 & 0.076 \\
    $M = 16$, $\Delta P$   & $-0.004$ & $+0.003$ & $+0.002$ \\
    \bottomrule
  \end{tabular}

%% file: tables/T14_sroom_rho_vs_T.tex
\begin{tabular}{@{}rccccc@{}}
\toprule
$T$ & $\rho_S(s_{\mathrm{room}}, p^\star)$ & $p$-value & $R^2$ & $\delta(|s|)$ & $\delta(s_{\mathrm{room}})$ \\
\midrule
$1$    & $+0.05$ & $0.45$             & $<1\%$   & $0.55\%$ & $0.71\%$ \\
$50$   & $+0.40$ & $7\times 10^{-9}$  & $16\%$   & $0.60\%$ & $0.88\%$ \\
$100$  & $+0.40$ & $7\times 10^{-9}$  & $16\%$   & $1.46\%$ & $1.37\%$ \\
$500$  & $+0.14$ & $0.043$            & $2\%$    & $5.13\%$ & $4.99\%$ \\
$1000$ & $+0.16$ & $0.026$            & $3\%$    & $5.62\%$ & $5.84\%$ \\
$2100$ & $+0.18$ & $0.010$            & $3\%$    & $3.33\%$ & $3.59\%$ \\
\midrule
\multicolumn{6}{@{}p{0.74\linewidth}@{}}{Rooms that benefit from per-room tuning at $T{=}1000$: 88 of 197.} \\
\multicolumn{6}{@{}p{0.74\linewidth}@{}}{Excluding the ten boundary rooms changes each $\rho_S$ by at most 0.03 and each median $\delta$ by at most 0.22\,pp ($\rho_S$ at $T{=}1000$: 0.14, $p = 0.065$).} \\
\bottomrule
\end{tabular}

%% file: tables/T38_capacity_check.tex
\begin{tabular}{lccc}
\toprule
Planted exponent $p_{\mathrm{target}}$ & 0.5 & 1.5 & 2.5 \\
\midrule
Recovered by M1 (CondNet, unrolled) & 0.18 & 0.72 & 0.47 \\
Recovered by M2 (shared $\Gamma$) & 0.59 & 0.81 & 0.86 \\
Recovered by M3 (CondNet, closed-form) & 0.38 & 1.33 & 2.38 \\
\bottomrule
\end{tabular}

%% file: tables/T15_lir_ablation.tex
\begin{tabular}{@{}l ccc@{}}
\toprule
$L$ & $T = 1$ & $T = 100$ & $T = 1000$ \\
\midrule
1   & $0.908 \pm 0.001$ & $0.783 \pm 0.000$ & $0.368 \pm 0.002$ \\
5   & $0.726 \pm 0.002$ & $0.624 \pm 0.002$ & $0.147 \pm 0.000$ \\
10  & $0.720 \pm 0.001$ & $0.609 \pm 0.003$ & $0.141 \pm 0.001$ \\
20  & $0.727 \pm 0.004$ & $0.607 \pm 0.002$ & $0.140 \pm 0.001$ \\
\midrule
Closed form & 0.726 & 0.603 & 0.131 \\
Oracle & 0.715 & 0.594 & 0.122 \\
\midrule
\multicolumn{4}{l}{Training loss at $T = 1$ (mean over seeds): $L = 10$ 0.709, $L = 20$ 0.702.} \\
\bottomrule
\end{tabular}

%% file: tables/T39_gap_correlations.tex
\begin{tabular}{lcc}
\toprule
Descriptor (in-scope rooms, $T = 1000$) & Raw Spearman $\rho$ with $\delta$ & Partial $\rho$ given area \\
\midrule
Herfindahl index $H$ & $-0.42$ & $-0.08$ \\
Room area and $K_{\mathrm{total}}$ (perfectly rank-correlated) & $+0.42$ & -- \\
$\|E\|_{\mathrm{op}}$ & $-0.30$ & $+0.10$ \\
$D_{\mathrm{KS}}$ ($+0.19$ on the eligible subset of \cref{app:worst-rooms}) & $+0.16$ & $+0.05$ \\
Eigenvalue dynamic range $(\lambda_K/\lambda_1)^{|s|}$ & $-0.04$ & $-0.03$ \\
Basin width on $p \in [0, 3]$ & $-0.64$ ($p < 10^{-4}$) & -- \\
\midrule
\multicolumn{3}{l}{All partial correlations have $p \geq 0.16$; $\rho(\mathrm{area}, K_{\mathrm{total}}) = +1.00$ and $\rho(H, K_{\mathrm{total}}) = -1.00$.} \\
\multicolumn{3}{l}{Drift-times-curvature reconstruction of $\delta$, log-log $R^2$ at $T = 1$ / $50$ / $100$: 0.87 / 0.78 / 0.75.} \\
\multicolumn{3}{l}{Random-forest cross-validated $R^2$ of the gap from all descriptors at $T \leq 100$: $\leq 0$.} \\
\bottomrule
\end{tabular}

%% file: tables/T18_gap_diagnostics.tex
\begin{tabular}{@{}r p{0.34\linewidth} cc@{}}
\toprule
$T$ & Geometry/noise descriptors & Curvature at $p = |s|$ & Basin width on $p \in [0,3]$ \\
\midrule
1    & no out-of-sample prediction & $\rho \approx 0$ & $\rho = +0.07$ ($p = 0.32$) \\
100  & no out-of-sample prediction & $\rho = -0.17$ ($p = 0.02$), weak & $\rho = -0.04$ ($p = 0.55$) \\
1000 & six descriptors give $R^2 = 0.18$ (relative gap) / $0.25$ (absolute), driven by room size; same-size partial correlations are insignificant & $\rho = -0.04$ ($p = 0.61$) & $\rho = -0.64$ ($p < 10^{-4}$) \\
\bottomrule
\end{tabular}

%% file: supplementary/E_heat_equation.tex

\section{Heat Equation: Extended Results}
\label{app:heat}

Section~\ref{sec:heat} predicts the two-parameter regularizer $\Gamma_k \propto \lambda_k^s\, e^{c\lambda_k}$ with $c = 2\kappa t$.
All five-room analyses below use the diagnostic rooms 00805, 00826, 00840, 00880, and 00950.


\begin{table}[t]
\centering
\small
\caption{\textbf{Heat modes die out many orders of magnitude faster than acoustic modes, a pure power law cannot follow them, and the two-parameter form recovers the predicted rate.}
The table collects mode survival, fit quality of the two spectral forms, the oracle exponents each physics needs, absolute error levels, and how the heat truncation noise concentrates with time.
They back the claim of \cref{sec:heat} that the correction is predicted, not fitted.}
\label{tab:heat-facts}
\input{tables/T40_heat_facts}\end{table}

\subsection{Why heat diffusion differs from acoustics}
\label{app:heat-why-different}

In acoustics every mode decays at the same rate $\gamma$, so relative mode amplitudes stay fixed: the prior spectrum $\sigma^2_{a,k} \propto \lambda_k^{-s}$ persists, and the optimal regularizer is a pure power law at every snapshot.
Heat diffusion breaks this: its Green's function damps mode $k$ as $e^{-\kappa \lambda_k t}$, where $\kappa$ is the thermal diffusivity, so high-frequency modes decay exponentially faster than low-frequency modes and soon carry negligible amplitude.
The amplitude spectrum becomes a power law times an exponential,
\begin{equation}
\sigma^2_{a,k}(t) \propto \lambda_k^{-s} \cdot e^{-2\kappa t \lambda_k},
\end{equation}
where $\lambda_k^{-s}$ carries the initial conditions, $e^{-2\kappa t \lambda_k}$ the temporal evolution, and the factor $2$ comes from $\sigma^2_{a,k}$ being a variance.
In room 00840, high heat modes retain almost none of their initial energy by the median observation time, while the acoustic envelope still holds a sizeable fraction at twice that time (Table~\ref{tab:heat-facts}).
At high $k$, heat and acoustic survival differ by tens of orders of magnitude (Figure~\ref{fig:heat-survival}; Table~\ref{tab:heat-facts}); a one-parameter power law cannot suppress modes that decay this fast, so the heat regularizer needs the exponential factor.

\begin{figure}[t]
  \centering
  \includegraphics[width=\linewidth]{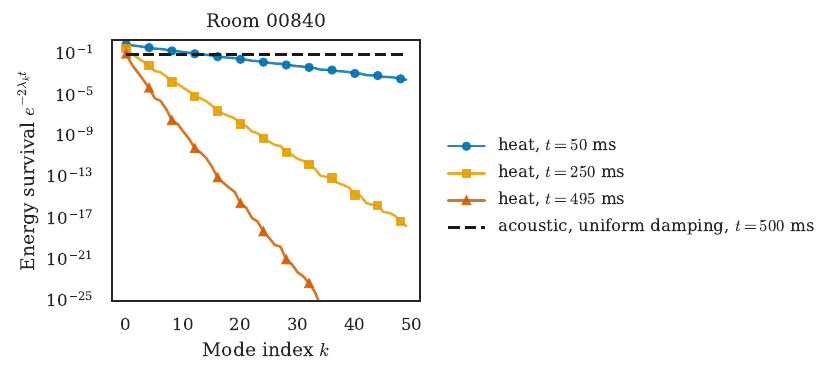}
  \caption{\textbf{High-frequency heat modes decay exponentially faster than acoustic modes, which all share one damping rate.}
  Survival is the fraction of a mode's initial energy left at time $t$, so on this log scale the steep heat curves mean that high modes carry almost no energy at late times.
  This mode-dependent decay is the exponential factor of \cref{eq:heat-spectrum}, which a pure power law cannot represent.}
  \label{fig:heat-survival}
\end{figure}

\paragraph{Diffusivity convention.}
The synthetic experiments set $\kappa = 1$ in simulation units (eigenvalues in $\mathrm{m}^{-2}$, time in $\mathrm{s}$, so $\kappa \lambda t$ is dimensionless), a computational convenience rather than a material; realistic diffusivities range from $\sim 10^{-7}\,\mathrm{m}^2/\mathrm{s}$ (water) to $\sim 10^{-4}\,\mathrm{m}^2/\mathrm{s}$ (metals), and a physical deployment substitutes the material's $\kappa$ into $c = 2\kappa t$.
The predictions hold for any $\kappa > 0$; $\kappa$ only shifts which snapshot indices fall in the regime where isotropy fails (Appendix~\ref{app:heat_herfindahl}).

\paragraph{Heat exponent convention.}
The heat regularizer uses $|s|=1.0$, the value implied by the synthetic initial-condition design $\mathrm{Var}(a_k(0)) \propto (1+\lambda_k)^{-1}$.
OLS fits on heat modal data return $|\hat{s}|$ slightly below one (Table~\ref{tab:heat-facts}); the deviation is the finite-$k$ effect of the $(1+\lambda_k)$ shift at small $k$, not an independent measurement.
A physical deployment estimates $|s|$ from the system at hand with the log-log OLS procedure of the acoustic experiments (\cref{sec:acoustic}); the synthetic value does not transfer across excitation regimes (\cref{app:heat-1vs2}).

\subsection{The three-step fitting procedure}
\label{app:heat-fitting}

\paragraph{Step 1: Spectral fit.}
For each room and observation time $t$, OLS fits the empirical amplitude variances $\hat{\sigma}^2_{a,k}(t)$, $k = 1, \ldots, K$, in log space,
\begin{equation}
\log \hat{\sigma}^2_{a,k}(t) = \beta_0 - |s| \log \lambda_k - c\,\lambda_k + \epsilon_k,
\end{equation}
with predictors $\log \lambda_k$ (power-law component) and $\lambda_k$ (exponential component), yielding $\hat{s}$, $\hat{c}$, and $R^2$.
The two-parameter model fits almost perfectly in the diagnostic rooms, a pure power law markedly worse (Figure~\ref{fig:heat-amplitude}; Table~\ref{tab:heat-facts}), and the power-law misfit is structural: its residuals are U-shaped at every snapshot and grow to tens of log-units at late snapshots, two orders of magnitude above those of the two-parameter model (Figure~\ref{fig:heat-residuals}).
In room 00840 the power-law fit worsens with observation time while its slope grows several-fold, absorbing the exponential factor instead of estimating the true $|s|$ (Table~\ref{tab:heat-facts}).

\begin{figure}[t]
  \centering
  \includegraphics[width=\linewidth]{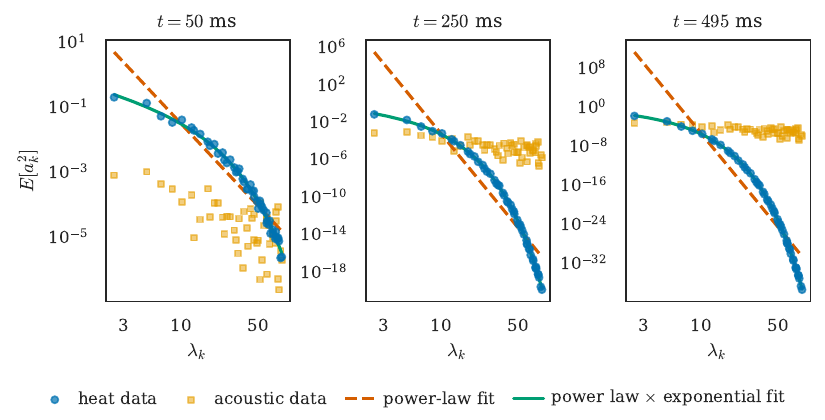}
  \caption{\textbf{An exponential-times-power-law form follows the heat amplitude spectrum, while a pure power law misses its roll-off at high eigenvalues.}
  Each panel plots mean squared modal amplitude against eigenvalue on log-log axes at one observation time, so a power law appears straight and the exponential factor bends it downward.
  The bend grows with time, as the rate $c = 2\kappa t$ predicts.
  This is the two-parameter spectrum of \cref{eq:heat-spectrum}, whose inverse is the heat regularizer.}
  \label{fig:heat-amplitude}
\end{figure}

\begin{figure}[t]
  \centering
  \includegraphics[width=\linewidth]{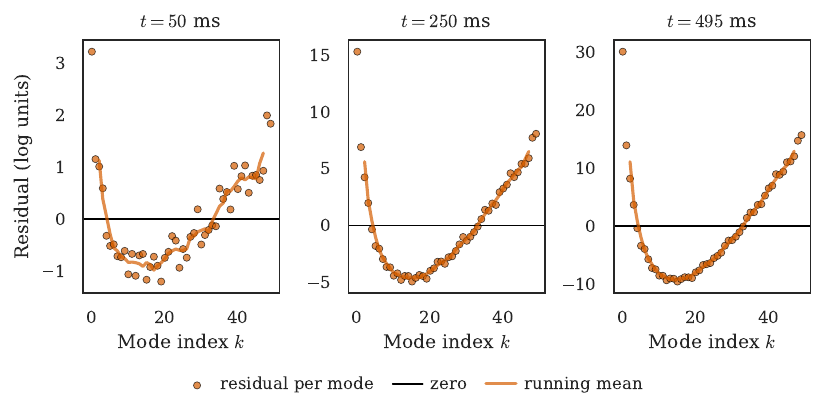}
  \caption{\textbf{A pure power law misfits the heat spectrum systematically, and the misfit grows with time.}
  Each point is the log-space residual of one mode, fit minus data, so zero is a perfect fit, and the curve is a running mean.
  The U-shape at every time is the signature of the missing exponential factor, which the power-law slope absorbs instead of estimating $|s|$.
  The two-parameter form of \cref{eq:heat-spectrum} removes this structure (\cref{fig:heat-amplitude}).}
  \label{fig:heat-residuals}
\end{figure}

\paragraph{Step 2: Rate verification.}
The Green's function predicts the fitted rate without free parameters, $c_{\mathrm{theory}} = 2\kappa t$, with $\kappa$ known a priori ($1$ in the synthetic units; a material value in deployments) and $t$ the observation time.
For each diagnostic room, the estimates $\hat{c}(t)$ at multiple observation times are regressed on the prediction,
\begin{equation}
\hat{c}(t) = \alpha_0 + \alpha_1 \cdot c_{\mathrm{theory}}(t) + \epsilon,
\end{equation}
and a correct theory gives $\alpha_1 \approx 1$ and $\alpha_0 \approx 0$.
All five slopes lie close to one with near-perfect fits (Figure~\ref{fig:c-vs-t}, Table~\ref{tab:per-room-c-fit}), and the confidence interval of the pooled slope contains one (Table~\ref{tab:heat-facts}).

\begin{figure}[t]
  \centering
  \includegraphics[width=\linewidth]{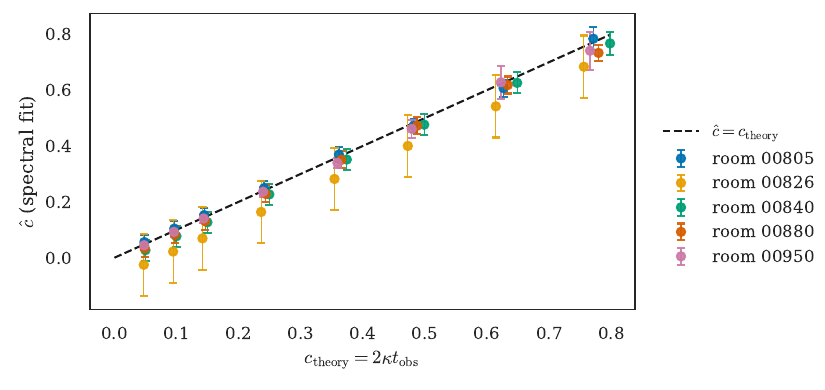}
  \caption{\textbf{The fitted exponential rate follows the prediction $c = 2\kappa t$ in every diagnostic room.}
  Each point is the rate fitted to one room at one observation time, with its standard error, so a point on the line means theory alone gives the rate.
  One room sits below the line by a constant offset that leaves its slope unchanged (\cref{tab:per-room-c-fit}).
  This supports setting the second parameter of \cref{eq:heat-spectrum} from $\kappa$ and $t$ alone, a consistency check on synthetic data rather than physical validation.}
  \label{fig:c-vs-t}
\end{figure}

\begin{table}[t]
\centering
\caption{\textbf{The fitted rate follows $c_{\mathrm{theory}}$ in every diagnostic room.}
Each room's row regresses $\hat{c}$ on $c_{\mathrm{theory}}$ across observation times, where a slope of one with an intercept near zero means agreement, and the pooled row fits all five rooms together.
\Cref{fig:c-vs-t} plots the underlying points of this check on the known second parameter of \cref{eq:heat-spectrum}.}
\label{tab:per-room-c-fit}
\input{tables/T19_per_room_c_fit}\end{table}

\paragraph{Step 3: Prediction, not post-hoc fit.}
The rate $c = 2\kappa t$ follows from $\kappa$ and $t$ before any data are collected; only $|s|$ is estimated from data.
In these synthetic data $\hat{c} \approx 2\kappa t$ is expected, because the same eigenvalues enter the forward model and the modal basis, so the recovery is a consistency check on the fitting procedure; the independent test of the framework is the acoustic FDTD experiments of \cref{sec:acoustic}.

\subsection{One-parameter fallback and two-parameter gain}
\label{app:heat-1vs2}

If the optimal regularizer is $\Gamma_k\propto\lambda_k^s e^{c\lambda_k}$ but only $\Gamma_k=\lambda_k^p$ is searched, the fitted exponent rises above $s$ to absorb the exponential roll-off at high modes.
Figure~\ref{fig:p-sweep-curves-app}(e) shows the rise: acoustic $p^\star$ stays near $|s|=1.13$ up to $T=200$, whereas heat $p^\star$ is already more than twice $|s|$ at $T=1$ (the exponential is present even at a single snapshot) and rises further for $T \geq 1000$ (Table~\ref{tab:heat-facts}).
The acoustic minima in panels (a) and (b) move up between $T = 1$ and $T = 2100$, the heat minima in panels (c) and (d) do the same, and in panel (e) heat $p^\star$ stays nearly constant for most $T \leq 200$ (Table~\ref{tab:heat-facts}).

\begin{figure}[t]
  \centering
  \includegraphics[width=\linewidth]{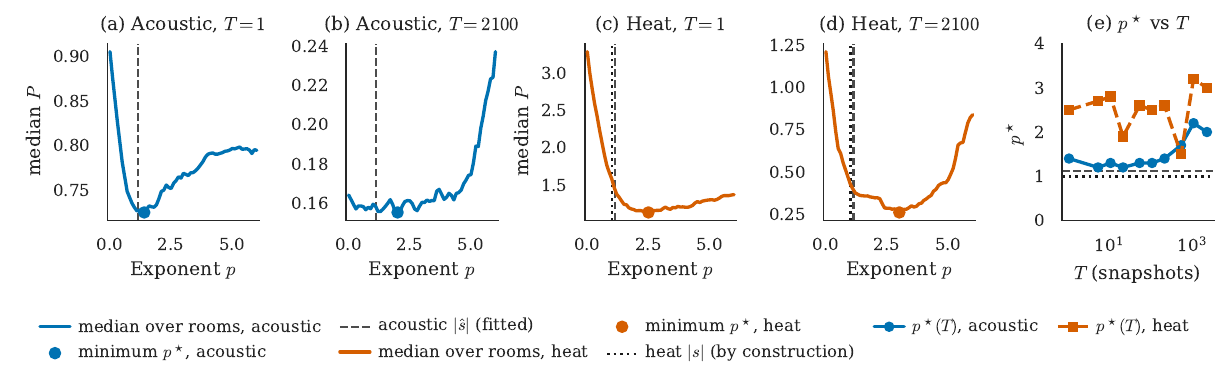}
  \caption{\textbf{Heat needs a steeper exponent than $|s|$ at every window, acoustics only at long ones.}
  Curves are room medians with their minima (a--d).
  The acoustic minimum moves up at a long window (b), and the heat minimum exceeds both references in a single snapshot (c) and moves up further at a long window (d).
  Both trends hold across all windows (e).
  The raised heat exponent absorbs the exponential factor of \cref{eq:heat-spectrum}, with the acoustic drift in \cref{app:temporal}.}
  \label{fig:p-sweep-curves-app}
\end{figure}

\paragraph{Absolute performance ($T=2100$, $M=8$).}
Ridge ($p=0$) has $P > 1$; the acoustic exponent $\Gamma_k=\lambda_k^{1.13}$ cuts that error to about a third, and the per-room one-parameter heat oracle lowers it by a further two fifths (Table~\ref{tab:heat-facts}).
$|s|$ is system-specific ($|s|_{\mathrm{heat}}\approx 1.0$ from the $t=0$ spectrum, against $|s|_{\mathrm{acoustic}}=1.13$): the procedure transfers, the number does not.

\paragraph{Two-parameter gain.}
The population gain of the second parameter is small (Table~\ref{tab:two-param-gain}) because one $c$ must serve all rooms at fixed $T$; per room, the two-parameter oracle improves on the one-parameter oracle by a few percent in the median, with a skewed spread: many rooms gain nothing and a quarter gain more than a tenth.
The difference reflects room-to-room variation in the optimal $c$, so for heat the per-room opportunity is not uniformly small.

\begin{table}[t]
\centering
\small
\caption{\textbf{The second parameter barely lowers the population heat error, and its per-room gains are skewed.}
$P^\star$ is the normalized modal error at the optimal parameters over the validation rooms; it exceeds one because heat amplitudes have a nonzero mean, which the variance in its denominator excludes.
The improvement columns give the relative reduction when one $c$ serves all rooms and its per-room median [IQR] when each room has its own parameters.}
\label{tab:two-param-gain}
\input{tables/T21_two_param_gain}\end{table}

\subsection{Per-room spectral fits}
\label{app:heat-per-room}

Table~\ref{tab:per-room-spectral-fits} checks the rate prediction room by room.
The slope $\hat{\alpha}_1$ and its fit quality $R^2_{\alpha_1}$ come from the Step~2 regression of \cref{app:heat-fitting} across all 8 observation times; $\hat{c}$, $c_{\mathrm{theory}}$, and $\hat{c}/c_{\mathrm{theory}}$ are single-time spectral-fit estimates at the snapshot nearest the median observation time, where $c_{\mathrm{theory}} = 2\kappa t$ with $\kappa = 1.0$ and $t$ set by each room's own time step, so $c_{\mathrm{theory}}$ differs slightly from room to room.
The across-time slope is close to one in every room.
The single-time ratio is within a few percent of one in four rooms and lower in room 00826, whose single-time spectral fits are the least well determined, as the last column shows.

\begin{table}[t]
\centering
\small
\caption{\textbf{Per-room spectral fits recover $\hat{c} \approx c_{\mathrm{theory}}$ across observation times in all five diagnostic rooms, and at a single time in four of them.}
A slope $\hat{\alpha}_1$ of one means that $\hat{c}(t)$ tracks $c_{\mathrm{theory}}$ across observation times, and a ratio of one means agreement at the single representative time.
The last column is the $R^2$ of that single-time fit and accounts for the room that falls short.
This room-level agreement backs \cref{sec:heat}, where the rate in \cref{eq:heat-spectrum} is predicted rather than fitted.}
\label{tab:per-room-spectral-fits}
\input{tables/T22_per_room_spectral_fits}\end{table}

\paragraph{Two timescales.}
The \emph{window duration} is $T \cdot \Delta t_{\mathrm{sim}}$, with the FDTD step $\Delta t_{\mathrm{sim}} \approx 24$\,$\mu$s set by the per-room CFL condition; the \emph{observation time} is the snapshot index times $\Delta t_{\mathrm{snap}} = 100 \cdot \Delta t_{\mathrm{sim}} \approx 2.4$\,ms.
Each reported $P(T)$ averages over ${\sim}190$ valid snapshots per room spanning $t_{\mathrm{obs}} \in [12, 493]$\,ms with median ${\sim}250$\,ms, a distribution roughly independent of $T$; the representative $t_{\mathrm{obs}} \approx 250$\,ms is that median.

\subsection{Herfindahl index degradation under heat diffusion}
\label{app:heat_herfindahl}

In acoustics the Herfindahl index (Table~\ref{tab:isotropy-facts}) is time-independent because all modes share the damping rate $\gamma$.
In heat diffusion the discarded-mode noise weights $w_n(t) \propto \lambda_n^{-s} \cdot e^{-2\kappa \lambda_n t}$ carry an exponential factor that suppresses high-frequency modes and concentrates the noise power into the lowest few discarded modes as $t$ grows.
Table~\ref{tab:heat_herfindahl} tracks this concentration through $H(t)$ and the effective contributor count $1/H(t)$.
The simulated heat signals of \cref{sec:heat} contain only the retained modes plus sensor noise, so they carry no truncation noise; the table is an analytic statement about the discarded modes a heat measurement would add.
By $t_{\mathrm{obs}} \approx 250$\,ms, about the median observation time, $1/H$ has dropped by an order of magnitude from its value at the earliest observation time, with the largest degradation in the smallest rooms (Table~\ref{tab:heat-facts}).
At the latest observation time, $H$ approaches one in the smallest diagnostic room, where the first discarded mode dominates because $e^{-2\kappa \lambda_{K+1} t_{\mathrm{obs}}}$ is exponentially larger than every later factor; the median over the five rooms is about one half (Table~\ref{tab:heat-facts}), concentrated but not rank-1.

\begin{table}[t]
\caption{\textbf{Heat diffusion concentrates the truncation noise into a few discarded modes as time grows, whereas acoustic noise stays spread over many modes at any time.}
$H$ is the Herfindahl index of \cref{eq:herfindahl} with the decaying heat weights, $1/H$ counts the discarded modes that effectively contribute, and a small $H$ means nearly isotropic noise.
\Cref{prop:isotropy} needs isotropic noise, so with discarded modes present its diagonal optimality for heat would hold only at early observation times.}
\label{tab:heat_herfindahl}
\centering
\small
\input{tables/T23_heat_herfindahl}\end{table}

\paragraph{Why the spectral fit still works.}
The fits of \cref{sec:heat} (Tables~\ref{tab:per-room-c-fit} and~\ref{tab:heat-facts}) verify that the signal-amplitude variance $\hat{\sigma}^2_{a,k}(t)$ follows $\lambda_k^{-s} \cdot e^{-2\kappa \lambda_k t}$, a property of the Green's function acting on the prior that does not require isotropic noise.
The diagonal optimality of Proposition~\ref{prop:isotropy} does require isotropic noise and therefore holds at early $t_{\mathrm{obs}}$ only; at late $t_{\mathrm{obs}}$, where $H$ is concentrated, a non-diagonal estimator could in principle reach lower error than the diagonal regularizer.

\paragraph{What the heat experiments test.}
In the simulated heat data the only noise is sensor noise, which is isotropic at every observation time, so the late-time concentration of $H$ does not enter them.
With discarded modes present, the early snapshots, where $H$ is still small (Table~\ref{tab:heat-facts}) and isotropy holds, carry the strongest signal and would dominate the average over the ${\sim}190$ snapshots of each $P(T)$; testing this needs heat data with truncation noise.

%% file: tables/T40_heat_facts.tex
\begin{tabular}{@{}>{\raggedright\arraybackslash}p{0.61\linewidth}>{\raggedright\arraybackslash}p{0.35\linewidth}@{}}
\toprule
Quantity & Value \\
\midrule
Room 00840: energy retained by modes $k \geq 24$ at $t = 250$\,ms / acoustic envelope at $t = 500$\,ms & $< 10^{-9}$ / about 8\% \\
Heat versus acoustic survival at $k = 49$, $t = 250$ / $495$\,ms (orders of magnitude) & 17 / 34 \\
OLS $|\hat{s}|$ on heat modal data, room 00840 / finite-$k$ prediction for the five rooms (design value 1.0) & 0.96 / 0.86--0.98 \\
Spectral fit $R^2$, room 00840 at the three times of Figure~\ref{fig:heat-amplitude}: two-parameter form / pure power law & $> 0.99$ / 0.86--0.91 \\
Residuals at late snapshots (log units): pure power law / two-parameter form & ${\sim}30$ / ${\sim}0.3$ \\
Room 00840, pure power law at $t = 50$ / $250$ / $495$\,ms: $R^2$ (fitted slope) & 0.91 (3.6) / 0.87 (14.4) / 0.86 (27.5) \\
Pooled regression slope of $\hat{c}$ on $c_{\mathrm{theory}}$, 95\% CI & 0.987 [0.953, 1.022] \\
Heat oracle exponent $p^\star$: $T = 1$ / $T \leq 200$ / $T \geq 1000$ & 2.5 / 2.5--2.8 / 3.0--3.2 \\
Acoustic oracle exponent $p^\star$: $T = 1$ / $T = 1000$ & 1.4 / $\approx 2.2$ \\
$P$ at $T = 2100$, $M = 8$: ridge / $\lambda_k^{1.13}$ / per-room one-parameter heat oracle (median $p^\star = 3.2$) & 1.210 / 0.399 / 0.235 \\
Heat Herfindahl index at the earliest observation time ($t \approx 12$\,ms) & $\approx 0.027$ \\
Heat effective contributors $1/H$: $t = 12$\,ms / $t \approx 250$\,ms & ${\sim}20$--$80$ / ${\sim}1$--$6$ \\
Heat $H$ at $t = 493$\,ms: smallest room (00950, $K_{\mathrm{total}} = 93$) / median over the five rooms ($1/H$) & 0.99 / 0.548 ($\approx 1.8$) \\
\bottomrule
\end{tabular}

%% file: tables/T19_per_room_c_fit.tex
\begin{tabular}{lccc}
\toprule
Room & Slope $\alpha_1$ & Intercept $\alpha_0$ & $R^2$ \\
\midrule
00805 & 0.984 & $+0.009$ & 0.9984 \\
00826 & 1.000 & $-0.072$ & 1.0000 \\
00840 & 0.993 & $-0.021$ & 0.9999 \\
00880 & 0.973 & $-0.010$ & 0.9990 \\
00950 & 0.981 & $-0.002$ & 0.9986 \\
\midrule
\textbf{Pooled} & \textbf{0.987} & $-0.020$ & 0.9880 \\
\bottomrule
\end{tabular}

%% file: tables/T21_two_param_gain.tex
\begin{tabular}{lcccc}
\toprule
$T$ & 1-param $P^\star$ & 2-param $P^\star$ & Pop.\ improv. & Per-room median [IQR] \\
\midrule
1   & 1.275 & 1.256 & 1.5\% & 4.0\% [0.1, 13.8] \\
10  & 1.244 & 1.244 & 0.0\% & 4.4\% [0.2, 13.1] \\
100 & 0.871 & 0.871 & 0.0\% & 3.1\% [0.0, 13.3] \\
500 & 0.738 & 0.727 & 1.4\% & 4.0\% [0.0, 12.2] \\
\bottomrule
\end{tabular}

%% file: tables/T22_per_room_spectral_fits.tex
\begin{tabular}{lcccccc}
\toprule
Room & $\hat{\alpha}_1$ & $R^2_{\alpha_1}$ & $c_{\mathrm{theory}}$ & $\hat{c}$ & $\hat{c}/c_{\mathrm{theory}}$ & $R^2_{\mathrm{spec}}$ \\
\midrule
00805 & 0.98 & 0.9984 & 0.482 & 0.472 & 0.981 & 0.980 \\
00826 & 1.00 & 1.0000 & 0.472 & 0.400 & 0.847 & 0.615 \\
00840 & 0.99 & 0.9999 & 0.499 & 0.476 & 0.955 & 0.954 \\
00880 & 0.97 & 0.9990 & 0.487 & 0.472 & 0.969 & 0.965 \\
00950 & 0.98 & 0.9986 & 0.478 & 0.461 & 0.963 & 0.985 \\
\bottomrule
\end{tabular}

%% file: tables/T23_heat_herfindahl.tex
\begin{tabular}{rcccl}
\toprule
$t_{\mathrm{obs}}$ & Median $H$ & $H$ range (5 rooms) & $1/H$ range & Isotropy regime \\
\midrule
12~ms  & 0.027 & [0.013, 0.052] & [19, 76]  & Holds (early) \\
250~ms & 0.357 & [0.156, 0.815] & [1.2, 6.4] & Degraded (median) \\
493~ms & 0.548 & [0.283, 0.987] & [1.0, 3.5] & Failed in smallest room \\
\midrule
Acoustic & $\approx 0.005$ & any $t$ & $\approx 200$ & Holds \\
\bottomrule
\end{tabular}

%% file: supplementary/F_prior_robustness.tex

\section{Prior Robustness}
\label{app:prior-robustness}

The theory assumes a Gaussian, independent, power-law prior ($\mathbf{a} \sim \mathcal{N}(0, \Sigma_{\mathbf{a}})$ with $\Sigma_{kk} \propto \lambda_k^{-s}$), and the simulated data of the main experiments follow its power-law variance decay (Appendix~\ref{app:estimating-s}).
What happens when the prior is wrong?
A heavy-tailed and a correlated alternative test whether $p^\star$, the cost of the closed form, and the landscape flatness change: with the discarded modes kept, so that truncation noise is present (\cref{app:prior-inscope}).
\cref{app:prior-misspec-learned} repeats the comparison for the learned diagonal model M3.

\subsection{Experimental setup}
\label{app:prior-setup}

\paragraph{The three priors.}
All three share the marginal variance $\sigma^2_{a,k} = \lambda_k^{-|s|}$ and differ in shape.
\begin{enumerate}
\item[(a)] \textbf{Gaussian (baseline).}
$a_k \sim \mathcal{N}(0, \lambda_k^{-|s|})$, independent across modes: the prior the theory assumes.
\item[(b)] \textbf{Heavy-tailed: Student-$t$ with $\nu = 3$.}
$a_k \sim t_3(0, \lambda_k^{-|s|})$, independent across modes and scaled to the Gaussian variance, with infinite kurtosis; it models occasional extreme modal amplitudes, as from a localized impact.
$\nu = 3$ is the most adversarial finite-variance choice: at $\nu = 2$ the variance is infinite, and at $\nu = 5$ the kurtosis is already $9$.
\item[(c)] \textbf{Correlated: adjacent-mode correlation $\rho = 0.3$.}
$\mathbf{a} \sim \mathcal{N}(0, \Sigma_{\mathrm{corr}})$ with diagonal entries $\lambda_k^{-|s|}$ and $(\Sigma_{\mathrm{corr}})_{jk} = \rho \sqrt{\lambda_j^{-|s|} \lambda_k^{-|s|}}$ for $|j - k| = 1$; it breaks the independence assumption and models a spatially extended source that excites neighboring modes together (larger $\rho$ makes $\Sigma_{\mathrm{corr}}$ near-singular).
\end{enumerate}

\paragraph{Protocol.}
For each prior, initial conditions are drawn for 20 rooms, and the closed-form Tikhonov solution sweeps $p$ over the 61-point grid at $T \in \{1, 100\}$, recording the oracle exponent $p^\star$, the cost $\delta(|s|) = (P(|s|) - P(p^\star)) / P(p^\star)$, and the flatness ratio $\max P / \min P$ over $p \in [0, 3]$.
In the \emph{primary} test, the prior drives all $K_{\mathrm{total}}$ modes of each room and the measurements are built from all of them, so the estimator, which models the retained $K = 50$, faces truncation noise as in \cref{sec:acoustic}; $16$ rooms of the set have $K_{\mathrm{total}} > K$ and hence discarded modes (the other four have none).
The primary test is evaluated as exact population risk (the risk of a fixed linear estimator depends on the prior only through its second moments, Proposition~\ref{prop:isotropy}) and empirically over $24$ independent seeds.
In the \emph{no-truncation control}, only $\min(K, K_{\mathrm{total}})$ modes are generated, so no truncation noise is present; the control uses one realization per room and lies outside the regime of Proposition~\ref{prop:isotropy}.
$T = 1000$ is omitted because without truncation noise $P \to 0$ for every $p$ and $\delta$ is undefined in the control.


\begin{table}[t]
\centering
\small
\caption{\textbf{The alternative priors leave the noise at the sensors close to Gaussian, and shape freedom pays only once the truncation noise is removed.}
The table collects how close the alternative priors stay to Gaussian noise, the largest movement under the correlated prior, and the one control in which the trained M3 beats the best power law.
They support the scope statement of \cref{tab:scope-layers}.}
\label{tab:prior-facts}
\input{tables/T41_prior_facts}\end{table}

\subsection{Results in the truncation regime (primary)}
\label{app:prior-inscope}

Table~\ref{tab:prior-inscope} gives the primary result, with $M = 8$ sensors.
In population risk, all three priors give the same $p^\star$, equal to $|s|$ at grid resolution, with negligible $\delta(|s|)$ and the same flatness (Table~\ref{tab:prior-inscope}).
Empirically, the Gaussian and the $t_3$ prior give the same median $p^\star$ at $T = 100$ with a wide per-seed range, and the closed form's median oracle cost stays below one percent at both windows (Table~\ref{tab:prior-inscope}).

\begin{table}[t]
\centering\small
\caption{\textbf{In the truncation regime, a misspecified prior leaves $p^\star$ and the closed form's cost unchanged.}
Entries are medians over in-scope rooms, where $\delta(|s|)$ is the closed form's relative extra error over the oracle exponent and a flatness ratio near one means $p$ hardly matters.
The upper block is exact population risk, the lower block empirical with per-seed ranges.
Gaussian and variance-matched $t_3$ population rows coincide by construction, since a fixed linear estimator's risk depends only on second moments (\cref{prop:isotropy}).}
\label{tab:prior-inscope}
\input{tables/T24_prior_inscope}\end{table}

Normality and isotropy play different roles: in the generator, the truncation noise at the sensors is a linear combination of the discarded-mode amplitudes, and finite-variance alternatives stay close to Gaussian in marginal distribution (Table~\ref{tab:prior-facts}).
One finite-sample effect remains: at the fixed estimator $p = |s|$ with $\alpha$ at the Gaussian population optimum, the empirical pooled ratio is somewhat higher under $t_3$ than under the Gaussian prior (Table~\ref{tab:prior-facts}).
This is a property of the ratio under infinite kurtosis, not of the population risk, and it does not shift $p^\star$ toward zero.

\subsection{Interpretation}
\label{app:prior-interpretation}

\paragraph{Gaussian prior.}
With the truncation noise present, the population-risk $p^\star$ equals $|s|$ at grid resolution and the median $\delta$ is below one percent (Table~\ref{tab:prior-inscope}).

\paragraph{Heavy-tailed prior.}
The variance-matched $t_3$ and Gaussian priors have identical population risk (\cref{app:prior-setup}), so the heavy tail can change the closed form's cost only through finite-sample scatter, and empirically it leaves the median $p^\star$ and cost unchanged (Table~\ref{tab:prior-inscope}).

\paragraph{Correlated prior.}
In the truncation regime (Table~\ref{tab:prior-inscope}) and on all $187$ in-scope rooms (\cref{app:prior-misspec-learned}), the correlation moves the median cost of the closed form by a negligible amount (Table~\ref{tab:misspec-learned}).
The correlation $\rho = 0.3$ between adjacent modes adds mild off-diagonal structure to $\Sigma_{\mathbf{a}}$; the optimal $\Gamma$ is no longer exactly diagonal, but the diagonal dominates and the diagonal power law remains a good approximation.

\paragraph{Summary.}
All three priors give flat landscapes, for three reasons.
(i) Noise isotropy is a property of the noise (Weyl's mode count and the measured diagonal term), not of the prior.
(ii) Landscape flatness is set by the discarded-mode count (Appendix~\ref{app:matched-ktotal}), not by the prior.
(iii) Under isotropic noise at $T = 1$, the closed form is optimal for the Gaussian prior and remains the best \emph{linear} estimator for any prior with the same second moments (Proposition~\ref{prop:isotropy}); for non-Gaussian priors no quadratic $\Gamma$ is Bayes-optimal among all estimators, and the flat landscape keeps that approximation error small.
In the truncation regime these misspecifications leave $p^\star$ and the cost of the closed form unchanged (\cref{app:prior-inscope}).
Priors with infinite-variance amplitudes, noise correlated with the signal, and strongly structured noise covariance are outside the scope of these tests.

\subsection{Learned diagonal models under misspecification}
\label{app:prior-misspec-learned}
\label{app:no-trunc-control}

Are learned diagonal regularizers more robust to prior misspecification than the closed form?
M3's conditioning features depend only on the room and the observation operator, so its emitted $\Gamma$ is prior-independent, and the risk of any fixed linear estimator depends on the prior only through its covariance (Proposition~\ref{prop:isotropy}).
The closed form and M3 shapes trained with the protocol of \cref{sec:learned} (five seeds at $n = 800$, evaluated under each prior rather than retrained for it, with $\alpha$ re-optimized over the paper's grid for every method) are therefore compared under exact population risk with truncation noise, on all $187$ in-scope rooms at $T \in \{1, 100\}$ (Table~\ref{tab:misspec-learned}); here the prior diagonal is each room's realized per-mode energy (the variance of its modal amplitudes) rather than $\lambda_k^{-|s|}$.
Under the variance-matched $t_3$ prior both are unchanged by construction.
Under the correlated prior ($\rho = 0.3$), the median costs of the closed form and of M3 each move by a negligible amount relative to the Gaussian prior, and the closed form's median cost stays below M3's at both $T$; the largest single-room movement is in M3's evaluation and favors M3 (Table~\ref{tab:prior-facts}).
In these tests, learned diagonal models are not more robust to prior misspecification than the closed form: both move together.
The learned models were not retrained under the perturbed priors or under the synthetic covariances of \cref{app:aniso-sweep}; in naturally anisotropic rooms they do not beat the per-room power-law oracle (\cref{sec:learned}).

\begin{table}[t]
\centering\small
\caption{\textbf{Under a correlated prior the costs of the closed form and M3 move together, and the closed form stays below M3.}
Entries are median relative costs against the per-room power-law oracle under exact population risk with truncation noise, with M3 as mean and standard deviation over seeds.
Variance-matched $t_3$ rows are omitted because they coincide with the Gaussian rows by construction.
M3 is no more robust to a wrong prior than the closed form, consistent with the saturation of \cref{sec:learned}.}
\label{tab:misspec-learned}
\input{tables/T26_misspec_learned}\end{table}

\paragraph{No-truncation control: where shape freedom pays.}
In the no-truncation control of \cref{app:prior-setup} (the same $20$ rooms with the discarded modes removed synthetically), the same trained M3 checkpoints beat the best power law by a clear median margin at $T = 100$ (Table~\ref{tab:prior-facts}).
That control has no discarded modes and hence no truncation noise, so it lies outside the regime of Proposition~\ref{prop:isotropy}: removing the truncation noise restores landscape curvature (cf.\ the $K_{\mathrm{total}} = K$ case of \cref{app:rect-landscape}), and shape freedom pays again.
The saturation claim of \cref{sec:learned} is therefore restricted to the truncation regime; this control is its one explicit exception.

%% file: tables/T41_prior_facts.tex
\begin{tabular}{@{}>{\raggedright\arraybackslash}p{0.61\linewidth}>{\raggedright\arraybackslash}p{0.35\linewidth}@{}}
\toprule
Quantity & Value \\
\midrule
Rooms in the App.~F test set / rooms among them with discarded modes & 20 / 16 \\
Largest KS gap between the noise marginals of the finite-variance alternatives and $\mathcal{N}(0,1)$ (187 in-scope rooms, 10 seeds) & 0.016 (at most 0.017) \\
Empirical pooled ratio at $p = |s|$ under $t_3$ relative to the Gaussian prior & about 6\% higher \\
Largest single-room movement under the correlated prior (M3 evaluation, $T = 1$, favors M3) & 13.3\,pp \\
No-truncation control: median gain of the trained M3 over the best power law at $T = 100$ & 2.8\,pp \\
\bottomrule
\end{tabular}

%% file: tables/T24_prior_inscope.tex
\begin{tabular}{@{}ll cccc@{}}
\toprule
Prior & $T$ & $p^\star$ & $P^\star$ & $\delta(|s|)$ & Flatness \\
\midrule
\multicolumn{6}{@{}l}{\emph{Population risk}} \\
Gaussian    & 1   & 1.1 & 0.659 & 0.1\% & 1.322 \\
Gaussian    & 100 & 1.1 & 0.566 & 0.0\% & 1.274 \\
Heavy-tail ($t_3$) & 1   & 1.1 & 0.659 & 0.1\% & 1.322 \\
Heavy-tail ($t_3$) & 100 & 1.1 & 0.566 & 0.0\% & 1.274 \\
Correlated  & 1   & 1.1 & 0.656 & 0.1\% & 1.329 \\
Correlated  & 100 & 1.1 & 0.567 & 0.0\% & 1.285 \\
\midrule
\multicolumn{6}{@{}l}{\emph{Empirical, 24 seeds (median [range])}} \\
Gaussian    & 1   & 1.0 [0.8, 1.4] & 0.702 & 0.6\% [0.0, 2.6] & 1.253 \\
Gaussian    & 100 & 1.1 [1.0, 1.4] & 0.603 & 0.3\% [$-0.1$, 1.6] & 1.283 \\
Heavy-tail ($t_3$) & 1   & 0.9 [0.5, 1.6] & 0.759 & 0.7\% [0.0, 4.0] & 1.182 \\
Heavy-tail ($t_3$) & 100 & 1.1 [0.6, 1.4] & 0.658 & 0.7\% [0.0, 3.5] & 1.219 \\
Correlated  & 1   & 1.0 [0.9, 1.4] & 0.705 & 0.4\% [0.0, 1.5] & 1.252 \\
Correlated  & 100 & 1.2 [1.0, 1.4] & 0.601 & 0.3\% [0.0, 1.2] & 1.273 \\
\bottomrule
\end{tabular}

%% file: tables/T26_misspec_learned.tex
\begin{tabular}{@{}llcc@{}}
\toprule
Prior & $T$ & Closed form $\delta(|s|)$ & M3 shape $\delta$ \\
\midrule
Gaussian & 1 & 0.46\% & $1.20 \pm 0.13\%$ \\
Gaussian & 100 & 1.41\% & $2.74 \pm 0.13\%$ \\
Correlated ($\rho = 0.3$) & 1 & 0.41\% & $1.11 \pm 0.15\%$ \\
Correlated ($\rho = 0.3$) & 100 & 1.35\% & $2.78 \pm 0.20\%$ \\
\bottomrule
\end{tabular}

%% file: supplementary/G_rectangular_control_experiment.tex

\section{Rooms Where Berry's Premise Fails: Rectangles and 3D Boxes}
\label{app:weyl-dominance}

Berry's conjecture is the only assumption in the isotropy argument that is neither an identity nor a calculation, so the direct test applies the closed form $\Gamma_k = \lambda_k^{|s|}$ where the conjecture fails.
Rectangles with Dirichlet walls fail it by construction: they are integrable billiards whose eigenfunctions are products of sines, and their analytic eigenpairs give $P(p)$ without FEM discretization error.
They also change the boundary condition of the main experiments (random convex polygons with Robin, or impedance, walls), so the test combines two departures from the main setting.
The outcome: Weyl's discarded-mode count, not Berry conformity, flattens the landscape.

\subsection{Setup}
\label{app:rect-setup}

A rectangle $L_x \times L_y$ with Dirichlet walls has eigenfunctions and eigenvalues
\begin{equation}
\varphi_{mn}(x, y) = \frac{2}{\sqrt{L_x L_y}} \sin\!\Bigl(\frac{m\pi x}{L_x}\Bigr) \sin\!\Bigl(\frac{n\pi y}{L_y}\Bigr), \qquad m, n = 1, 2, 3, \ldots,
\end{equation}
\begin{equation}
\lambda_{mn} = \pi^2\!\Bigl(\frac{m^2}{L_x^2} + \frac{n^2}{L_y^2}\Bigr).
\end{equation}
Their nodal lines are straight and parallel to the walls, so at random sensor positions the cross-correlations $C_{kn}$ follow the distribution of products of sines rather than a Gaussian.
The four test rooms measure $3\times6$, $2\times8$, $4\times4$, and $3\times5$\,m (areas 15--18\,m$^2$; $K_{\mathrm{total}} = 287, 255, 255, 239$, Weyl's estimate of the count with $\lambda \leq 200$\,m$^{-2}$); the $4 \times 4$ square is the most symmetric, with frequent eigenvalue degeneracies, and the $2 \times 8$ room has a $4{:}1$ aspect ratio and a distinct eigenvalue distribution.
Each room retains $K = 50$ modes, with $M = 8$ sensors placed uniformly at random, and $P(p)$ is evaluated at $T = 100$.


\begin{table}[t]
\centering
\small
\caption{\textbf{Rectangles reject the chaotic spacing statistics that Berry's conjecture presumes, yet a few dozen discarded modes flatten their landscapes nearly as far as those of generic rooms.}
The table collects the spacing tests, the flattening of one rectangle with its mode count, the matched-count comparison, and the 3D box stress test.
They separate the role of Weyl's count from the role of \cref{conj:berry}.}
\label{tab:rectangle-box-facts}
\input{tables/T42_rectangle_box_facts}\end{table}

\subsection{Eigenvalue spacing: Berry's premise fails}
\label{app:rect-nnsd}

The standard diagnostic for ``quantum chaos'' is the nearest-neighbor spacing distribution (NNSD) of the eigenvalues, unfolded to mean spacing 1 to remove the density trend of Weyl's law.
Here $s_i = (\lambda_{i+1} - \lambda_i) / \bar{\Delta}$ is the normalized spacing, as on the axis of Figure~\ref{fig:berry-weyl-app}a, not the prior exponent.
Integrable systems such as rectangles are expected to follow the Poisson density $f(s) = e^{-s}$, and chaotic systems such as generic convex polygons, the setting of Berry's conjecture, the GOE density $f(s) = \frac{\pi s}{2}\,e^{-\pi s^2/4}$.

\begin{figure}[t]
  \centering
  \includegraphics[width=\linewidth]{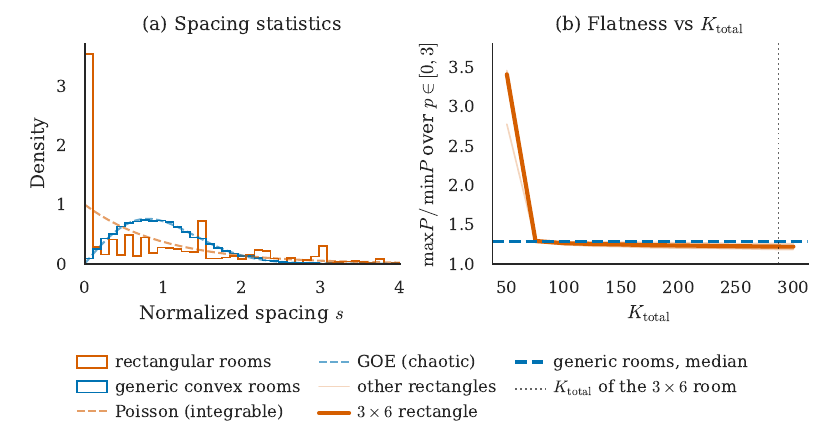}
  \caption{\textbf{Rectangles violate Berry's premise, yet discarded modes flatten their landscapes to the generic-room level.}
  Generic convex rooms show the GOE spacing of chaotic systems and rectangles do not (a).
  A ratio of $1$ means the exponent does not matter (b).
  The rectangle curve is one click-excited realization, so its no-truncation anchor is not comparable cell by cell with \cref{tab:matched-ktotal}.
  Weyl's law (\cref{eq:weyl}) gives practical rooms hundreds of discarded modes, and this count, not conformity to \cref{conj:berry}, sets the flattening.}
\label{fig:berry-weyl-app}
\end{figure}

Generic convex rooms match the GOE closely, consistent with Berry's conjecture (Figure~\ref{fig:berry-weyl-app}a; Table~\ref{tab:rectangle-box-facts}).
Rectangular rooms reject it decisively, every one of them, and two of the four lie closer to Poisson, while the $3 \times 6$ and $4 \times 4$ rooms are as far from Poisson as from GOE because their degeneracies produce exact-zero spacings (Table~\ref{tab:rectangle-box-facts}).
The $4 \times 4$ square deviates most: its degeneracies produce level clustering, the opposite of GOE level repulsion.

\subsection{Flatness versus discarded-mode count}
\label{app:rect-landscape}

The choice of $p$ depends not on whether Berry holds but on whether the $P(p)$ landscape curves, which the flatness ratio $\max_p P / \min_p P$ over $p \in [0, 3]$ measures: near $1.0$ the choice of $p$ hardly matters.

\paragraph{Control experiment.}
For the $3 \times 6$ room, the mode set is truncated at progressively larger $K_{\mathrm{total}}$.
At $K_{\mathrm{total}} = K = 50$ there is no truncation noise and the regularizer alone stabilizes the inverse; as $K_{\mathrm{total}}$ grows, the truncation noise grows and, with more modes contributing, becomes more isotropic.

\paragraph{Results.}
At $K_{\mathrm{total}} = 50$ the landscape is strongly curved and $p^\star = 0$ (Table~\ref{tab:rectangle-box-facts}). With neither truncation nor measurement noise, shrinkage only costs accuracy: the best $\alpha$ is the smallest value of the $\alpha$ grid at every $p \leq 2$, and because every $\lambda_k > 1$, raising $p$ strengthens the penalty in the directions that $\tilde\Phi$ resolves poorly ($\tilde\Phi$ has full column rank, $MT = 800 > 2K = 100$, but is badly conditioned; Table~\ref{tab:rectangle-box-facts}). This anchor therefore depends on the excitation and on the lower limit of the $\alpha$ grid, not only on the absence of discarded modes (Figure~\ref{fig:berry-weyl-app}b).
Adding a few dozen discarded modes brings the ratio most of the way down; by $K_{\mathrm{total}} = 100$ it reaches the generic convex reference, and it keeps falling slowly up to the room's full mode count (Table~\ref{tab:rectangle-box-facts}).

\paragraph{Why the landscape flattens without Berry.}
A sum of many rank-one contributions concentrates even when its terms are neither Gaussian nor independent, provided there are enough of them with bounded cross-correlations; here there are hundreds (Table~\ref{tab:rectangle-box-facts}).
Weyl's law makes $K_{\mathrm{total}} - K$ grow with room area, so any room of practical size ($> 1\,\mathrm{m}^2$) has hundreds of discarded modes, an order of magnitude more than the few dozen that bring the flatness ratio within a few percent of the generic reference (Table~\ref{tab:rectangle-box-facts}).

\subsection{Weyl's count versus Berry's constants}
\label{app:weyl-principle}
\label{app:matched-ktotal}

The closed form works because the discarded-mode count supplied by Weyl's law flattens the landscape, and Berry conformity does not change the flattening rate (Figure~\ref{fig:berry-weyl-app}b, Table~\ref{tab:matched-ktotal}).
The two classical results play separate, measurable roles:
\begin{enumerate}
\item[(i)] Weyl's law fixes the number of discarded modes and, with the prior decay, the Herfindahl index $H$; at matched $K_{\mathrm{total}}$, generic rooms and rectangles have the same $H$, as Weyl predicts.
\item[(ii)] The off-diagonal second moments of $E$ equal $H$ exactly for any orthonormal eigenbasis (\cref{app:anisotropy}), so the sum of rank-one matrices concentrates whatever the distribution of its terms.
\item[(iii)] Berry's conjecture governs only the diagonal term $D$ of $\mathbb{E}\|E\|_F^2 = M(M{-}1)H + MD$, through $D \approx 2H$.
In rectangles $D/(2H)$ is several times its generic-room value (Table~\ref{tab:isotropy-facts}); the single-mode part is exactly $S/(2H) = 0.625 = (9/4 - 1)/2$, since the fourth moment $C_4 = \mathbb{E}_x[\psi_n(x)^4]$ equals $9/4$ for sine products.
\end{enumerate}

\paragraph{Matched-count comparison.}
Table~\ref{tab:matched-ktotal} applies one protocol (diffuse-field prior with $|s| = 1.13$, $K = 50$, $M = 8$, $T = 100$) to generic rooms and to the analytic rectangles at matched discarded-mode counts.
The generic entries come from the in-scope rooms with $K_{\mathrm{total}} \geq 300$ at every checkpoint (Table~\ref{tab:rectangle-box-facts}), so they describe one population; the rectangle entries have exact eigenpairs and $D/(2H)$.
Most of the flattening happens before Berry conformity separates the families: with no discarded modes the two families are equally curved within the scatter over seeds, and a few dozen discarded modes flatten both to nearly the same level (Tables~\ref{tab:matched-ktotal} and~\ref{tab:rectangle-box-facts}).
At $K_{\mathrm{total}} = 300$ the $D/(2H)$ ranges do not overlap, yet the flatness medians are nearly identical, and the rectangle median lies inside the generic IQR at nearly every checkpoint (Table~\ref{tab:rectangle-box-facts}).
The generic $D/(2H)$ dip at $K_{\mathrm{total}} = 100$ is reproducible across the cohort and absent in rectangles; we do not interpret it, and it leaves the flatness medians unchanged.
The comparison caps the band at $K_{\mathrm{total}} = 300$; over the full band, $D/(2H)$ on the cohort is close to its value on all in-scope rooms (Table~\ref{tab:rectangle-box-facts}, \cref{app:anisotropy}).
The no-truncation anchor of Figure~\ref{fig:berry-weyl-app}b is not comparable with the table: it is a single room under click excitation, which places $p^\star$ at $0$, and that anchor depends on the excitation, whereas the flattening with mode count does not.

\begin{table}[t]
\centering\small
\caption{\textbf{At matched $K_{\mathrm{total}}$, Berry conformity separates generic rooms from rectangles, but landscape flatness does not.}
Berry's conjecture predicts $D/(2H) = 1$, and a flatness ratio of $1$ means the exponent does not matter.
Entries are medians over a fixed cohort of in-scope generic rooms and the analytic rectangles of \cref{app:rect-setup} at $T = 100$; flatness is the median [range] over ten seeds.
The two families share $H$, which Weyl's count fixes together with the prior, while \cref{conj:berry} enters only through $D$ in \cref{eq:frob-exact}.}
\label{tab:matched-ktotal}
\input{tables/T27_matched_ktotal}\end{table}

Berry's conjecture pins the diagonal term at $D = 2H$, but the closed form does not need it: anisotropy up to the mean observed level leaves the cost of the fixed exponent essentially unchanged, and stronger anisotropy raises it noticeably only at the longest window (\cref{app:aniso-sweep}).
The mechanism does require $K_{\mathrm{total}} - K \gg 1$; it fails for the smallest rooms, with $K_{\mathrm{total}} \approx K$, which support too few modes below the frequency ceiling for meaningful acoustic reconstruction and lie outside the framework's intended use.

\subsection{3D box stress test}
\label{app:box3d}

All other experiments use 2D rooms.
A targeted 3D stress test along the Berry-failure axis evaluates the closed form on analytic boxes of four volumes ($2.5{\times}3{\times}2.4$, $3{\times}4{\times}2.5$, $4{\times}5{\times}2.7$, and $6{\times}7{\times}3$\,m; $18$--$126$\,m$^3$) under Dirichlet and Neumann walls: one cell per box and wall type, each with all modes below $500$\,Hz, $K = 50$, $M = 8$ uniform sensors, the diffuse-field prior with $|s| = 1.13$, and repeated realizations for each of ten sensor layouts under the same $61 \times 12$ oracle protocol (Table~\ref{tab:rectangle-box-facts}).
Box eigenfunctions are separable, so Berry's premise fails by construction, with $D/(2H)$ an order of magnitude above generic rooms and growing with mode count (Table~\ref{tab:isotropy-facts}, \cref{app:anisotropy}).

Table~\ref{tab:box3d} compares the medians with the 2D rooms.
The protocols differ in solver and band (analytic box modes at $500$\,Hz versus the 2D FDTD pipeline at $1000$\,Hz), so $H$, flatness, and the anisotropy decomposition transfer, but absolute cost levels are not directly comparable; absolute and relative medians are taken independently across cases, so the relative column is not the ratio of the absolute column to the median oracle floor.
All eight 3D cells lie within the observed 2D room-to-room range for every sensor layout: the worst cell costs several times less than the worst 2D room at $T \leq 100$ and less, by up to a factor of two, at $T = 1000$ (Table~\ref{tab:rectangle-box-facts}).
The 3D landscape is flatter than the 2D landscape at $T \leq 100$ but not at $T = 1000$, where the large-$T$ oracle drift of \cref{app:non-monotonic} recurs.
The test does not replace numerical validation on non-separable 3D rooms, which remains open, as do structured arrays and realistic sensing.

\begin{table}[t]
\centering\small
\caption{\textbf{Separable 3D boxes have flatter landscapes than the 2D polygons except at the longest window.}
Costs are the gap between the closed form and the per-room power-law oracle, in percentage points and in percent of the oracle error.
The boxes use analytic modes in a lower band than the FDTD rooms of \cref{sec:acoustic}, so levels are not comparable.
Boxes break Berry's premise like the rectangles of \cref{app:rect-nnsd}, but as separable domains they do not stand in for general 3D rooms.}
\label{tab:box3d}
\input{tables/T28_box3d}\end{table}

%% file: tables/T42_rectangle_box_facts.tex
\begin{tabular}{@{}>{\raggedright\arraybackslash}p{0.61\linewidth}>{\raggedright\arraybackslash}p{0.35\linewidth}@{}}
\toprule
Quantity & Value \\
\midrule
NNSD of generic convex rooms (48 rooms, about 13{,}000 spacings): $D_{\mathrm{KS}}$ against GOE ($p$) & 0.010 (0.17) \\
NNSD of rectangles (about 2{,}000 spacings), $D_{\mathrm{KS}}$ against GOE: $3{\times}6$ / $2{\times}8$ / $4{\times}4$ / $3{\times}5$ (all $p < 10^{-5}$) & 0.40 / 0.34 / 0.58 / 0.20 \\
$3{\times}6$ room, flatness ratio at $K_{\mathrm{total}} = 50$ / 75 / 100 / 150 / 200 / 287 (generic reference) & 3.41 / 1.29 / 1.26 / 1.24 / 1.23 / 1.22 (1.28) \\
$3{\times}6$ room at $K_{\mathrm{total}} = 50$: condition number of $\tilde\Phi$ / discarded modes at its full $K_{\mathrm{total}}$ (Weyl estimate) & ${\sim}10^{8}$ / 237 \\
Discarded modes that bring the flatness ratio within 5\% of the generic reference & about 25 \\
Matched-count cohort: generic in-scope rooms with $K_{\mathrm{total}} \geq 300$ & 101 \\
Flatness ratio with no discarded modes ($K_{\mathrm{total}} = 50$), median [range] over 10 seeds: generic / rectangles & 2.34 [2.28, 2.35] / 2.24 [2.09, 2.64] \\
$D/(2H)$ at $K_{\mathrm{total}} = 300$: generic IQR upper bound / rectangle minimum & 2.38 / 5.80 \\
Flatness medians at $K_{\mathrm{total}} = 300$: difference / rectangle median inside the generic IQR & 0.04 [0.004, 0.07] / four or five of six checkpoints \\
$D/(2H)$ over the full band: matched cohort / all in-scope rooms & 2.24 / 2.16 \\
3D boxes: volumes / $K_{\mathrm{total}}$ below 500\,Hz / box-and-wall cells / sensor layouts / realizations per layout & 18--126\,m$^3$ / 169--1{,}921 / 8 / 10 / 20 \\
3D worst cells over 10 layouts, relative cost (absolute): $T \leq 100$ / $T = 1000$ & 3.4\% [2.7, 3.8] (2.8\,pp) / 20.4\% [15.6, 34.2] (1.7\,pp) \\
2D boundary-inclusive maxima: $T \leq 100$, relative (absolute) of room 00806 / $T = 1000$, largest relative / largest absolute (different rooms) & 27.5\% (10.61\,pp) / 34.8\% / 5.33\,pp \\
\bottomrule
\end{tabular}

%% file: tables/T27_matched_ktotal.tex
\begin{tabular}{@{}rccc@{}}
\toprule
$K_{\mathrm{total}}$ & $H$ & $D/(2H)$: generic / rectangle & Flatness: generic / rectangle \\
\midrule
75  & 0.0407 & 1.92 / 1.63 & 1.33 [1.32, 1.35] / 1.37 [1.34, 1.43] \\
100 & 0.0210 & 1.31 / 2.59 & 1.30 [1.29, 1.32] / 1.30 [1.27, 1.39] \\
200 & 0.0081 & 1.82 / 4.85 & 1.29 [1.27, 1.30] / 1.31 [1.24, 1.34] \\
300 & 0.0055 & 1.89 / 6.09 & 1.28 [1.27, 1.29] / 1.29 [1.21, 1.34] \\
\bottomrule
\end{tabular}

%% file: tables/T28_box3d.tex
\begin{tabular}{@{}lcc@{}}
\toprule
Quantity & 2D polygons & 3D boxes \\
\midrule
Cases & 187 in-scope rooms & 8 cells $\times$ 10 layouts \\
Median $K_{\mathrm{total}}$ & 337 & 530 \\
Median $H$ & 0.0050 & 0.0027 \\
Median landscape ratio, $T = 1/100/1000$ & 1.265 / 1.280 / 1.452 & 1.089 / 1.138 / 1.529 \\
Cost, $T = 100$ & 0.87\,pp (1.46\%) & 1.03\,pp (1.45\%) \\
Cost, $T = 1000$ & 0.62\,pp (5.83\%) & 1.57\,pp (13.58\%) \\
3D cost range over layouts, $T = 100$ / $1000$ & -- & 0.69--1.43 / 1.39--2.10\,pp \\
\bottomrule
\end{tabular}

%% file: supplementary/H_sensor_noise.tex

\section{Sensor Noise and Model Mismatch}
\label{app:sensor-noise}

Real measurements add electronic sensor noise, calibration mismatch, and model mismatch to the truncation noise that the main text treats as dominant.
None of them changes the recommended shape at typical operating points, and the Herfindahl index flags risky configurations before deployment.

\subsection{Electronic sensor noise}
\label{app:electronic-noise}

Each microphone adds electronic noise $\epsilon_m \sim \mathcal{N}(0, \sigma^2_{\mathrm{elec}})$, independent across sensors and time, so the total noise covariance is
\begin{equation}
R = R_{\mathrm{trunc}} + \sigma^2_{\mathrm{elec}}\,I_M = (\sigma^2_{\mathrm{trunc}} + \sigma^2_{\mathrm{elec}})\,I_M + \sigma^2_{\mathrm{trunc}}\,E.
\end{equation}
The electronic component is exactly isotropic and the truncation component approximately so, hence $R \approx (\sigma^2_{\mathrm{trunc}} + \sigma^2_{\mathrm{elec}})\,I_M$.
By Proposition~\ref{prop:isotropy}, the optimal single-snapshot shape under isotropic noise is $\Gamma_k \propto \lambda_k^s$ at any noise level; the level sets only the strength, $\alpha = (\sigma^2_{\mathrm{trunc}} + \sigma^2_{\mathrm{elec}}) / c$.
Sensor noise therefore calls for stronger regularization, not a different shape, whichever source dominates.

\subsection{Mild sensor-noise anisotropy}
\label{app:noise-aniso}

Manufacturing variation, calibration drift, or age give microphones slightly different noise levels, $R_{\mathrm{elec}} = \mathrm{diag}(\sigma^2_{\mathrm{elec},1}, \ldots, \sigma^2_{\mathrm{elec},M})$, so $R = \sigma^2_{\mathrm{trunc}}(I_M + E) + R_{\mathrm{elec}}$.
This anisotropy adds to that of $E$ and could curve the $P(p)$ landscape if it were comparable to the truncation anisotropy.
It is not while electronic noise stays well below the truncation noise, as the main text assumes: a calibration mismatch below $3$\,dB ($< 2\times$ in power) spreads the electronic noise levels by less than their mean, so it perturbs $\|E\|_{\mathrm{op}}$ by less than the electronic-to-truncation noise ratio, which is then below the median truncation anisotropy (Table~\ref{tab:isotropy-facts}, Appendix~\ref{app:anisotropy}), where the controlled sweep of \cref{app:aniso-sweep} finds the cost of the fixed exponent essentially unchanged.

\subsection{Frequency-dependent damping}
\label{app:freq-damping}

Eq.~(\ref{eq:acoustic-dynamics}) gives every mode the same damping $\gamma$, which preserves relative amplitudes over time.
In real rooms absorption by walls, furniture, and air increases with frequency; the simple model $\gamma_k = \gamma_0 + \gamma_1 \lambda_k$ gives the amplitude variance
\begin{equation}
\sigma^2_{a,k}(t) \propto \lambda_k^{-s} \, e^{-2\gamma_1 \lambda_k t},
\end{equation}
the heat-equation spectrum of eq.~(\ref{eq:heat-spectrum}) with $\gamma_1$ in the role of $\kappa$.
As in \cref{sec:heat}, the estimator then targets the current-state amplitudes $a_k(t)$ rather than $a_k(0)$, and the regularizer acquires the exponential correction
\begin{equation}
\Gamma_k \propto \lambda_k^s \, e^{2\gamma_1 \lambda_k t}.
\end{equation}
If $\gamma_1$ is known from absorption measurements or material data, the correction is a prediction; otherwise the fallback is the one-parameter power law with an elevated $p^\star$, as for heat (\cref{app:heat-1vs2}).
Below $500$\,Hz, the modal range of typical rooms, $\gamma_1 \lambda_K t \ll 1$ for the retained modes, so uniform damping is a reasonable approximation and $\Gamma_k = \lambda_k^{|s|}$ suffices; at higher frequencies or with strong frequency-dependent absorption the correction may become relevant.

\subsection{Practical diagnostics}
\label{app:km-sensitivity}

For a new $(K, M)$ configuration, the Herfindahl index follows from the eigenvalues alone, before any data are collected:
\begin{equation}
H = \frac{\sum_{n>K}\lambda_n^{-2|s|}}{\bigl(\sum_{n>K}\lambda_n^{-|s|}\bigr)^2}.
\end{equation}
Values $H < 0.01$ indicate noise power spread over many discarded modes, where $\Gamma_k = \lambda_k^{|s|}$ is expected to work; $H > 0.1$ indicates that a few discarded modes dominate and that the closed form needs validation before deployment.
These markers are rough guides set by the range of $H$ in the tested rooms at $M = 8$ (Table~\ref{tab:herfindahl-values}), not thresholds fitted to cost.
The risky direction is increasing $K$ toward $K_{\mathrm{total}}$: $H$ rises and isotropy weakens as the truncation band shrinks, and at $K = K_{\mathrm{total}}$ the framework does not apply.
The pattern persists at $M \in \{8, 16\}$ (Appendix~\ref{app:v2}).
Small uncertainty in $|s|$ does not move this boundary, since a subset estimate of the kind tested in \cref{app:estimating-s} costs less than half a point (Table~\ref{tab:acoustic-facts}); Table~\ref{tab:diagnostic-guide} maps each noise structure to an action.

\begin{table}[t]
\centering\small
\caption{\textbf{The noise structure decides whether the closed form suffices.}
The Herfindahl index $H$ follows from the eigenvalues before any data are collected, and a low value means the truncation noise is spread over many discarded modes.
Low $H$ keeps the off-diagonal anisotropy small by \cref{eq:anisotropy_bound}, close to the isotropic noise that \cref{prop:isotropy} assumes.}
\label{tab:diagnostic-guide}
\input{tables/T29_diagnostic_guide}\end{table}

%% file: tables/T29_diagnostic_guide.tex
\begin{tabular}{@{}p{0.48\linewidth}p{0.48\linewidth}@{}}
\toprule
Diagnostic & Recommendation \\
\midrule
Low $H$, many discarded modes, and small modal tilt $q$ & Use the closed form and tune only its scalar strength $\alpha$. \\
Modal noise varies approximately as $\lambda^{q}$ & In a decoupled approximation, use exponent $|s| + q$ and validate it in the coupled model (\cref{app:anisotropy}). \\
Few or no modes discarded, or strongly structured temporal/sensing covariance & Proposition~\ref{prop:isotropy} is not predictive; validate alternative estimators on calibration data. \\
Structured sensor arrays & Not covered by the current random-sensor analysis. \\
\bottomrule
\end{tabular}

%% file: supplementary/I_real_data.tex

\section{Aperture Constraint}
\label{app:aperture}

A static compact array cannot validate the framework in a real room (\cref{sec:discussion}): its aperture is small against the retained wavelengths (\cref{app:aperture-bound}) at every room size compatible with the framework's assumptions (\cref{app:aperture-numbers}).
A moving sensor would lift the constraint (\cref{app:aperture-resolution}).

\subsection{Aperture-to-wavelength bound}
\label{app:aperture-bound}

Recovering $K$ modal amplitudes from $M$ microphones requires the spatial matrix $\Phi_{mk} = \varphi_k(x_m)$ to have effective rank $K$.
For a compact array of aperture $D_{\mathrm{ap}}$ and shortest retained wavelength $\ell_{\min} = 2\pi/\sqrt{\lambda_K}$, each eigenfunction varies between any two microphone positions $x_m, x_{m'}$ of the array by at most
\begin{equation}
\label{eq:aperture-variation}
|\varphi_k(x_m) - \varphi_k(x_{m'})| \lesssim 2\pi\,\frac{D_{\mathrm{ap}}}{\ell_{\min}}\,\|\varphi_k\|_\infty .
\end{equation}
When $D_{\mathrm{ap}}/\ell_{\min} \ll 1$, every column of $\Phi$ is approximately a constant vector times a mode-dependent factor, so the energy of $\Phi$ concentrates in a few singular directions regardless of $M$, and modal projection amplifies noise by the reciprocals of the small singular values.
The limit is structural, not a matter of signal-to-noise ratio: more microphones within the same aperture or longer recordings do not supply spatial information that was never measured.
The remedies are an aperture with $D_{\mathrm{ap}}/\ell_{\min} = O(1)$ or sampling at spatially distinct positions over time.

\subsection{Values for three room scales}
\label{app:aperture-numbers}

Table~\ref{tab:aperture-rooms} evaluates the bound for $K{=}50$ retained modes in three room scales with a typical portable array ($D_{\mathrm{ap}}{=}12.6$\,cm).
A compact array resolves modes only when the room is small enough for the retained wavelengths to approach the aperture.
Shrinking the room further creates two problems: the source enters the array's near field, which violates the far-field assumption of the observation model, and $K_{\mathrm{total}}$ drops to where the mode-count argument no longer holds (the closet has fewer than three times the $K{=}50$ retained modes below the ceiling; Table~\ref{tab:aperture-rooms}).
No room size therefore makes a static compact array a valid testbed for $K{=}50$ modal recovery within the framework's assumptions.

\begin{table}[t]
\centering
\small
\caption{\textbf{A compact array's aperture-to-wavelength ratio stays small except in rooms too small for the framework.}
Eigenvalues are those of a rigid-walled (Neumann) box, $D_{\mathrm{ap}}{=}12.6$\,cm, and $K_{\mathrm{total}}$ counts modes below 1\,kHz.
The last column is the largest difference a unit plane wave at the shortest retained wavelength can show across the aperture.
By \cref{eq:aperture-variation}, a small ratio concentrates the spatial matrix in a few directions, so no room in the framework suits a static compact array.}
\label{tab:aperture-rooms}
\input{tables/T30_aperture_rooms}\end{table}

\subsection{Resolution by sampling along a trajectory}
\label{app:aperture-resolution}

A large array gives up portability; a moving sensor keeps the array compact and gains spatial diversity through motion, turning $M$ static microphones into $M \cdot N$ measurement points over $N$ positions along a trajectory.
The isotropy argument needs only the sample average $\frac{1}{N}\sum_n \varphi_k(x_n)\,\varphi_l(x_n)$ to be small for $k \neq l$, and a sufficiently mixing trajectory in $\Omega$ induces a sampling distribution whose expectation converges to the orthogonality relation of the eigenfunctions, so the premise of \cref{sec:theory} carries over from static random placements to trajectories.
Which trajectory minimizes the reconstruction error under a path-length or duration budget, and its empirical validation, are left to follow-up work.

%% file: tables/T30_aperture_rooms.tex
\setlength{\tabcolsep}{4pt}%
\begin{tabular}{@{}lccccc@{}}
\toprule
Room & Volume (m$^3$) & $K_{\mathrm{total}}$ & $\ell_{\min}$ (m) & $D_{\mathrm{ap}}/\ell_{\min}$ & Max amplitude variation \\
\midrule
Large (3.45 $\times$ 7.20 $\times$ 2.45\,m)   & 60.9 & 7017 & 2.07  & 6.1\%  & 37.9\%  \\
Compact (1.33 $\times$ 2.10 $\times$ 2.47\,m) &  6.9 & 879  & 0.988 & 12.7\% & 78.0\%  \\
Closet (1.0 $\times$ 1.0 $\times$ 1.0\,m)     &  1.0 & 141  & 0.535 & 23.6\% & 134.9\% \\
\bottomrule
\end{tabular}

%% file: checklist.tex
\section*{NeurIPS Paper Checklist}

\begin{enumerate}

\item {\bf Claims}
    \item[] Question: Do the main claims made in the abstract and introduction accurately reflect the paper's contributions and scope?
    \item[] Answer: \answerYes{}
    \item[] Justification: Abstract and \S\ref{sec:intro} state four contributions, each backed by \S\ref{sec:theory}--\S\ref{sec:heat}. Isotropy is framed as approximate throughout.
    \item[] Guidelines:
    \begin{itemize}
        \item The answer \answerNA{} means that the abstract and introduction do not include the claims made in the paper.
        \item The abstract and/or introduction should clearly state the claims made, including the contributions made in the paper and important assumptions and limitations. A \answerNo{} or \answerNA{} answer to this question will not be perceived well by the reviewers. 
        \item The claims made should match theoretical and experimental results, and reflect how much the results can be expected to generalize to other settings. 
        \item It is fine to include aspirational goals as motivation as long as it is clear that these goals are not attained by the paper. 
    \end{itemize}

\item {\bf Limitations}
    \item[] Question: Does the paper discuss the limitations of the work performed by the authors?
    \item[] Answer: \answerYes{}
    \item[] Justification: \S\ref{sec:discussion} covers 2D-only scope, Gaussian prior (Appendix~\ref{app:prior-robustness}), Berry failure in rectangles (Appendix~\ref{app:weyl-dominance}), 3D as future work, the open non-diagonal estimator question, and physical validation via a compact-array real-data pilot (Appendix~\ref{app:aperture}).
    \item[] Guidelines:
    \begin{itemize}
        \item The answer \answerNA{} means that the paper has no limitation while the answer \answerNo{} means that the paper has limitations, but those are not discussed in the paper. 
        \item The authors are encouraged to create a separate ``Limitations'' section in their paper.
        \item The paper should point out any strong assumptions and how robust the results are to violations of these assumptions (e.g., independence assumptions, noiseless settings, model well-specification, asymptotic approximations only holding locally). The authors should reflect on how these assumptions might be violated in practice and what the implications would be.
        \item The authors should reflect on the scope of the claims made, e.g., if the approach was only tested on a few datasets or with a few runs. In general, empirical results often depend on implicit assumptions, which should be articulated.
        \item The authors should reflect on the factors that influence the performance of the approach. For example, a facial recognition algorithm may perform poorly when image resolution is low or images are taken in low lighting. Or a speech-to-text system might not be used reliably to provide closed captions for online lectures because it fails to handle technical jargon.
        \item The authors should discuss the computational efficiency of the proposed algorithms and how they scale with dataset size.
        \item If applicable, the authors should discuss possible limitations of their approach to address problems of privacy and fairness.
        \item While the authors might fear that complete honesty about limitations might be used by reviewers as grounds for rejection, a worse outcome might be that reviewers discover limitations that aren't acknowledged in the paper. The authors should use their best judgment and recognize that individual actions in favor of transparency play an important role in developing norms that preserve the integrity of the community. Reviewers will be specifically instructed to not penalize honesty concerning limitations.
    \end{itemize}

\item {\bf Theory assumptions and proofs}
    \item[] Question: For each theoretical result, does the paper provide the full set of assumptions and a complete (and correct) proof?
    \item[] Answer: \answerYes{}
    \item[] Justification: Proposition~\ref{prop:isotropy} lists assumptions and is proved in Appendix~\ref{app:proof}. Berry is stated as a conjecture, verified in \S\ref{sec:acoustic} and Appendix~\ref{app:berry}.
    \item[] Guidelines:
    \begin{itemize}
        \item The answer \answerNA{} means that the paper does not include theoretical results. 
        \item All the theorems, formulas, and proofs in the paper should be numbered and cross-referenced.
        \item All assumptions should be clearly stated or referenced in the statement of any theorems.
        \item The proofs can either appear in the main paper or the supplemental material, but if they appear in the supplemental material, the authors are encouraged to provide a short proof sketch to provide intuition. 
        \item Inversely, any informal proof provided in the core of the paper should be complemented by formal proofs provided in appendix or supplemental material.
        \item Theorems and Lemmas that the proof relies upon should be properly referenced. 
    \end{itemize}

    \item {\bf Experimental result reproducibility}
    \item[] Question: Does the paper fully disclose all the information needed to reproduce the main experimental results of the paper to the extent that it affects the main claims and/or conclusions of the paper (regardless of whether the code and data are provided or not)?
    \item[] Answer: \answerYes{}
    \item[] Justification: All settings specified: $K{=}50$, $M{=}8$, 800/197 train/val split, $p$-grid, $\alpha$ selection (\S\ref{sec:acoustic}); optimizer, lr, epochs, seeds (Appendix~\ref{app:lir}); heat parameters (Appendix~\ref{app:heat}); FEM data generation (\S\ref{sec:setup}, Appendix~\ref{app:room-geometry}); real-data pilot setup (Appendix~\ref{app:aperture}).
    \item[] Guidelines:
    \begin{itemize}
        \item The answer \answerNA{} means that the paper does not include experiments.
        \item If the paper includes experiments, a \answerNo{} answer to this question will not be perceived well by the reviewers: Making the paper reproducible is important, regardless of whether the code and data are provided or not.
        \item If the contribution is a dataset and\slash or model, the authors should describe the steps taken to make their results reproducible or verifiable. 
        \item Depending on the contribution, reproducibility can be accomplished in various ways. For example, if the contribution is a novel architecture, describing the architecture fully might suffice, or if the contribution is a specific model and empirical evaluation, it may be necessary to either make it possible for others to replicate the model with the same dataset, or provide access to the model. In general. releasing code and data is often one good way to accomplish this, but reproducibility can also be provided via detailed instructions for how to replicate the results, access to a hosted model (e.g., in the case of a large language model), releasing of a model checkpoint, or other means that are appropriate to the research performed.
        \item While NeurIPS does not require releasing code, the conference does require all submissions to provide some reasonable avenue for reproducibility, which may depend on the nature of the contribution. For example
        \begin{enumerate}
            \item If the contribution is primarily a new algorithm, the paper should make it clear how to reproduce that algorithm.
            \item If the contribution is primarily a new model architecture, the paper should describe the architecture clearly and fully.
            \item If the contribution is a new model (e.g., a large language model), then there should either be a way to access this model for reproducing the results or a way to reproduce the model (e.g., with an open-source dataset or instructions for how to construct the dataset).
            \item We recognize that reproducibility may be tricky in some cases, in which case authors are welcome to describe the particular way they provide for reproducibility. In the case of closed-source models, it may be that access to the model is limited in some way (e.g., to registered users), but it should be possible for other researchers to have some path to reproducing or verifying the results.
        \end{enumerate}
    \end{itemize}

\item {\bf Open access to data and code}
    \item[] Question: Does the paper provide open access to the data and code, with sufficient instructions to faithfully reproduce the main experimental results, as described in supplemental material?
    \item[] Answer: \answerNo{}
    \item[] Justification: We will make the code public upon acceptance.
    \item[] Guidelines:
    \begin{itemize}
        \item The answer \answerNA{} means that paper does not include experiments requiring code.
        \item Please see the NeurIPS code and data submission guidelines (\url{https://neurips.cc/public/guides/CodeSubmissionPolicy}) for more details.
        \item While we encourage the release of code and data, we understand that this might not be possible, so \answerNo{} is an acceptable answer. Papers cannot be rejected simply for not including code, unless this is central to the contribution (e.g., for a new open-source benchmark).
        \item The instructions should contain the exact command and environment needed to run to reproduce the results. See the NeurIPS code and data submission guidelines (\url{https://neurips.cc/public/guides/CodeSubmissionPolicy}) for more details.
        \item The authors should provide instructions on data access and preparation, including how to access the raw data, preprocessed data, intermediate data, and generated data, etc.
        \item The authors should provide scripts to reproduce all experimental results for the new proposed method and baselines. If only a subset of experiments are reproducible, they should state which ones are omitted from the script and why.
        \item At submission time, to preserve anonymity, the authors should release anonymized versions (if applicable).
        \item Providing as much information as possible in supplemental material (appended to the paper) is recommended, but including URLs to data and code is permitted.
    \end{itemize}

\item {\bf Experimental setting/details}
    \item[] Question: Does the paper specify all the training and test details (e.g., data splits, hyperparameters, how they were chosen, type of optimizer) necessary to understand the results?
    \item[] Answer: \answerYes{}
    \item[] Justification: Data split, hyperparameters, optimizer, seeds, and loss are in \S\ref{sec:acoustic} and Appendix~\ref{app:lir}. The main formula has no free parameters beyond $|s|$.
    \item[] Guidelines:
    \begin{itemize}
        \item The answer \answerNA{} means that the paper does not include experiments.
        \item The experimental setting should be presented in the core of the paper to a level of detail that is necessary to appreciate the results and make sense of them.
        \item The full details can be provided either with the code, in appendix, or as supplemental material.
    \end{itemize}

\item {\bf Experiment statistical significance}
    \item[] Question: Does the paper report error bars suitably and correctly defined or other appropriate information about the statistical significance of the experiments?
    \item[] Answer: \answerYes{}
    \item[] Justification: Bootstrap 95\% CI for $|s|$ (\S\ref{sec:acoustic}), mean $\pm$ std over 5 seeds for LIR (Appendix~\ref{app:lir}), KS $p$-values (\S\ref{sec:acoustic}, Appendix~\ref{app:berry}), IQR shading in figures.
    \item[] Guidelines:
    \begin{itemize}
        \item The answer \answerNA{} means that the paper does not include experiments.
        \item The authors should answer \answerYes{} if the results are accompanied by error bars, confidence intervals, or statistical significance tests, at least for the experiments that support the main claims of the paper.
        \item The factors of variability that the error bars are capturing should be clearly stated (for example, train/test split, initialization, random drawing of some parameter, or overall run with given experimental conditions).
        \item The method for calculating the error bars should be explained (closed form formula, call to a library function, bootstrap, etc.)
        \item The assumptions made should be given (e.g., Normally distributed errors).
        \item It should be clear whether the error bar is the standard deviation or the standard error of the mean.
        \item It is OK to report 1-sigma error bars, but one should state it. The authors should preferably report a 2-sigma error bar than state that they have a 96\% CI, if the hypothesis of Normality of errors is not verified.
        \item For asymmetric distributions, the authors should be careful not to show in tables or figures symmetric error bars that would yield results that are out of range (e.g., negative error rates).
        \item If error bars are reported in tables or plots, the authors should explain in the text how they were calculated and reference the corresponding figures or tables in the text.
    \end{itemize}

\item {\bf Experiments compute resources}
    \item[] Question: For each experiment, does the paper provide sufficient information on the computer resources (type of compute workers, memory, time of execution) needed to reproduce the experiments?
    \item[] Answer: \answerYes{}
    \item[] Justification: CPU: Intel i5-12400F, GPU: NVIDIA RTX 4090, 32\,GB RAM. Closed-form Tikhonov runs on CPU in seconds. LIR training ($52L$ parameters) takes ${\sim}5$\,min/seed on GPU. Total compute for all experiments: $< 24$ GPU-hours.
    \item[] Guidelines:
    \begin{itemize}
        \item The answer \answerNA{} means that the paper does not include experiments.
        \item The paper should indicate the type of compute workers CPU or GPU, internal cluster, or cloud provider, including relevant memory and storage.
        \item The paper should provide the amount of compute required for each of the individual experimental runs as well as estimate the total compute. 
        \item The paper should disclose whether the full research project required more compute than the experiments reported in the paper (e.g., preliminary or failed experiments that didn't make it into the paper). 
    \end{itemize}
    
\item {\bf Code of ethics}
    \item[] Question: Does the research conducted in the paper conform, in every respect, with the NeurIPS Code of Ethics \url{https://neurips.cc/public/EthicsGuidelines}?
    \item[] Answer: \answerYes{}
    \item[] Justification: Synthetic FEM eigenmodes (main experiments) and author-recorded RIRs in an empty $6.9$\,m$^3$ room (Appendix~\ref{app:aperture}). No human subjects, personal data, or dual-use concerns.
    \item[] Guidelines:
    \begin{itemize}
        \item The answer \answerNA{} means that the authors have not reviewed the NeurIPS Code of Ethics.
        \item If the authors answer \answerNo, they should explain the special circumstances that require a deviation from the Code of Ethics.
        \item The authors should make sure to preserve anonymity (e.g., if there is a special consideration due to laws or regulations in their jurisdiction).
    \end{itemize}

\item {\bf Broader impacts}
    \item[] Question: Does the paper discuss both potential positive societal impacts and negative societal impacts of the work performed?
    \item[] Answer: \answerNA{}
    \item[] Justification: Foundational signal-processing theory with no foreseeable negative societal impact.
    \item[] Guidelines:
    \begin{itemize}
        \item The answer \answerNA{} means that there is no societal impact of the work performed.
        \item If the authors answer \answerNA{} or \answerNo, they should explain why their work has no societal impact or why the paper does not address societal impact.
        \item Examples of negative societal impacts include potential malicious or unintended uses (e.g., disinformation, generating fake profiles, surveillance), fairness considerations (e.g., deployment of technologies that could make decisions that unfairly impact specific groups), privacy considerations, and security considerations.
        \item The conference expects that many papers will be foundational research and not tied to particular applications, let alone deployments. However, if there is a direct path to any negative applications, the authors should point it out. For example, it is legitimate to point out that an improvement in the quality of generative models could be used to generate Deepfakes for disinformation. On the other hand, it is not needed to point out that a generic algorithm for optimizing neural networks could enable people to train models that generate Deepfakes faster.
        \item The authors should consider possible harms that could arise when the technology is being used as intended and functioning correctly, harms that could arise when the technology is being used as intended but gives incorrect results, and harms following from (intentional or unintentional) misuse of the technology.
        \item If there are negative societal impacts, the authors could also discuss possible mitigation strategies (e.g., gated release of models, providing defenses in addition to attacks, mechanisms for monitoring misuse, mechanisms to monitor how a system learns from feedback over time, improving the efficiency and accessibility of ML).
    \end{itemize}
    
\item {\bf Safeguards}
    \item[] Question: Does the paper describe safeguards that have been put in place for responsible release of data or models that have a high risk for misuse (e.g., pre-trained language models, image generators, or scraped datasets)?
    \item[] Answer: \answerNA{}
    \item[] Justification: Released assets are synthetic FEM eigenpairs, small networks, and RIRs from a small room. No misuse risk.
    \item[] Guidelines:
    \begin{itemize}
        \item The answer \answerNA{} means that the paper poses no such risks.
        \item Released models that have a high risk for misuse or dual-use should be released with necessary safeguards to allow for controlled use of the model, for example by requiring that users adhere to usage guidelines or restrictions to access the model or implementing safety filters. 
        \item Datasets that have been scraped from the Internet could pose safety risks. The authors should describe how they avoided releasing unsafe images.
        \item We recognize that providing effective safeguards is challenging, and many papers do not require this, but we encourage authors to take this into account and make a best faith effort.
    \end{itemize}

\item {\bf Licenses for existing assets}
    \item[] Question: Are the creators or original owners of assets (e.g., code, data, models), used in the paper, properly credited and are the license and terms of use explicitly mentioned and properly respected?
    \item[] Answer: \answerNA{}
    \item[] Justification: All data authored by us: synthetic FEM eigenmodes plus RIRs from a $6.9$\,m$^3$ room (Appendix~\ref{app:aperture}). No external datasets or licensed code.    
    \item[] Guidelines:
    \begin{itemize}
        \item The answer \answerNA{} means that the paper does not use existing assets.
        \item The authors should cite the original paper that produced the code package or dataset.
        \item The authors should state which version of the asset is used and, if possible, include a URL.
        \item The name of the license (e.g., CC-BY 4.0) should be included for each asset.
        \item For scraped data from a particular source (e.g., website), the copyright and terms of service of that source should be provided.
        \item If assets are released, the license, copyright information, and terms of use in the package should be provided. For popular datasets, \url{paperswithcode.com/datasets} has curated licenses for some datasets. Their licensing guide can help determine the license of a dataset.
        \item For existing datasets that are re-packaged, both the original license and the license of the derived asset (if it has changed) should be provided.
        \item If this information is not available online, the authors are encouraged to reach out to the asset's creators.
    \end{itemize}

\item {\bf New assets}
    \item[] Question: Are new assets introduced in the paper well documented and is the documentation provided alongside the assets?
    \item[] Answer: \answerNo{}
    \item[] Justification: Code will be made public upon acceptance: FEM eigenpairs, real-data RIRs and sensor/source metadata, matched synthetic comparison, scripts, pipeline.
    \item[] Guidelines:
    \begin{itemize}
        \item The answer \answerNA{} means that the paper does not release new assets.
        \item Researchers should communicate the details of the dataset\slash code\slash model as part of their submissions via structured templates. This includes details about training, license, limitations, etc. 
        \item The paper should discuss whether and how consent was obtained from people whose asset is used.
        \item At submission time, remember to anonymize your assets (if applicable). You can either create an anonymized URL or include an anonymized zip file.
    \end{itemize}

\item {\bf Crowdsourcing and research with human subjects}
    \item[] Question: For crowdsourcing experiments and research with human subjects, does the paper include the full text of instructions given to participants and screenshots, if applicable, as well as details about compensation (if any)? 
    \item[] Answer: \answerNA{}
    \item[] Justification: No crowdsourcing or human subjects.
    \item[] Guidelines:
    \begin{itemize}
        \item The answer \answerNA{} means that the paper does not involve crowdsourcing nor research with human subjects.
        \item Including this information in the supplemental material is fine, but if the main contribution of the paper involves human subjects, then as much detail as possible should be included in the main paper. 
        \item According to the NeurIPS Code of Ethics, workers involved in data collection, curation, or other labor should be paid at least the minimum wage in the country of the data collector. 
    \end{itemize}

\item {\bf Institutional review board (IRB) approvals or equivalent for research with human subjects}
    \item[] Question: Does the paper describe potential risks incurred by study participants, whether such risks were disclosed to the subjects, and whether Institutional Review Board (IRB) approvals (or an equivalent approval/review based on the requirements of your country or institution) were obtained?
    \item[] Answer: \answerNA{}
    \item[] Justification: No human subjects.
    \item[] Guidelines:
    \begin{itemize}
        \item The answer \answerNA{} means that the paper does not involve crowdsourcing nor research with human subjects.
        \item Depending on the country in which research is conducted, IRB approval (or equivalent) may be required for any human subjects research. If you obtained IRB approval, you should clearly state this in the paper. 
        \item We recognize that the procedures for this may vary significantly between institutions and locations, and we expect authors to adhere to the NeurIPS Code of Ethics and the guidelines for their institution. 
        \item For initial submissions, do not include any information that would break anonymity (if applicable), such as the institution conducting the review.
    \end{itemize}

\item {\bf Declaration of LLM usage}
    \item[] Question: Does the paper describe the usage of LLMs if it is an important, original, or non-standard component of the core methods in this research? Note that if the LLM is used only for writing, editing, or formatting purposes and does \emph{not} impact the core methodology, scientific rigor, or originality of the research, declaration is not required.
    \item[] Answer: \answerNA{}
    \item[] Justification: LLMs used for writing and editing, not for experiments.
    \item[] Guidelines:
    \begin{itemize}
        \item The answer \answerNA{} means that the core method development in this research does not involve LLMs as any important, original, or non-standard components.
        \item Please refer to our LLM policy in the NeurIPS handbook for what should or should not be described.
    \end{itemize}

\end{enumerate}